\documentclass{article}

\PassOptionsToPackage{numbers, compress}{natbib}
 \usepackage[preprint]{neurips_2026}

\usepackage[utf8]{inputenc} 
\usepackage[T1]{fontenc}    
\usepackage[pdfencoding=auto,psdextra,colorlinks,linkcolor=red,citecolor=blue,urlcolor=blue]{hyperref}       
\usepackage{url}            
\usepackage{booktabs}       
\usepackage{amsfonts}       
\usepackage{nicefrac}       
\usepackage{microtype}      
\usepackage{subcaption}
\usepackage{caption}
\usepackage{amsmath}
\usepackage{amssymb}
\usepackage{mathtools}
\usepackage{amsthm}
\usepackage{tabularray}
\usepackage[dvipsnames, table]{xcolor}
\usepackage{amsfonts}
\usepackage{comment}
\usepackage{makecell}
\usepackage{algorithm}
\usepackage{algorithmic}
\usepackage{enumitem}
\usepackage{tcolorbox}
\usepackage{multirow}
\usepackage{placeins}
\usepackage{makecell}

\usepackage[capitalize,noabbrev]{cleveref}
\UseTblrLibrary{booktabs}
\usepackage{bbding}
\input{Definitions}

\definecolor{darkgreen}{RGB}{0,100,0}
\definecolor{darkred}{RGB}{139,0,0}
\definecolor{darksalmon}{rgb}{0.91, 0.59, 0.48}
\definecolor{emerald}{rgb}{0.31, 0.78, 0.47}
\definecolor{green(pigment)}{rgb}{0.0, 0.65, 0.31}

\theoremstyle{plain}
\theoremstyle{definition}
\theoremstyle{remark}

\usepackage[textsize=tiny]{todonotes}

\renewcommand{\citet}{\citep}

\title{TodyComm: Task-Oriented Dynamic Communication for Multi-Round LLM-based Multi-Agent System}

\author{%
  Wenzhe Fan $^{1}$\thanks{Correspondence to: Wenzhe Fan, \texttt{wfan23@uic.edu}.  $^{1}$Department of Compute Science, University of Illinois at Chicago. } \\
  \And
  Tommaso Tognoli $^{1}$ \\
  \And
  Henry Peng Zou $^{1}$ \\
  \AND
  Chunyu Miao $^{1}$ \\
  \And
  Yibo Wang $^{1}$ \\
  \And
  Xinhua Zhang $^{1}$ \\
}

\begin{document}

\maketitle

\vspace{-1em}
\begin{abstract}
Multi-round LLM-based multi-agent systems rely on effective communication structures to support collaboration across rounds. 
However, most existing methods employ a fixed communication topology during inference,
which falls short in many realistic applications where the agents' roles may change \textit{across rounds} due to dynamic adversary, task progression, or time-varying constraints such as communication bandwidth.
In this paper, we propose addressing this issue through TodyComm, a \textbf{t}ask-\textbf{o}riented \textbf{dy}namic \textbf{comm}unication algorithm.
It  produces behavior-driven collaboration topologies that adapt to the dynamics at each round,
optimizing the utility for the task through policy gradient.
Experiments on five benchmarks demonstrate that, under both dynamic adversarial settings and communication budget constraints, TodyComm achieves superior task performance while maintaining token efficiency, scalability, and strong generalizability across varying adversarial conditions.
\end{abstract}

\section{Introduction}
\vspace{-0.8em}
\label{sec:intro}
Large language model (LLM)-based multi-agent systems (MAS) \citep{li2023camel, hong2023metagpt, zhuge2024gptswarm, wu2024autogen, liang2024encouraging, zou2025latent} have emerged as a powerful paradigm for collaborative problem solving. 
Within this research area, multi-agent, multi-round collaboration \citep{du2023improving, liu2024dynamic} constitutes an important interaction mechanism, 
in which agents iteratively exchange information across multiple rounds and refine their responses before producing a final solution. 
Such iterative collaboration enables the integration of diverse reasoning perspectives, supports error correction over time, 
and progressively improves solution quality, making it particularly effective for complex reasoning, planning, and decision-making tasks \citep{zhangcut, parmar-etal-2025-plangen, fan2025evomem, zhaosirius}.

Recent research on multi-round LLM-based multi-agent systems (MAS) has explored a range of methods for enabling effective collaboration across multiple rounds of interaction. 
A prominent line of work focuses on communication topology design, where agent interactions are modeled as graphs that are either predefined or optimized to improve collaboration \citep{ zhangcut,zhangg, zhang2025safesieve,wang-etal-2025-agentdropout,  li2025adaptive}. 
However, these approaches rely on fixed communication graphs \textit{at inference time}, which prevents adaptation to evolving agent behaviors (e.g., historical response, neighboring info and agent-related info, etc) and consequently degrades the task performance, particularly in dynamic settings.

Complementary to topology design, a growing body of work has investigated the security and robustness of multi-round LLM-based MAS. 
Existing approaches \citep{wang-etal-2025-g, miao2025blindguard} primarily rely on graph-based anomaly detection, while \citet{zhou2026resmas} proposes resilient MAS architectures that tolerate agent corruption. 
However, \citet{zhou2026resmas} still depends on a fixed communication graph, and \citet{wang-etal-2025-g} and \citep{miao2025blindguard} do not incorporate task-level feedback (e.g., utility or task accuracy), which can lead to less effective interactions.
Moreover, \citet{wang-etal-2025-g} requires adversity labels during training, which is impractical in real-world settings.

\begin{table*}[t]
\centering
\footnotesize
\renewcommand{\arraystretch}{1}  
\captionsetup{skip=2pt}
\caption{Comparing and contrasting our method with baseline methods. 
``$\ast$'': Edge selection per round indicates that graph structures vary across rounds due to changes in the underlying graph distribution, rather than stochastic sampling from an \textit{invariant} set of edge weights. 
``$\ast\ast$'': Edge selection per round in training'' means the algorithm changes the edges during training, which in turn \textbf{actively impacts the training data}. 
``-'' denotes the absence of adversarial setting in the experiment. ``N.A.'' indicates that no optimization is involved.}
\label{tab:methods}
\begingroup
\setlength{\tabcolsep}{6pt}
\begin{tabular}{l|cccccc}
\toprule
 & \begin{tabular}[c]{@{}c@{}}Edge selection \\ $\textbf{per round}^{\ast}$ \\ in \textbf{inference}\end{tabular}
 & \begin{tabular}[c]{@{}c@{}}Edge selection  \\ $\textbf{per round}^{\ast}$ \\ in $\textbf{training}^{\ast\ast}$\end{tabular}
 & \begin{tabular}[c]{@{}c@{}}Adversary \\ label \textbf{not}\\  required \end{tabular}
 & \begin{tabular}[c]{@{}c@{}}Parametric edge \\ selector based on \\ agent behavior \end{tabular} 
 & \begin{tabular}[c]{@{}c@{}}Task-driven \\ edge \\ selection \end{tabular} 
 & \begin{tabular}[c]{@{}c@{}}Optimization \\ mode\end{tabular} 
 \\
\midrule
G-Designer   & \textcolor{darksalmon}{\XSolidBrush}  & \textcolor{darksalmon}{\XSolidBrush} & \textcolor{green(pigment)}{\large{$\boldsymbol{\checkmark}$}} & \textcolor{darksalmon}{\XSolidBrush} & \textcolor{green(pigment)}{\large{$\boldsymbol{\checkmark}$}}  &RL + aux. loss \\
AgentPrune   & \textcolor{darksalmon}{\XSolidBrush}  & \textcolor{darksalmon}{\XSolidBrush} & \textcolor{green(pigment)}{\large{$\boldsymbol{\checkmark}$}} & \textcolor{darksalmon}{\XSolidBrush} & \textcolor{green(pigment)}{\large{$\boldsymbol{\checkmark}$}} & RL + aux. loss \\
AgentDropout & \textcolor{darksalmon}{\XSolidBrush}   & \textcolor{green(pigment)}{\large{$\boldsymbol{\checkmark}$}} & {\textcolor{gray}{\small\bfseries{-}}} & \textcolor{darksalmon}{\XSolidBrush} & \textcolor{green(pigment)}{\large{$\boldsymbol{\checkmark}$}} & RL + aux. loss \\
SafeSieve  & \textcolor{darksalmon}{\XSolidBrush}  & \textcolor{darksalmon}{\XSolidBrush} & \textcolor{green(pigment)}{\large{$\boldsymbol{\checkmark}$}} & \textcolor{darksalmon}{\XSolidBrush} & \textcolor{green(pigment)}{\large{$\boldsymbol{\checkmark}$}} & N.A. \\
AGP       & \textcolor{darksalmon}{\XSolidBrush}  & \textcolor{darksalmon}{\XSolidBrush} & {\textcolor{gray}{\small\bfseries{-}}} & \textcolor{darksalmon}{\XSolidBrush} & \textcolor{green(pigment)}{\large{$\boldsymbol{\checkmark}$}} &Supervised\\
G-safeguard & \textcolor{green(pigment)}{\large{$\boldsymbol{\checkmark}$}}  & \textcolor{darksalmon}{\XSolidBrush} & \textcolor{darksalmon}{\XSolidBrush} & \textcolor{green(pigment)}{\large{$\boldsymbol{\checkmark}$}}  & \textcolor{darksalmon}{\XSolidBrush} &Supervised\\
BlindGuard   & \textcolor{green(pigment)}{\large{$\boldsymbol{\checkmark}$}}  & \textcolor{darksalmon}{\XSolidBrush} & \textcolor{green(pigment)}{\large{$\boldsymbol{\checkmark}$}} & \textcolor{green(pigment)}{\large{$\boldsymbol{\checkmark}$}} & \textcolor{darksalmon}{\XSolidBrush} & Unsupervised \\
ResMAS       & \textcolor{darksalmon}{\XSolidBrush} & \textcolor{darksalmon}{\XSolidBrush} & \textcolor{green(pigment)}{\large{$\boldsymbol{\checkmark}$}} & \textcolor{green(pigment)}{\large{$\boldsymbol{\checkmark}$}} & \textcolor{green(pigment)}{\large{$\boldsymbol{\checkmark}$}} & SFT + GRPO\\
DyTopo       & \textcolor{green(pigment)}{\large{$\boldsymbol{\checkmark}$}} & \textcolor{darksalmon}{\XSolidBrush} & - & \textcolor{darksalmon}{\XSolidBrush} & \textcolor{darksalmon}{\XSolidBrush} & N.A.\\
Graph-GRPO       & \textcolor{darksalmon}{\XSolidBrush} & \textcolor{darksalmon}{\XSolidBrush} & - & \textcolor{darksalmon}{\XSolidBrush} & \textcolor{green(pigment)}{\large{$\boldsymbol{\checkmark}$}} & GRPO\\
\textbf{TodyComm}       &  \textcolor{green(pigment)}{\large{$\boldsymbol{\checkmark}$}} & \textcolor{green(pigment)}{\large{$\boldsymbol{\checkmark}$}} & \textcolor{green(pigment)}{\large{$\boldsymbol{\checkmark}$}} & \textcolor{green(pigment)}{\large{$\boldsymbol{\checkmark}$}} & \textcolor{green(pigment)}{\large{$\boldsymbol{\checkmark}$}} & RL \\
\bottomrule
\end{tabular}
\endgroup
\vspace{-2em}
\end{table*}
\vspace{-0.2em}

In this paper, we are motivated by scenarios of LLM-based MAS where the agents' communication graph needs to be adapted \textit{across rounds}.
For example, a task must be accomplished through multiple stages,
each of which may require a distinct graph.
The communication channels may impose time-varying constraints on bandwidth.
In a dynamically adversarial setting,
agents may randomly transition to adversarial behavior at unknown rounds, while retaining their original roles and producing plausibly structured but misleading analyses to other agents. 
It reflects realistic collaborative scenarios, where participants may gradually become unreliable due to misunderstanding, bias, fatigue, or external incentives, rather than through explicit role changes.

Despite its importance, learning dynamic communication structures in multi-agent, multi-round settings remains challenging.
Directly optimizing discrete inter-agent graphs across rounds induces a combinatorial action space.
Moreover, determining how communication structures should evolve based on agents’ behaviors in previous rounds is itself non-trivial, as agent outputs are unstructured, stochastic, and context-dependent.

To address these challenges, we propose TodyComm, a \textbf{t}ask-\textbf{o}riented framework that \textbf{dy}namically learns behavior-driven \textbf{comm}unication structures to coordinate LLM-based agents across multiple interaction rounds during both training and inference.
TodyComm formulates the multi-agent multi-round interaction process as a Markov decision process \citep{puterman1990markov}, 
where each episode corresponds to solving a query over $T$ interaction rounds, 
and actions correspond to the sequence of inter-agent communication graphs across rounds together with the final decision graph.


The communication graphs are learned as constrained topologies by parameterizing pairwise edge influences through embeddings of the endpoint nodes, subject to edge masks, an acyclicity constraint, and node-wise degree budgets.
The node embeddings are modeled using gated recurrent networks \citep[GRN,][]{chung2014empirical}, 
allowing temporal dependencies to be captured across rounds while adapting to each agent’s historical contributions, reliability, and interaction context.
The graph structures are optimized using reinforcement learning \citep[RL,][]{williams1992simple}, 
with task utility serving as the reward signal.

We conduct extensive empirical evaluations on five diverse benchmarks, in addition to budget control, scalability and generalization, demonstrating that TodyComm consistently outperforms baseline methods and highlighting the effectiveness of learning dynamic communication topologies from agent behaviors in multi-round LLM-based MAS. 

To the best of our knowledge, TodyComm is the first framework to learn task-oriented, behavior-driven dynamic communication topologies in multi-round LLM-based multi-agent systems,
enabling task-adaptive and round-adaptive communication during both training and inference.

\vspace{-1em}
\section{Related Work on Graph Learning for Multi-agent Communication}
\vspace{-0.8em}

In conventional multi-agent RL, methods such as TarMAC~\citep{das2019tarmac}, IC3Net~\citep{singhlearning}, SchedNet~\citep{kimlearning}, ToM2C~\citep{wangtom2c}, and MAGIC~\citep{niu2021multi} learn communication patterns via attention, gating, or pairwise edge probabilities, enabling agents to modulate their interactions.
LSC \citep{sheng2023learning} further introduces dynamic hierarchical grouping, 
but its topology remains restricted and not optimized as a general graph.
Although these approaches can induce dynamic interaction patterns, 
they do not explicitly optimize communication topology under structural constraints such as bounded node degree.
Moreover, many rely on numeric-valued messages to enable differentiation or weighted combinations, 
which is infeasible for language agents,
especially closed-source LLMs.

In LLM-based MAS, early works assume typically rely on fixed or predefined interaction structures, where agents communicate under static coordination patterns \citep{jiang2023llm, du2023improving, qianscaling, chanchateval, wu2024autogen}. Subsequent methods introduce explicit communication graphs, typically by assigning learnable edge weights over a fixed set of potential connections. For example, GPTSwarm \citep{zhuge2024gptswarm} and Graph-GRPO \citep{cang2026graph} construct communication structures through behavior-\textit{independent} parameters and they focus on the single round communication.

More generally, multi-round MAS often adopt a spatial–temporal graph formulation, where communication is defined over a fixed graph and refined via pruning or sampling. Representative examples include G-Designer \citep{zhangg}, which constructs graphs from predefined connections, and pruning-based approaches such as AgentPrune \citep{zhangcut}, AgentDropout \citep{zhang2025safesieve}, and SafeSieve \citep{wang-etal-2025-agentdropout}, which remove edges or agents during training. These methods rely on graph structures that \textit{remain fixed at inference time}, and the observed dynamism arises primarily from stochastic sampling rather than behavior-driven structural adaptation.
DyTopo \citep{lu2026dytopo} induces a dynamic graph at inference time; however, it remains a rule-based method based on semantic matching, without learning the communication structure from agent behavior or task-driven utility signals.

Several works also study robustness in LLM-based MAS. Methods such as G-Safeguard \citep{wang-etal-2025-g} and BlindGuard \citep{miao2025blindguard} employ graph-based anomaly detection to identify adversarial agents, while ResMAS \citep{zhou2026resmas} designs resilient architectures under agent corruption. 
But they focus on detection or robustness rather than \textit{cooperation driven by task performance}, 
relying on \textit{fixed} communication topologies.


In contrast, we consider a dynamically adversarial setting, \textbf{where agents may transition to adversarial behavior at unknown rounds \underline{during interaction}}.
The number of adversarial agents is controlled by an attack rate, enabling systematic evaluation of whether a system can adapt its communication topology online in response to evolving agent behaviors.
Unlike existing approaches that fix communication structures at inference time or adapt them without regard to task performance, 
TodyComm produces behavior-driven, round-adaptive, and task-oriented communication graphs.
Table~\ref{tab:methods} compares its properties with state-of-the-art methods,
and its connection with multi-agent influence diagrams \citep[MAID,][]{koller2003multi} is presented in Appendix~\ref{sec:app_maid}.

\vspace{-1.2em}
\section{Preliminaries}
\vspace{-0.8em}
In this paper, we adopt multi-round LLM-based multi-agent framework \citep{zhangcut,zhangg}, in which multiple LLM agents interact over multiple rounds to solve a given query.

\textbf{Multi-round LLM-based MAS}\quad
Given a query \(q\), a set of \(N\) agents \(\{A_1, \ldots, A_N\}\) engage in a multi-round interaction over \(T\) rounds. 
Each agent \(A_i := (\text{base}_i, \text{role}_i)\) is instantiated with an underlying LLM model and an assigned role at initialization.
At round \(t = 1\), the agents independently generate an initial response \(\{o_i^{1}\}_{i=1}^N\) to the query without inter-agent communication. 
For rounds \(t \in [2, \ldots, T]\), agents iteratively update their responses by communicating and collaborating with other agents, producing updated outputs \(o_i^{t} := \{\text{sol}_i^t, \text{ana}_i^t\}\), 
where $\text{sol}_i^t$ is the answer such as a one-hot vector for multiple choice,
and $\text{ana}_i^t$ is the embedding of the agent's analysis.

After \(T\) rounds, the final answer is obtained by aggregating the agents’ outputs from the last round \(\{o_i^{T}\}_{i=1}^N\). 
For convenience, we refer to the decision phase as round $T+1$ in the following sections.

\textbf{Multi-round MAS as Graphs}\quad
We model multi-agent interactions over \(T\) rounds using spatial graphs and distinguish two types of structures. 
Inter-agent communication at round \(t \in [2, \ldots, T]\) is modeled as a directed acyclic graph (DAG) \(\mathcal{G}_C^{t} = (\mathcal{V}_C^{t}, \mathcal{E}_C^{t})\), 
where \(\mathcal{V}_C^{t}\) denotes the set of participating agent nodes at round \(t\), 
with \(|\mathcal{V}_C^{t}| \leq N\), and \(\mathcal{E}_C^{t}\) is the set of directed edges specifying message-passing relationships among agents -
messages serve as prompts for the recipient.

The aggregation of agent responses into a final decision is modeled using a separate decision graph \(\mathcal{G}_D = (\mathcal{V}_D, \mathcal{E}_D)\), 
where \(\mathcal{V}_D\) denotes the set of agent nodes contributing to the final decision, with \(|\mathcal{V}_D| \leq N\), 
and \(\mathcal{E}_D\) consists of edges connecting selected agents to the decision node.

\textbf{Multi-agent Communication} \quad 
Given a query \(q\), information is propagated through agent nodes in a topological order, ensuring that each node is executed only after all its predecessors are resolved. Consequently, the inter-agent communication graph \(\mathcal{G}_C^t\) is restricted to be a DAG.
At round \(t \in [2,\ldots, T]\), each node \(v_i\) receives the query \(q\) along with messages from its parent nodes \(\mathcal{N}^t(v_i)\). Each parent \(v_j \in \mathcal{N}^t(v_i)\) provides a message
$
m_j^t = f_m(o_j^t, \text{role}_j),
$
where \(f_m\) is a prompt template function and \(\text{role}_j\) denotes the role assigned to agent \(j\) at initialization. 
Node \(v_i\) then computes its output $o_i^t$ based on its prompt \(p_i^t\),
which is defined using another prompt template function \(f_v\):
$p_i^t \sim f_v \Bigl(
q,\ 
\bigcup\nolimits_{v_j \in \mathcal{N}^t(v_i)} m_j^t
\Bigr)
$.

The final answer to query \(q\), denoted by \(\text{ans}_q\), is produced by a predefined decision function \(f_d\) that aggregates the outputs of selected agents connected to the decision node \(v_d\):
$\text{ans}_q \sim f_d \Bigl(
\bigcup\nolimits_{v_j \in \mathcal{N}^{T+1}(v_d)} o_j^T
\Bigr)$.
The complete execution procedure of the communication graph is summarized in Algorithm~\ref{algo:graph_exec} in Appendix~\ref{sec:app_graph_execution}.

\begin{figure*}[!t] 
    \centering
    \includegraphics[width=\textwidth]{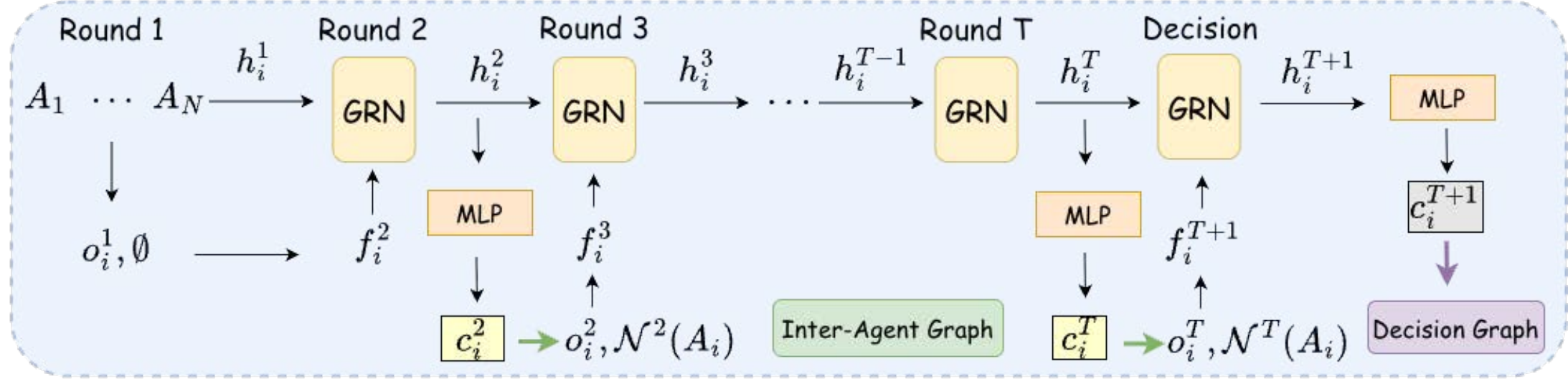}
    \caption{Overview of TodyComm training workflow.
\(\{A_i\}_{i=1}^N\) denote the agents, and \(\{o_i^t\}_{i=1}^N\) their outputs at round \(t\). 
The neighborhood of agent \(A_i\) at round \(t\) is denoted by \(\mathcal{N}^t(A_i)\), initialized as \(\emptyset\) at the first round. 
Node features \(\{f_i^t\}_{i=1}^N\) are constructed from each agent’s output and those of its neighbors from the previous round. Node embeddings \(\{h_i^t\}_{i=1}^N\) are computed via a GRN, with \(h_i^1=\emptyset\). 
Node potentials \(\{c_i^t\}_{i=1}^N\) are then derived from the embeddings,
and contribute to the construction of the communication graph at each round \(t \in [2, T+1]\).
    }
    \label{fig:framework}
    \vspace{-1.2em}
\end{figure*}

\section{Method}
\vspace{-0.8em}
We formulate this multi-agent collaboration as a Markov decision process \citep[MDP,][]{puterman1990markov}, where each episode corresponds to solving a query through $T$ interaction rounds.
The \textbf{actions} correspond to the inter-agent communication graphs $\mathcal{G}_C^t$ for $t \in [2, \ldots, T]$, together with the decision graph \(\mathcal{G}_D\) at round \(T+1\). 
The inter-agent \emph{directed} communication graph \(\mathcal{G}_C^t\) at round \(t\) is represented by binary edge indicators \(\mathcal{E}_C^t \in \{0,1\}^{\mathcal{V}_C^{t} \times \mathcal{V}_C^{t}}\). 
The decision graph \(\mathcal{G}_D\) is represented by binary node indicators \(\mathcal{E}_D \in \{0,1\}^{\mathcal{V}_D}\) with associated decision weights $\{w_i\}_{i=1}^{|\mathcal{V}_D|}$.

The \textbf{state} of the MDP is the response of all agents $\{o_i^t\}_i$,
along with all the relevant information such as $q$ and previous $\{o_i^{\tau}\}_i$ ($\tau < t$).
The \textbf{transition} probability concerns the distribution of the next-round response $\{o_i^{t+1}\}_i$ under the given graph (i.e., action) at $t$. 
The \textbf{reward}, in the simplest form, is the supervised loss on the response of the decision node at round $T+1$,
i.e., the task utility on the final answer.

We denote the policy of constructing the communication graph at round \(t\) by \(\pi_{t}\), and the policy of the decision graph by \(\pi_{T+1}\). 
Their specific construction and parametrization will be detailed in Section~\ref{sec:inter_agent_opt}.
Our goal is to learn policy $\pi_{t}$ that induces communication and decision structures maximizing the expected utility of query $q$ through multiple interaction rounds. 
Given a query $q$, we denote its utility by $U_q(\tau)$,
where $\tau$ is the trajectory of states and actions. 
Parameterizing the policy by $\theta$,
the reinforcement learning (RL) objective is then defined as:
\begin{align}
\label{eq:graph_optimization}
\mathcal{L}(\theta)
=
\max_{\theta \in \Theta}\;
\mathbb{E}_{q \sim \mathcal{Q}}\;
\mathbb{E}_{\tau \sim \pi_{t, \theta}}
\big[ U_q(\tau) \big].
\end{align}

\vspace{-1.6em}
\subsection{Dynamic Edge Optimization}
\vspace{-0.7em}
In our \textbf{dynamically adversarial setting}, agents may turn adversarial at any round, so the communication graph must continuously adapt to preserve cooperation among reliable agents.
Since our policy is the communication graph constructed at each round, the key challenge is therefore \emph{edge selection}: which directed edges to form at each round to form a DAG.
Drawing on the standard graph-learning practice of predicting edges from node features \citep{kipf2016variational, hoff2002latent}, 
we learn per-round node representations that capture each agent’s behavior and messages sent by neighboring agents.
These node features then drive the sequential edge selection decisions,
determining the communication edges among agents and the connections to the decision node.

\vspace{-0.7em}
\subsection{Inter-agent Edges Optimization}
\label{sec:inter_agent_opt}
\vspace{-0.5em}

\textbf{Node Embeddings} \quad
To better capture multi-round interactions and the temporal flow of information across rounds, 
we employ a many-to-one Gated Recurrent Network \citep[GRN,][]{chung2014empirical} to learn per-round node embeddings \(\{h_i^t\}_{t=2}^{T+1}\).
Since each agent produces outputs at every round, the GRN aggregates information across all \(T\) rounds into a unified representation that captures both individual behavior and relational context with neighboring agents.
Following the standard formulation of GRNs, at each round \(t \in [2, \ldots, T+1]\), we denote by \(h_i^{t-1}\) that summarizes information from previous rounds
and by \(f_i^t\) the current node feature of agent \(i\), which is constructed in the embedding space from the agent’s own output \(o_i^{t-1}\) and the outputs of its neighboring agents \(\mathcal{N}^{t-1}(v_i)\) at round \(t-1\).
Let $\{h^1_i\}_{i=1}^N = \emptyset$.
The node embedding at round \(t\) is updated as $h_i^t = \mathrm{GRN}(h_i^{t-1}, f_i^t)$.
The architecture of GRN is illustrated in Figure~\ref{fig:framework},
with pseudo-code available in Algorithm~\ref{algo:node_representation_generation}.

\textbf{Node Feature Construction} \quad
At each round \(t\), we construct the node features \(f_i^t\) for each agent \(i\) based on the agent’s own output \(o_i^{t-1}\) and the outputs of its neighboring agents \(\mathcal{N}^{t-1}(v_i)\) from the previous round.
\(f_i^t\) serves as the input to the GRN. 
To capture both individual behavior and inter-agent interactions, 
we construct it by concatenating three complementary components: self information \(\mathbf{s}_i^{t-1}\), neighborhood information \(\mathbf{n}_i^{t-1}\), and difference information \(\mathbf{d}_i^{t-1}\):
$f_i^t := \big[ \mathbf{s}_i^{t-1} \Vert \mathbf{n}_i^{t-1} \Vert \mathbf{d}_i^{t-1} \big]$.
Detailed feature definitions are deferred to Appendix~\ref{app:definition_node_features}.


\textbf{Edge Construction based on Node Embeddings} \quad
With node embeddings $h_i^t$ at hand,
we could apply the common approach to stochastically determine whether an edge exists between $i$ and $j$ based on $h_i^t$ and $h_j^t$ \citep{kipf2016variational,hoff2002latent}.
However, there are multiple constraints we need to respect:
i) the resulting graph must be a DAG;
ii) extrinsic constraints that forbid an edge;
iii) node-wise degree budgets on outgoing and incoming edges, denoted as $B_{\text{out}}$ and $B_{\text{in}}$.
Denote as $\Gcal_a$ all admissible graphs satisfying these constraints,
and denote as $\Gcal_a(A)$ when the nodes are restricted to a subset $A$.
We next propose a principled construction based on ``projection'' to $\Gcal_a$.

For a fixed interaction history $H_t$, the conditional value of a communication graph can be decomposed into first-order edge contributions plus higher-order edge-interaction terms. 
We assume that, over the sparse budgeted DAGs considered by TodyComm, 
the total higher-order interaction is bounded by $\beta_t$.
\begin{assumption}[Approximate edge-additive task utility at graph level]
\label{assump:edge-additive}
There is a constant $\beta_t \ge 0$ and a latent task-level edge-influence matrix 
$\Delta_t = \{\Delta_{ij}^t\}$
such that for every feasible graph $G$,

\vspace{-2em}
\begin{align}
\label{eq:edge-additive}    
\left|
    F_t(G)-F_t(\emptyset)
    -
    \sum\nolimits_{(i,j)\in E(G)}\Delta_{ij}^t
    \right|
    \le
    \beta_t,
\where
    F_t(G) := \EE[U\mid H_t, G_C^t=G].
\end{align}
Here $F_t(G)$ is the conditional final-task utility of using graph $G$ at round $t$, given the history $H_t$.
\end{assumption}
\begin{assumption}[Scalar-potential edge influence]
\label{assump:scalar-potential}
We follow existing literature \citep{chung2002average,chatterjee2011random,lovasz2010regularity} to assume that the edge influence can be effectively modeled via scalar \textbf{node potentials} $u_i^t\in[0,1]$.
That is, 

\vspace{-1.2em}
\begin{equation}
    \Delta_{ij}^t = u_i^t u_j^t + \xi_{ij}^t \quad \text{for all } i, j \in [1, N],
    \qquad  \text{and}  \quad
    |\xi_{ij}^t|\le \nu_t \text{ (for some } \nu_t > 0).
    \label{eq:scalar-potential}
\end{equation}
\end{assumption}
Note that $u_i^t$ should not be interpreted only as ``agent $i$ is reliable.''
Rather, it encodes a communication potential: how useful agent $i$ is as an endpoint of communication at round $t$. The graph still depends on pairwise products, edge feasibility, acyclicity, and degree budgets.

In practice, we estimate $u^t_i$ through an MLP based on the node embeddings $h^t_i$.
Denote the estimate as $\cwhat_i^t \in [0, 1]$, 
and call it empirical node potential.
We make a standard assumption on the boundedness of the associated total statistical error, model approximation error, and optimization error.
\begin{assumption}[Uniform scalar-potential estimation]
\label{assump:credit-estimation}
At round $t$, the learned potentials satisfy
\begin{equation}
    \max_i |\cwhat_i^t-u_i^t|\le \alpha_t
    \label{eq:credit-estimation}
    \quad
    \text{with probability at least }
    1-\delta_t.
\end{equation}
\end{assumption}

\vspace{-1.2em}
Given $F_t$, $\Delta_t$, or $\cwhat_i^t$,
Assumption~\ref{assump:edge-additive} motivates defining the optimal graphs by
\begin{align}
\label{eq:def_G_star}
    G_{F,t}^{\star}
    \in
    \argmax_{G\in \Gcal_a} F_t(G),
    \quad \ 
    G_{\Delta,t}^{\star}
    \in
    \argmax_{G\in \Gcal_a}
    \sum_{(i,j)\in E(G)} \!\!\! \Delta_{ij}^t,
    \quad \ 
    \widehat G_t
    \in
    \argmax_{G\in \Gcal_a}
    \sum_{(i,j)\in E(G)} \!\!\! \cwhat_i^t \cwhat_j^t.
\end{align}
Existing methods such as \citep{zhuge2024gptswarm,zhangg} solve the discrete optimization approximately by greedy search:
starting from an empty graph $\emptyset$,
enumerate all possible edges in a descending order of $\Delta^t_{ij}$ (or $\cwhat_i^t \cwhat_j^t$),
and add an edge to the graph if, and only if, 
the resulting graph remains admissible in $\Gcal_a$.
We adopt the same method, and the pseudo-code is in Algorithm~\ref{algo:graph_construction}.

While RL is not directly formulated to estimate $u^t_i$,
it could be interpreted as a bi-level optimization where the inner optimization is Equation~\eqref{eq:def_G_star},
and the outer (RL) optimization estimates the best parameters $\cwhat^t_i$ that allow the inner optimization to recover the optimal graph.

Basing edge influence $\Delta^t_{ij}$ on scalar product may yield a large value of $\cwhat^t_i \cwhat^t_j$ when $\cwhat^t_j$ is very high while $\cwhat^t_i$ is quite low.
To introduce another layer of safeguard,
we first screen all nodes to conservatively preclude those with a low value of $c^t_i$;
see Algorithm \ref{algo:screening}.
Denoting the remaining set of nodes as $\widehat A_t$,
we use the true $\Delta^t$ to measure the utility lost (i.e., regret) due to errors in node screening:
\begin{align}
\label{eq:participation-loss}
\Pi_t(\widehat A_t)
    :={}
    \max_{G\in \Gcal_a}
    \sum\nolimits_{(i,j)\in E(G)}\Delta_{ij}^t
    \quad - \ \  
    \max_{G\in \Gcal_a(\widehat A_t)}
    \sum\nolimits_{(i,j)\in E(G)}\Delta_{ij}^t.
\end{align}
Our main theoretical result bounds the regret in the ultimate utility via the screening regret:
\begin{theorem}[Scalar-potential graph projection]
\label{thm:scalar-graph-projection}
Fix a round $t$. Let $\widehat A_t$ be the candidate node set resulting from the screening, and let $\widehat G_t$ be the score-projected graph in Equation~\eqref{eq:def_G_star}. 
Under Assumptions~\ref{assump:edge-additive},~\ref{assump:scalar-potential}, and~\ref{assump:credit-estimation}, with probability at least $1-\delta_t$,
\begin{equation}
    F_t(G_{F,t}^{\star}) - F_t(\widehat G_t)
    \le
    \Pi_t(\widehat A_t)
    +
    2 (2 \alpha_t + \alpha_t^2+\nu_t) \cdot \max\nolimits_{G\in\Gcal_a}|E(G)| 
    +
    2\beta_t.
    \label{eq:main-bound}
\end{equation}
Intuitively,
graph regret
$\le$
screening regret
$+$
product-score estimation loss
$+$
edge-additivity loss.
\end{theorem}
The proof is relegated to Appendix~\ref{sec:proof}.
Interestingly, false positive and false negative in screening impact the final graph in very different ways.
We defer their theoretical analysis to Corollary~\ref{cor:false-positive} and \ref{cor:false-negative}.

To summarize,
the communication graph is learned as a constrained topology by parameterizing pairwise edge influences using scalar potentials associated with the endpoint nodes, 
subject to edge masks, a DAG constraint, and node-wise in-degree and out-degree budgets.
\vspace{-0.7em}
\subsection{Decision Edges Optimization}
\vspace{-0.8em}
After $T$ rounds of communication,
we compute the final node potentials by performing one additional GRN update that aggregates all \(T\) interaction rounds, as described in Algorithm~\ref{algo:node_representation_generation}.
Note that agents may still transition to adversarial behavior at round \(T\).
The final normalized decision weights \(\{w_i\}_{i=1}^N\) are obtained from the node potentials \(\{c_i^{T+1}\}_{i=1}^N\), 
preceded with a node screening.

We finally employ REINFORCE \citep{williams1992simple} to optimize the policy for multi-agent, multi-round collaboration.
The parameters include those in the GRN, 
and the MLP that maps embedding $h^t_i$ to potential $\cwhat^t_i$.

\vspace{-0.8em}
\section{Experiment}
\vspace{-0.8em}
\label{sec:experiment}
\textbf{Dataset} \quad
We evaluate TodyComm on five benchmarks over multiple reasoning domains:
(i) \textit{Commonsense Reasoning}: MMLU \citep{hendrycksmeasuring}, ARC-Challenge \citep{clark2018think};
(ii) \textit{Mathematical Reasoning}: GSM8K \citep{cobbe2021training};
(iii) \textit{Scientific Reasoning}: OpenBookQA \citep{mihaylov2018can}, MedQA \citep{yang2025llm}.

\textbf{Dynamically Adversarial Setting} \quad
We consider a dynamically adversarial setting in which agents may transition to adversarial behavior at unknown rounds without announcing. Unless stated otherwise, 6 agents collaborate over 4 rounds, with adversarial transitions occurring at round 3 or 4. Attack rates of $<50\%$, $=50\%$, and $>50\%$ correspond to 2, 3, and 4 adversarial agents, respectively. Deviations from these settings will be explicitly noted.

\textbf{Impact of Adversarial Agents} \quad
Adversarial agents produce incorrect answers with plausible but misleading analyses, following the prompt in Appendix~\ref{app:adversarial_prompts}. 
This is a \textit{targeted} attack, as an agent deliberately prompts its LLM to output a \textit{specific} wrong answer.
Before training, we evaluated the impact of adversarial communication,
where a reliable agent receives message from a single adversarial agent. 
As shown in Table~\ref{tab:attack_single}, even a single adversarial sender causes substantial accuracy drops:
47.42\% on MMLU, 35.30\% on ARC-Challenge, 32.91\% on GSM8K, 61.11\% on OpenBookQA, and 22.00\% on MedQA. 
Table~\ref{tab:adv_scaling_among_six} further analyzes adversarial scaling by varying the number of adversarial agents among six senders and measuring the impact on the receiving agent.

\begin{table*}[ht!]
\centering
\caption{Performance comparison with baselines across five benchmarks. 
\(\boldsymbol{\downarrow}\) next to a dataset name indicates performance degradation of a reliable agent after receiving a message from an adversarial agent (see also Table~\ref{tab:attack_single}). 
All results are reported in percentages, 
with the best results shown in bold.
}
\label{tab:main_result}
\vspace{-6pt}
\resizebox{\textwidth}{!}{ 
\begin{tblr}{
  colspec = {l | l *{15}{c}}, 
  cell{1}{1,2} = {r=2}{c}, 
  cell{1}{3,6,9,12,15} = {c=3}{c}, 
  cell{3,5,7,9,11}{1} = {r=2}{l}, 
  cell{3,5,7,9}{1} = {r=2}{l},   
  cell{11}{1} = {r=2}{l},        
  row{3,5,7,9,11} = {abovesep=2pt, belowsep=1pt}, 
  row{4,6,8,10,12} = {abovesep=1pt, belowsep=2pt}, 
  hline{1} = {0.08em}, 
  hline{3} = {0.05em}, 
  hline{13} = {0.08em}, 
  hline{2} = {3-17}{0.03em},
  hline{5,7,9,11} = {-}{0.02em}, 
  vline{3,6,9,12,15} = {1-12}{0.05em},
  row{1-2} = {font=\bfseries},
  colsep = 3pt,
}
Methods & Metrics & MMLU ($\downarrow$ 47.2\%) & & & ARC-C ($\downarrow$ 35.30\%) & & & GSM8K ($\downarrow$ 32.91\%) & & & OpenBookQA ($\downarrow$ 61.11\%)& & & MedQA ($\downarrow$ 22.00\%) & & \\
 & & $<50\%$ & $=50\%$ & $>50\%$ & $<50\%$ & $=50\%$ & $>50\%$ & $<50\%$ & $=50\%$ & $>50\%$ & $<50\%$ & $=50\%$ & $>50\%$ & $<50\%$ & $=50\%$ & $>50\%$ \\

\SetCell[r=2]{l} {Random \\ Graph} & Acc & 
{59.92 \\ \hfill \scriptsize $\pm 3.29$} & {45.97 \\ \hfill \scriptsize $\pm 2.72$} & {42.05 \\ \hfill \scriptsize $\pm 3.72$} & 
{79.15 \\ \hfill \scriptsize $\pm 1.51$} & {64.10 \\ \hfill \scriptsize $\pm 3.80$} & {64.88 \\ \hfill \scriptsize $\pm 4.93$} & 
{83.01 \\ \hfill \scriptsize $\pm 0.92$} & {76.33 \\ \hfill \scriptsize $\pm 2.16$} & {68.26 \\ \hfill \scriptsize $\pm 0.69$} & 
{64.33 \\ \hfill \scriptsize $\pm 4.25$} & {42.17 \\ \hfill \scriptsize $\pm 0.58$} & {36.33 \\ \hfill \scriptsize $\pm 3.01$} & 
{58.33 \\ \hfill \scriptsize $\pm 2.02$} & {48.33 \\ \hfill \scriptsize $\pm 1.53$} & {42.00 \\ \hfill \scriptsize $\pm 2.65$} \\
 & Token & 620 & 619 & 610 & 642 & 614 & 628 & 515 & 569 & 553 & 490 & 553 & 550 & 764 & 822 & 784 \\

\hline
\SetCell[r=2]{l} {Complete \\ Graph} & Acc & 
{64.94 \\ \hfill \scriptsize $\pm 2.01$} & {46.63 \\ \hfill \scriptsize $\pm 4.82$} & {41.39 \\ \hfill \scriptsize $\pm 3.94$} & 
{81.17 \\ \hfill \scriptsize $\pm 0.77$} & {67.67 \\ \hfill \scriptsize $\pm 2.42$} & {64.33 \\ \hfill \scriptsize $\pm 2.05$} & 
{81.00 \\ \hfill \scriptsize $\pm 0.87$} & {75.46 \\ \hfill \scriptsize $\pm 0.81$} & {66.34 \\ \hfill \scriptsize $\pm 2.10$} & 
{67.83 \\ \hfill \scriptsize $\pm 2.52$} & {42.50 \\ \hfill \scriptsize $\pm 2.78$} & {37.33 \\ \hfill \scriptsize $\pm 1.26$} & 
{59.67 \\ \hfill \scriptsize $\pm 2.57$} & {43.67 \\ \hfill \scriptsize $\pm 3.69$} & {41.75 \\ \hfill \scriptsize $\pm 2.47$} \\
 & Token & 741 & 768 & 749 & 793 & 785 & 769 & 663 & 696 & 682 & 621 & 655 & 643 & 976 & 961 & 940 \\

\hline
\SetCell[r=2]{l} G-Designer & Acc & 
{64.92 \\ \hfill \scriptsize $\pm 1.72$} & {52.07 \\ \hfill \scriptsize $\pm 1.72$} & {45.97 \\ \hfill \scriptsize $\pm 1.87$} & 
{83.05 \\ \hfill \scriptsize $\pm 0.16$} & {72.91 \\ \hfill \scriptsize $\pm 0.98$} & {67.22 \\ \hfill \scriptsize $\pm 0.47$} & 
{87.23 \\ \hfill \scriptsize $\pm 0.72$} & {81.12 \\ \hfill \scriptsize $\pm 2.26$} & {73.00 \\ \hfill \scriptsize $\pm 4.77$} & 
{73.67 \\ \hfill \scriptsize $\pm 2.09$} & {40.83 \\ \hfill \scriptsize $\pm 2.62$} & {41.00 \\ \hfill \scriptsize $\pm 0.82$} & 
{\textbf{67.00} \\ \hfill \scriptsize $\pm 1.47$} & {48.33 \\ \hfill \scriptsize $\pm 2.09$} & {43.17 \\ \hfill \scriptsize $\pm 0.94$} \\
 & Token & 652 & 684 & 678 & 585 & 617 & 610 & 465 & 509 & 501 & 534 & 566 & 563 & 803 & 841 & 829 \\

\hline
\SetCell[r=2]{l} AgentPrune & Acc & 
{69.28 \\ \hfill \scriptsize $\pm 0.53$} & {55.34 \\ \hfill \scriptsize $\pm 3.63$} & {53.38 \\ \hfill \scriptsize $\pm 2.46$} & 
{86.73 \\ \hfill \scriptsize $\pm 0.79$} & {74.02 \\ \hfill \scriptsize $\pm 0.63$} & {69.79 \\ \hfill \scriptsize $\pm 1.76$} & 
{\textbf{88.73} \\ \hfill \scriptsize $\pm 0.16$} & {80.59 \\ \hfill \scriptsize $\pm 0.82$} & {70.91 \\ \hfill \scriptsize $\pm 0.53$} & 
{73.83 \\ \hfill \scriptsize $\pm 0.24$} & {44.17 \\ \hfill \scriptsize $\pm 1.03$} & {39.17 \\ \hfill \scriptsize $\pm 0.24$} & 
{66.65 \\ \hfill \scriptsize $\pm 0.41$} & {48.50 \\ \hfill \scriptsize $\pm 1.47$} & {45.67 \\ \hfill \scriptsize $\pm 0.85$} \\
 & Token & 361 & 408 & 417 & 301 & 353 & 369 & 345 & 429 & 444 & 388 & 451 & 471 & 520 & 579 & 590 \\

\hline
\SetCell[r=2]{l} {TodyComm \\ (Ours)}  & Acc & 
{\textbf{72.11} \\ \hfill \scriptsize $\pm 0.82$} & {\textbf{68.85} \\ \hfill \scriptsize $\pm 0.31$} & {\textbf{64.71} \\ \hfill \scriptsize $\pm 0.00$} & 
{\textbf{88.52} \\ \hfill \scriptsize $\pm 0.16$} & {\textbf{85.95} \\ \hfill \scriptsize $\pm 1.52$} & {\textbf{81.05} \\ \hfill \scriptsize $\pm 0.88$} & 
{\textbf{88.68} \\ \hfill \scriptsize $\pm 0.45$} & {\textbf{88.10} \\ \hfill \scriptsize $\pm 0.89$} & {\textbf{83.19} \\ \hfill \scriptsize $\pm 0.69$} & 
{\textbf{84.83} \\ \hfill \scriptsize $\pm 1.31$} & {\textbf{84.17} \\ \hfill \scriptsize $\pm 0.24$} & {\textbf{80.50} \\ \hfill \scriptsize $\pm 0.41$} & 
{61.00 \\ \hfill \scriptsize $\pm 0.41$} & {\textbf{60.17} \\ \hfill \scriptsize $\pm 0.47$} & {\textbf{57.50} \\ \hfill \scriptsize $\pm 0.71$} \\
 & Token & 397 & 405 & 414 & 350 & 361 & 366 & 311 & 332 & 346 & 428 & 448 & 459 & 560 & 565 & 566 \\
\end{tblr}
}
\vspace{-0.7em}
\end{table*}


\textbf{Baselines} \quad
We consider four baselines in communication topology design: Random Graph \citep{qianscaling}, Complete Graph \citep{qianscaling}, G-Designer \citep{zhangg} and AgentPrune \citep{zhangcut}.

\textbf{Implementation Details} \quad
We report overall task accuracy (Acc), the average token usage per agent per round for each query (Token).
All experiments were conducted with GPT-4.1-nano under the default temperature setting, unless otherwise specified.
Implementation details, 
including dataset statistics, training/test splits,
and hyperparameter settings, are provided in Appendix~\ref{sec:implementation_details}.

\begin{figure*}[t!]
    \centering
    \begin{subfigure}[t]{0.24\linewidth}
        \centering
        \includegraphics[width=\linewidth]{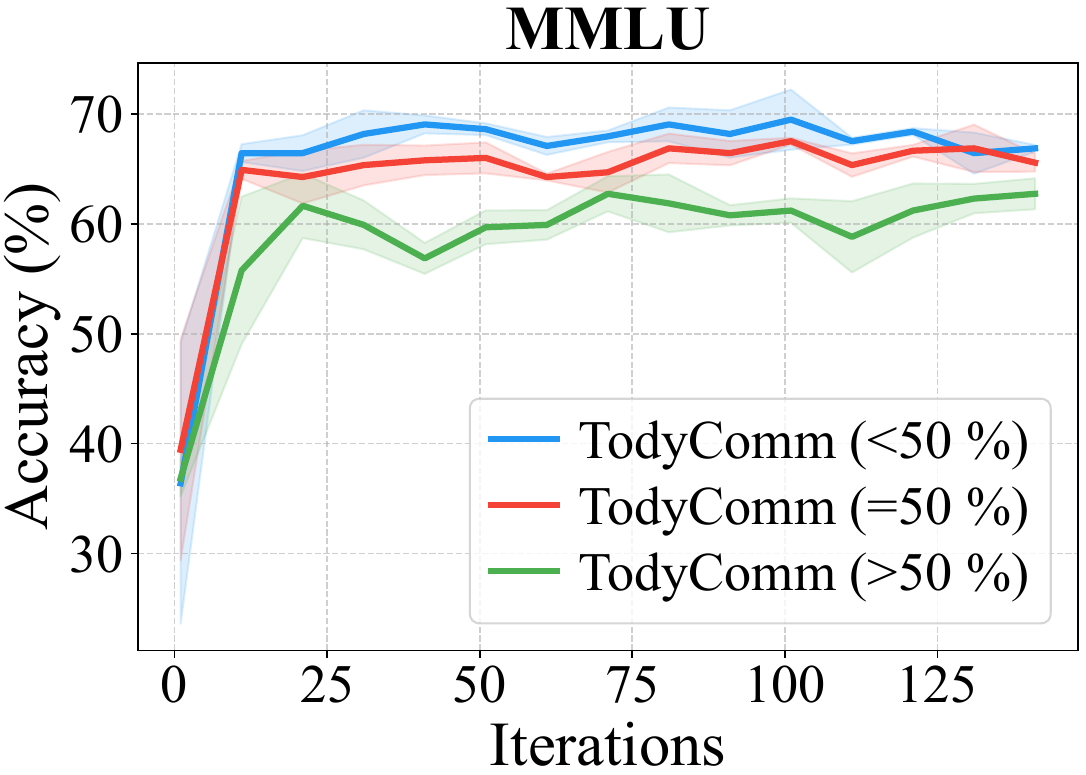}
    \end{subfigure}
    \hfill
    \begin{subfigure}[t]{0.24\linewidth}
        \centering
        \includegraphics[width=\linewidth]{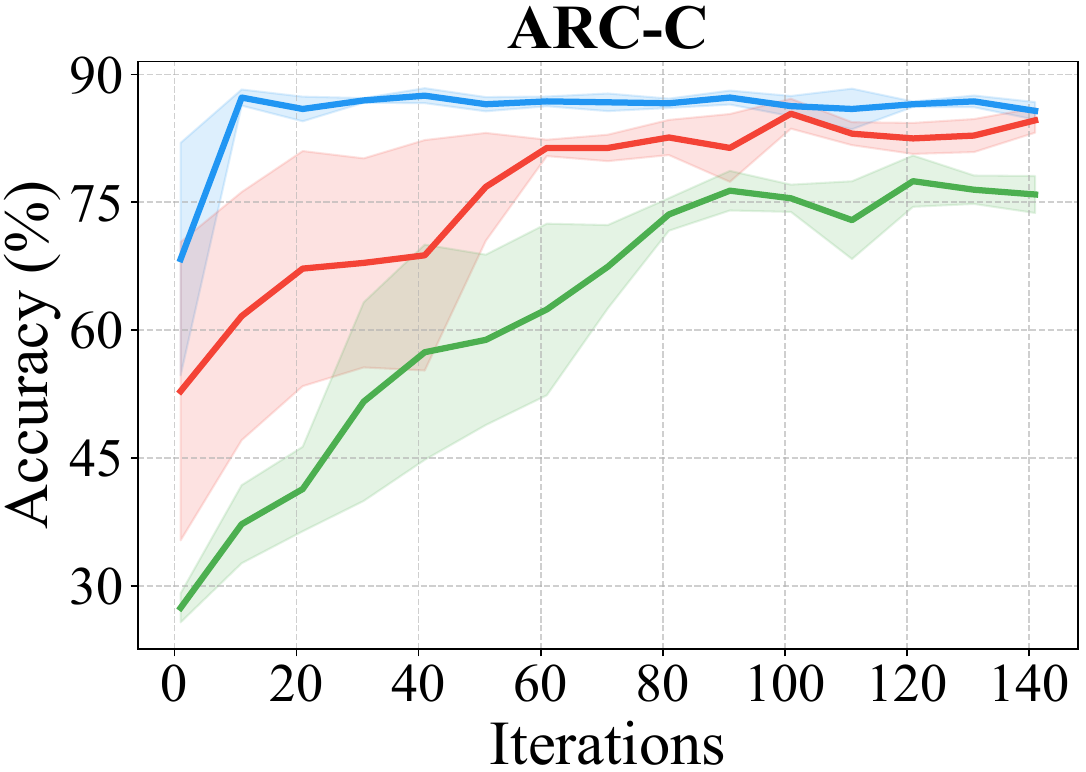}
    \end{subfigure}
    \hfill
    \begin{subfigure}[t]{0.24\linewidth}
        \centering
        \includegraphics[width=\linewidth]{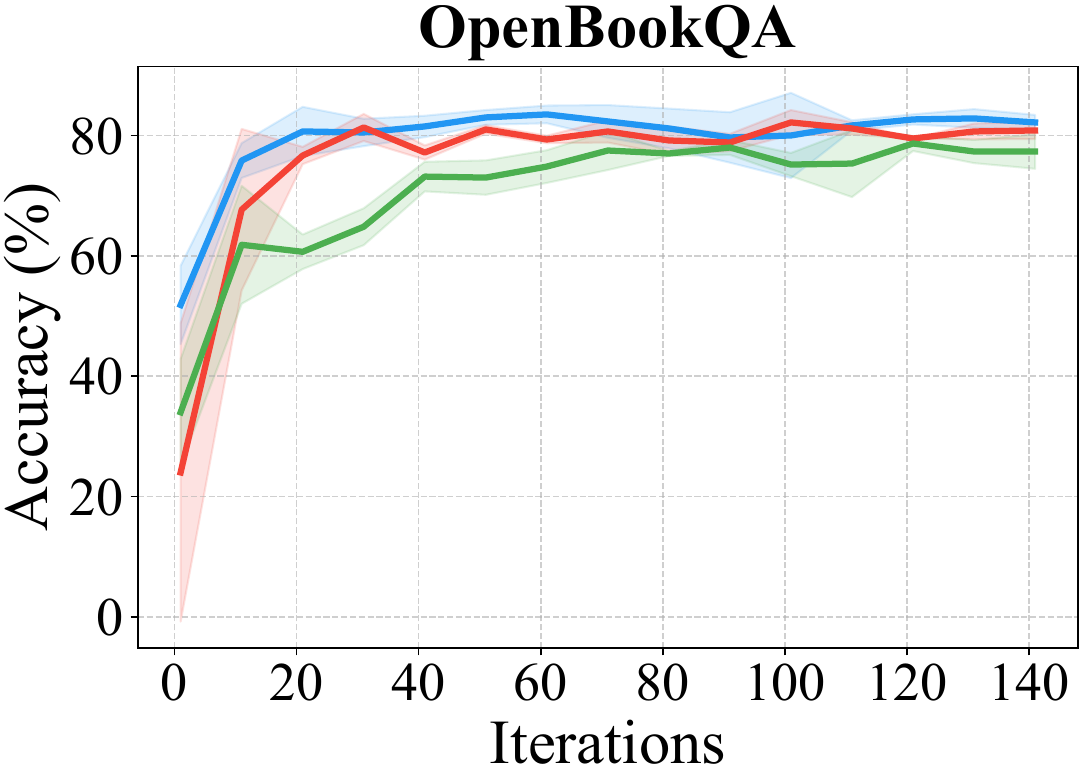}
    \end{subfigure}
    \hfill
    \begin{subfigure}[t]{0.24\linewidth}
        \centering
        \includegraphics[width=\linewidth]{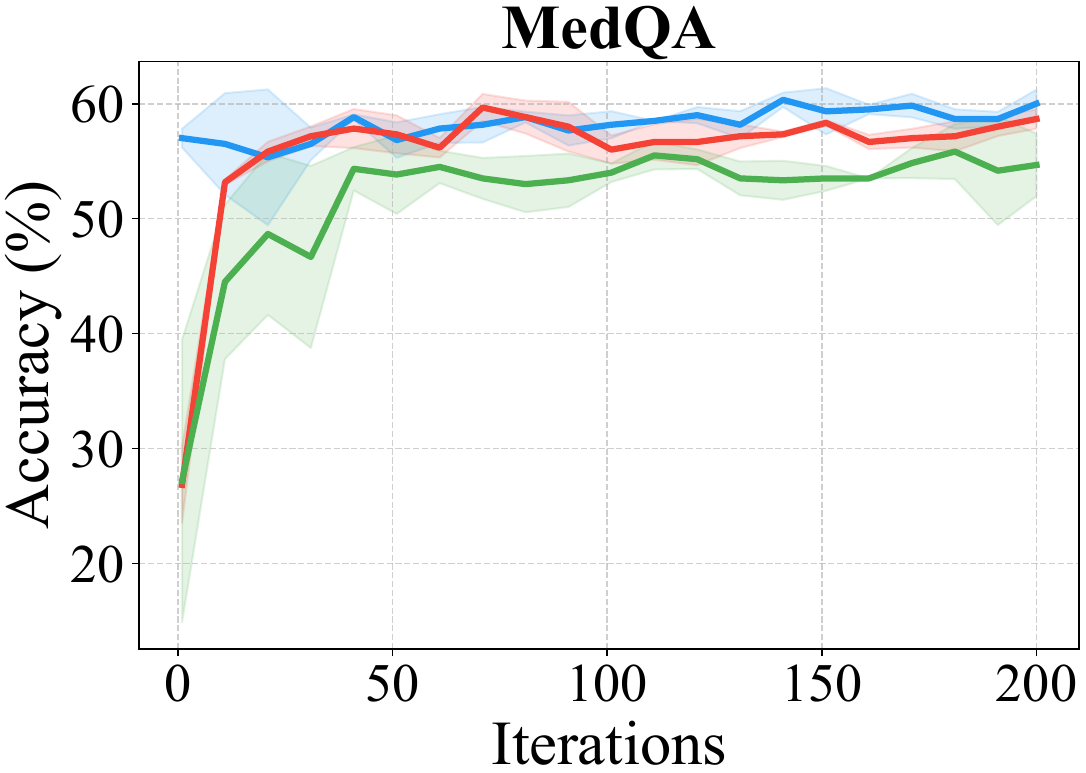}
    \end{subfigure}
    \caption{Training curves across four benchmarks: test  accuracy v.s. training iterations.}
    \label{fig:train_acc_adacc}
    \vspace{-1em}
\end{figure*}

\vspace{-1.2em}
\subsection{Main Result}
\vspace{-1em}

\textbf{Task accuracy} \quad
The results in Table~\ref{tab:main_result} show that TodyComm significantly outperforms the baseline methods in accuracy when the attack rate is 50\% or higher.
When the attack rate is below 50\%,
its performance as compared with the baselines is either significantly better (MMLU, ARC-Challenge, and OpenBookQA) or comparable  (GSM8K). 
Moreover, TodyComm exhibits low variance across three random seeds in all settings, achieving zero variance on MMLU under $>50\%$ attack and a maximum variance of only 1.52\%, indicating stable and reliable learning.

In contrast, Random and Complete Graphs rely on static topologies, 
while G-Designer constructs spatial graphs from predefined role information.
AgentPrune updates structures only at fixed training intervals. 
Consequently, they cannot adapt to agent's varying behavior that alternates between adversarial and normal across inference rounds and they connect all agents to the decision node. 

Figure~\ref{fig:train_acc_adacc} plots test accuracy against training iterations, showing that TodyComm consistently converges across benchmarks. 
Figure~\ref{fig:overfitting} demonstrates training vs. test accuracy on MMLU, indicating no overfitting or data leakage.

\textbf{Token Usage in Inference} \quad
In terms of token usage, TodyComm achieves comparable efficiency to AgentPrune, 
both of which are more token-efficient than the other baselines. AgentPrune starts from a complete graph and prunes edges at fixed training intervals. 
To optimize its performance, we set the pruning interval to one during training. 
TodyComm, in contrast, tends to avoid connectivity with adversarial agents, 
leading to similar token usage. 
Averaging accuracy and token consumption across all the five benchmarks,
Table~\ref{tab:average_results_main_result} further confirms that TodyComm achieves the best overall performance in both task accuracy and token efficiency.

\textbf{Training Efficiency and Token usage} \quad
Table~\ref{tab:training_efficiency} records accuracy and token usage at both the performance parity point and the peak point, where the performance parity point is defined as the point at which TodyComm rapidly converges to near-peak accuracy, after which further improvement becomes slow and marginal. 
Across all three attack rates and four benchmarks, 
the performance of TodyComm at its parity point already surpasses the peak accuracy of both baselines while consuming fewer training tokens. 
Figure~\ref{fig:training_efficiency} further confirms this advantage, showing that TodyComm achieves higher accuracy gain per million training tokens compared to GDesigner and AgentPrune.

\vspace{-0.8em}
\subsection{Node-wise In-degree and Out-degree Budgets}
\begin{table*}[t!]
\centering
\caption{Performance comparison under node-wise in-degree and out-degree budget constraints across four benchmarks.
\(\boldsymbol{\downarrow}\) next to a dataset name indicates performance degradation under a single-agent attack (see also Table~\ref{tab:attack_single}).
$\textcolor{darkgreen}{\boldsymbol{\downarrow}}$ indicates an increase in task accuracy or a reduction in token usage.
$\textcolor{darkred}{\boldsymbol{\uparrow}}$ indicates a decrease in task accuracy or an increase in token usage. 
$\textcolor{black}{\boldsymbol{\rightarrow}}$ denotes no change. 
All comparisons in $\textcolor{darkgreen}{\boldsymbol{\downarrow}}$, $\textcolor{darkred}{\boldsymbol{\uparrow}}$, and $\textcolor{black}{\boldsymbol{\rightarrow}}$ are made to the no-budget setting. 
All results are reported in percentages.
}
\vspace{-6pt}
\label{tab:budget_result}
\resizebox{\textwidth}{!}{%
\begin{tblr}{
  colspec = {l | l *{12}{c}}, 
  cell{1}{1,2} = {r=2}{c},        
  cell{1}{3,6,9,12} = {c=3}{c},   
  cell{3,5}{1} = {r=2}{l},
  row{3,5} = {abovesep=3pt, belowsep=1pt}, 
  row{4,6} = {abovesep=1pt, belowsep=3pt}, 
  hline{1,7} = {0.08em},            
  hline{3} = {0.05em},              
  hline{2} = {3-14}{0.03em},        
  hline{5} = {-}{0.02em},           
  vline{3,6,9,12} = {1-7}{0.05em},  
  row{1-2} = {font=\bfseries},
  colsep = 4pt,
}
Methods & Metrics & MMLU ($\downarrow$ 47.2\%) & & & ARC-C ($\downarrow$ 35.30\%)  & & & OpenBookQA ($\downarrow$ 61.11\%)& & & MedQA ($\downarrow$ 22.00\%) & & \\
   & & $<50\%$ & $=50\%$ & $>50\%$ & $<50\%$ & $=50\%$ & $>50\%$ & $<50\%$ & $=50\%$ & $>50\%$ & $<50\%$ & $=50\%$ & $>50\%$ \\

  \SetCell[r=2]{l} budget = 1
  & Acc         & {71.90 \\ \hfill \scriptsize \textcolor{darkred}{$\boldsymbol{\downarrow}$0.21}} & {69.28 \\ \hfill \scriptsize \textcolor{darkgreen}{$\boldsymbol{\uparrow}$0.43}} & {70.59 \\ \hfill \scriptsize \textcolor{darkgreen}{$\boldsymbol{\uparrow}$5.88}} & {88.96 \\ \hfill \scriptsize \textcolor{darkgreen}{$\boldsymbol{\uparrow}$0.44}} & {86.62 \\ \hfill \scriptsize \textcolor{darkgreen}{$\boldsymbol{\uparrow}$0.67}} & {83.28 \\ \hfill \scriptsize \textcolor{darkgreen}{$\boldsymbol{\uparrow}$2.23}} & {85.50 \\ \hfill \scriptsize \textcolor{darkgreen}{$\boldsymbol{\uparrow}$0.67}} & {83.50 \\ \hfill \scriptsize \textcolor{darkred}{$\boldsymbol{\downarrow}$0.67}} & {83.50 \\ \hfill \scriptsize \textcolor{darkgreen}{$\boldsymbol{\uparrow}$3.00}} & {61.00 \\ \hfill \scriptsize \textcolor{black}{$\boldsymbol{\rightarrow}$0.00}} & {60.50 \\ \hfill \scriptsize \textcolor{darkgreen}{$\boldsymbol{\uparrow}$0.33} }  & {56.50 \\ \hfill \scriptsize \textcolor{darkred}{$\boldsymbol{\downarrow}$1.00}} \\
  & Token       & {378 \\ \hfill \scriptsize \textcolor{darkgreen}{$\boldsymbol{\downarrow}$4.79} }   & {390 \\ \hfill \scriptsize \textcolor{darkgreen}{$\boldsymbol{\downarrow}$3.70}}   & {397 \\ \hfill \scriptsize \textcolor{darkgreen}{$\boldsymbol{\downarrow}$4.11}}   & {338 \\ \hfill \scriptsize \textcolor{darkgreen}{$\boldsymbol{\downarrow}$3.43}}   & {346 \\ \hfill \scriptsize \textcolor{darkgreen}{$\boldsymbol{\downarrow}$4.16}}   & {358 \\ \hfill \scriptsize \textcolor{darkgreen}{$\boldsymbol{\downarrow}$2.19}}   & {316 \\ \hfill \scriptsize \textcolor{darkgreen}{$\boldsymbol{\downarrow}$26.17}}   & {326 \\ \hfill \scriptsize \textcolor{darkgreen}{$\boldsymbol{\downarrow}$27.23}}   & {337 \\ \hfill \scriptsize \textcolor{darkgreen}{$\boldsymbol{\downarrow}$26.58}}   & {539 \\ \hfill \scriptsize \textcolor{darkgreen}{$\boldsymbol{\downarrow}$3.75}}   & {546 \\ \hfill \scriptsize \textcolor{darkgreen}{$\boldsymbol{\downarrow}$3.36}}   & {551 \\ \hfill \scriptsize \textcolor{darkgreen}{$\boldsymbol{\downarrow}$2.65}}   \\

  \hline

  \SetCell[r=2]{l} budget = 2
  & Acc         & {69.93 \\ \hfill \scriptsize \textcolor{darkred}{$\boldsymbol{\downarrow}$2.18} } & {67.32 \\ \hfill \scriptsize \textcolor{darkred}{$\boldsymbol{\downarrow}$1.53}} & {65.36 \\ \hfill \scriptsize \textcolor{darkgreen}{$\boldsymbol{\uparrow}$0.65}} & {87.96 \\ \hfill \scriptsize \textcolor{darkred}{$\boldsymbol{\downarrow}$0.56}} & {87.63 \\ \hfill \scriptsize \textcolor{darkgreen}{$\boldsymbol{\uparrow}$1.68}} & {84.95 \\ \hfill \scriptsize \textcolor{darkgreen}{$\boldsymbol{\uparrow}$3.90}} & {85.50 \\ \hfill \scriptsize \textcolor{darkgreen}{$\boldsymbol{\uparrow}$0.67}} & {83.50 \\ \hfill \scriptsize \textcolor{darkred}{$\boldsymbol{\downarrow}$0.67} } & {81.00 \\ \hfill \scriptsize \textcolor{darkgreen}{$\boldsymbol{\uparrow}$0.50}} & {62.00 \\ \hfill \scriptsize \textcolor{darkgreen}{$\boldsymbol{\uparrow}$1.00}} & {59.00 \\ \hfill \scriptsize \textcolor{darkred}{$\boldsymbol{\downarrow}$1.17} } & {58.00 \\ \hfill \scriptsize \textcolor{darkgreen}{$\boldsymbol{\uparrow}$0.50}} \\
  & Token       & {393 \\ \hfill \scriptsize \textcolor{darkgreen}{$\boldsymbol{\downarrow}$1.01}}   & {402 \\ \hfill \scriptsize \textcolor{darkgreen}{$\boldsymbol{\downarrow}$0.74}}   & {408\\ \hfill \scriptsize \textcolor{darkgreen}{$\boldsymbol{\downarrow}$1.45}}   & {349 \\ \hfill \scriptsize \textcolor{darkgreen}{$\boldsymbol{\downarrow}$0.29}}   & {349 \\ \hfill \scriptsize \textcolor{darkgreen}{$\boldsymbol{\downarrow}$3.32}}   & {372 \\ \hfill \scriptsize \textcolor{darkred}{$\boldsymbol{\uparrow}$1.64}}   & {328 \\ \hfill \scriptsize \textcolor{darkgreen}{$\boldsymbol{\downarrow}$23.36}}   & {338 \\ \hfill \scriptsize \textcolor{darkgreen}{$\boldsymbol{\downarrow}$24.55}}   & {348\\ \hfill \scriptsize \textcolor{darkgreen}{$\boldsymbol{\downarrow}$24.18}}   & {550 \\ \hfill \scriptsize \textcolor{darkgreen}{$\boldsymbol{\downarrow}$1.79}}   & {554 \\ \hfill \scriptsize \textcolor{darkgreen}{$\boldsymbol{\downarrow}$1.95}}   & {564 \\ \hfill \scriptsize \textcolor{darkgreen}{$\boldsymbol{\downarrow}$0.35}} \\
\end{tblr}
}
\end{table*}
\vspace{-0.8em}

Communication cost often necessitates degree budgets for each agent.
Previously, 
when $k \in \{2,3,4\}$ out of 6 agents become adversarial, the maximum number of reliable outgoing edges per agent is $5-k$.
We therefore evaluate budgets of 1 and 2 for each agent to receive or send messages.
As shown in Table~\ref{tab:budget_result}, moderate budgets preserve or even improve task accuracy while consistently reducing token usage.
In particular, when four agents are adversarial, 
a budget of one substantially improves accuracy upon the no-budget setting, indicating that restricting communication helps suppress unreliable messages, while a budget of 2 provides smaller but consistent gains.
The training curves for budget control are presented in Figure~\ref{fig:budget_training_curve}.

Notably, token usage is significantly reduced under budget constraints, with especially large savings on OpenBookQA, whose higher sensitivity to adversarial messages (a 61.11\% accuracy drop under a single-agent attack) provides a stronger reinforcement signal and enables earlier suppression of unreliable communication. 

\vspace{-0.9em}
\subsection{Scalability in conjunction with degree budget}
\vspace{-0.8em}
\begin{table*}[t!]
\centering
\caption{Scalability performance on MMLU measured by accuracy. \#Agents means the number of agents. \#RAgents means the number of reliable agents.}
\vspace{-6pt}
\label{tab:scalibility_result}
\footnotesize
\setlength{\tabcolsep}{2.5pt} 
\renewcommand{\arraystretch}{1}  
\begin{tabular}{l|c|c|c|c|c|c|c|c|c}
\toprule
\#Agents  & 6 agents & 6 agents & 10 agents & 10 agents & 15 agents & 15 agents & 15 agents & 20 agents & 20 agents \\
\midrule
Attack rate & 0.6      & 0.6      & 0.6       & 0.6       & 0.7       & 0.7       & 0.6       & 0.6       & 0.8       \\
\#RAgents      & 2        & 2        & 4         & 4         & 4        & 4        & 6         & 8        & 4        \\
Budget      & 2        & 1        & 4         & 3         & 4         & 3         & 5         & 5         & 3         \\
Acc         & 65.36    & 70.59    & 62.75     & 67.32     & 62.75     & 61.44     & 67.97     & 68.63     & 54.90     \\
Token       & 408      & 397      & 438       & 429       & 452       & 440       & 460       & 469       & 443       \\
\bottomrule
\end{tabular}
\end{table*}
We tested the flexibility of our method with $N \in \{ 6, 10, 15, 20\}$ agents,
using varied values of node degree budget ($B_{\text{out}}=B_{\text{in}}$).
As shown in Table~\ref{tab:scalibility_result}, 
at an attack rate of 0.6 and under a reasonable budget, 
task accuracy remains relatively stable as the number of agents increases. 
Inference token usage grows only slightly from 6 to 20 agents due to explicit budget control.
The training curves for scalability are presented in Figure~\ref{fig:scalability_training_curve}.

When the number of reliable agents is fixed and more adversarial agents are added, performance degrades but remains at a reasonable level even under highly imbalanced settings (e.g., \#reliable : \#adversarial $= 1:4$). 

We also tested the impact of budget,
ranging over \#RAgents (\# reliable agent), \#RAgents$-1$, and values smaller than \#RAgents.
While the performance exhibits minor fluctuations across these choices, smaller budgets generally offer a favorable trade-off between task performance and token efficiency. 
Thus, we recommend using a modest budget to balance scalability and performance.

\vspace{-0.8em}
\subsection{Generalization across attack mechanism, attack outbreak mode, task domain, \& \# Agent}
\vspace{-0.8em}

\begin{table*}[t]
\centering
\caption{Attack Mechanism Generalization with Qwen3-8B.}
\vspace{-6pt}
\label{tab:attack_mechanism_generalization}
\resizebox{\textwidth}{!}{
\renewcommand{\arraystretch}{1.2}
\begin{tabular}{c | c c c | c c c | c c c | c c c}
\toprule
\multirow{2}{*}{\textbf{Methods}}
& \multicolumn{3}{c|}{\textbf{MMLU}}
& \multicolumn{3}{c|}{\textbf{ARC-C}}
& \multicolumn{3}{c|}{\textbf{OBQA}}
& \multicolumn{3}{c}{\textbf{MedQA }} \\
\cmidrule(lr){2-4} \cmidrule(lr){5-7} \cmidrule(lr){8-10} \cmidrule(lr){11-13}
& \multicolumn{1}{c}{$<50\%$} & \multicolumn{1}{c}{$=50\%$} & \multicolumn{1}{c|}{$>50\%$}
& \multicolumn{1}{c}{$<50\%$} & \multicolumn{1}{c}{$=50\%$} & \multicolumn{1}{c|}{$>50\%$}
& \multicolumn{1}{c}{$<50\%$} & \multicolumn{1}{c}{$=50\%$} & \multicolumn{1}{c|}{$>50\%$}
& \multicolumn{1}{c}{$<50\%$} & \multicolumn{1}{c}{$=50\%$} & \multicolumn{1}{c}{$>50\%$} \\
\hline
\rowcolor{gray!10}
\multicolumn{13}{c}{\rule{0pt}{0.9em}\textbf{\textit{Targeted Attack}}\rule[-0.2em]{0pt}{0pt}} \\
\hline
G-Designer
& 51.63 & 37.26 & 22.53
& 56.86 & 46.49 & 39.47
& 52.50 & 26.00 & 21.00
& 43.00 & 25.00 & 17.50 \\
AgentPrune
& 67.97 & 39.87 & 24.18
& 87.63 & 49.50 & 38.13
& 71.50 & 34.50 & 19.00
& 63.50 & 35.00 & 20.50 \\
TodyComm (Ours)
& \textbf{69.28} & \textbf{67.32} & \textbf{60.13}
& \textbf{88.29} & \textbf{81.94} & \textbf{79.56}
& \textbf{84.00} & \textbf{80.50} & \textbf{70.50}
& \textbf{64.50} & \textbf{60.00} & \textbf{53.50} \\
\hline
\rowcolor{gray!10}
\multicolumn{13}{c}{\rule{0pt}{0.9em}\textbf{\textit{Untargeted Attack}}\rule[-0.2em]{0pt}{0pt}} \\
\hline
G-Designer
& 54.92 & 41.18 & 33.99
& 65.55 & 56.19 & 46.82
& 64.50 & 52.00 & 36.00
& 43.00 & 23.50 & 17.50 \\
AgentPrune
& 67.32 & 47.06 & 33.33
& 87.96 & 66.22 & 46.49
& 72.00 & 56.00 & 42.00
& 61.50 & 43.50 & 24.00 \\
TodyComm (Ours)
& \textbf{67.97} & \textbf{67.97} & \textbf{59.48}
& \textbf{88.29} & \textbf{83.61} & \textbf{74.58}
& \textbf{83.00} & \textbf{81.50} & \textbf{73.50}
& \textbf{63.50} & \textbf{61.50} & \textbf{55.00} \\
\bottomrule
\end{tabular}
}
\vspace{-0.7em}
\end{table*}

We introduced a new attack variant called the \textit{untargeted} attack, 
in which adversarial agents prompt their LLMs with option sets that exclude the correct answer. 
This contrasts with the \textit{targeted} attack used in the main results. 
In this experiment, we used Qwen3-8B and measured the performance drop of vanilla single agent upon receiving harmful responses from one adversarial agent. 
The drops are substantial across all benchmarks, where ``T'' and ``U'' denote targeted and untargeted attacks respectively: MMLU (T$\downarrow$ 80.37\%, U$\downarrow$ 90.65\%), ARC-C (T$\downarrow$ 87.40\%, U$\downarrow$ 83.97\%), OBQA (T$\downarrow$ 93.37\%, U$\downarrow$ 94.59\%), and MedQA (T$\downarrow$ 93.85\%, U$\downarrow$ 93.08\%), highlighting the vulnerability of small models and the importance of our dynamic design. 

Table~\ref{tab:attack_mechanism_generalization} shows that TodyComm maintains strong performance under both attack types, while G-Designer and AgentPrune degrade significantly at higher attack rates, demonstrating TodyComm's generalization across attack mechanisms.

Beyond evaluating new attack mechanisms, we further assess generalization across multiple dimensions: 
(i) \textbf{number of agents}: training with fewer agents and evaluating with more (Table~\ref{tab:num_agent_generalization}); 
(ii) \textbf{task domain}: training on one dataset and evaluating on another (Table~\ref{tab:cross_dataset}); 
(iii) \textbf{adversary outbreak}: training under one outbreak mode (\textit{random} or \textit{fixed}) and evaluating under the other across different datasets (Table~\ref{tab:attack_inception}); and 
(iv) \textbf{attack mechanism}: training on one type of attack (\textit{targeted} or \textit{untargeted attack}) and evaluating on another (Table~\ref{tab:diff_attack_in_train_infer}). We also examine combinations of these factors (Tables~\ref{tab:3factors_generalization} and~\ref{tab:4factors_generalization}). The results consistently demonstrate the strong generalizability of TodyComm.

\vspace{-0.5em}
\subsection{Ablation Study}
\vspace{-0.8em}
\label{sec:ablation}

\textbf{Inter-agent Graph Construction} \quad
Table~\ref{tab:ablation_method_design} reports the ablation results in constructing the communication graph.
The variant of ``w/o learning'' has access to oracle reliability labels and sets node potential $c_i^t$ to 1 for reliable agents.
It yields only 49.70\% accuracy when the attack rate exceeds 50\%.
A vanilla agent achieves 70.6\% accuracy on MMLU, and the probability that two agents produce the same correct answer is around 49\%. When the two agents disagree, the final answer is selected uniformly at random, causing performance to degrade toward chance level (\(49.70\%\)). 
The variant of ``w random ordering'' constructs edges in a random order rather than following the learned sequence, degrading performance in all attack rates. 
These results highlight the importance of both learning edge probabilities and constructing edges in the proper sequence for effective communication and decision making.
Details are provided in Appendix~\ref{app:graph_construction_ablation}.

\textbf{Node Features} \quad 
We conducted ablation on node features,
demonstrating that different components contribute to the effectiveness of TodyComm.
Details are provided in Appendix~\ref{app:node_features_ablation}.

\textbf{Choice of LLM}   \quad
We evaluated TodyComm on three additional LLMs of varying sizes and model families to further validate its robustness. 
Combined with the results on GPT-4.1-nano and Qwen3-8B, TodyComm demonstrates strong and consistent robustness across five diverse LLMs. Details are provided in Appendix \ref{app:choice_llms}.

\textbf{Choice of Embedding Model} \quad
We evaluated our method with multiple embedding models to demonstrate robustness to the choice of embedding. Details are provided in Appendix~\ref{app:choice_embedding}.

\textbf{Case Study} \quad
A case study in Appendix~\ref{sec:case_study} visualizes the dynamics of node potentials and graph evolution for six agents interacting over four rounds under a budget of 2.

\vspace{-1em}
\section{Conclusion}
\vspace{-1em}
\label{sec:conclusion}
In this paper, we present a comprehensive study of multi-round LLM-based multi-agent systems and propose TodyComm, a framework that learns round-adaptive communication topologies from agent behaviors.
Extensive experiments on diverse benchmarks and attack rates/mechanisms, together with the node-wise budget control, scalability and generalization, demonstrate the effectiveness of TodyComm.
We also theoretically bounded the regret in the ultimate utility via the screening regret.
The current work focuses on the adversarial setting,
while dynamic graph can also be useful in collaborative settings.
We will extend TodyComm to such scenarios in the future work.

\clearpage
\newpage

\bibliographystyle{xinhua_num}
\bibliography{reference}

@inproceedings{zhangcut,
  title={Cut the Crap: An Economical Communication Pipeline for LLM-based Multi-Agent Systems},
  author={Zhang, Guibin and Yue, Yanwei and Li, Zhixun and Yun, Sukwon and Wan, Guancheng and Wang, Kun and Cheng, Dawei and Yu, Jeffrey Xu and Chen, Tianlong},
  booktitle={The Thirteenth International Conference on Learning Representations},
  year = {2025}
}

@inproceedings{zhangg,
  title={G-Designer: Architecting Multi-agent Communication Topologies via Graph Neural Networks},
  author={Zhang, Guibin and Yue, Yanwei and Sun, Xiangguo and Wan, Guancheng and Yu, Miao and Fang, Junfeng and Wang, Kun and Chen, Tianlong and Cheng, Dawei},
  booktitle={Forty-second International Conference on Machine Learning},
  year = {2025}
}

@article{chung2014empirical,
  title={Empirical evaluation of gated recurrent neural networks on sequence modeling},
  author={Chung, Junyoung and Gulcehre, Caglar and Cho, KyungHyun and Bengio, Yoshua},
  journal={arXiv preprint arXiv:1412.3555},
  year={2014}
}

@inproceedings{du2023improving,
  title={Improving factuality and reasoning in language models through multiagent debate},
  author={Du, Yilun and Li, Shuang and Torralba, Antonio and Tenenbaum, Joshua B and Mordatch, Igor},
  booktitle={Forty-first International Conference on Machine Learning},
  year={2023}
}

@inproceedings{liu2024dynamic,
  title={A dynamic llm-powered agent network for task-oriented agent collaboration},
  author={Liu, Zijun and Zhang, Yanzhe and Li, Peng and Liu, Yang and Yang, Diyi},
  booktitle={First Conference on Language Modeling},
  year={2024}
}

@inproceedings{zhuge2024gptswarm,
  title={{GPTSwarm}: Language agents as optimizable graphs},
  author={Zhuge, Mingchen and Wang, Wenyi and Kirsch, Louis and Faccio, Francesco and Khizbullin, Dmitrii and Schmidhuber, J{\"u}rgen},
  booktitle={Forty-first International Conference on Machine Learning},
  year={2024}
}

@inproceedings{hong2023metagpt,
  title={{MetaGPT}: Meta programming for a multi-agent collaborative framework},
  author={Hong, Sirui and Zhuge, Mingchen and Chen, Jonathan and Zheng, Xiawu and Cheng, Yuheng and Wang, Jinlin and Zhang, Ceyao and Wang, Zili and Yau, Steven Ka Shing and Lin, Zijuan and others},
  booktitle={The Twelfth International Conference on Learning Representations},
  year={2023}
}

@inproceedings{liang2024encouraging,
  title={Encouraging divergent thinking in large language models through multi-agent debate},
  author={Liang, Tian and He, Zhiwei and Jiao, Wenxiang and Wang, Xing and Wang, Yan and Wang, Rui and Yang, Yujiu and Shi, Shuming and Tu, Zhaopeng},
  booktitle={Proceedings of the 2024 conference on empirical methods in natural language processing},
  pages={17889--17904},
  year={2024}
}

@inproceedings{parmar-etal-2025-plangen,
    title = "{P}lan{GEN}: A Multi-Agent Framework for Generating Planning and Reasoning Trajectories for Complex Problem Solving",
    author = "Parmar, Mihir  and
      Liu, Xin  and
      Goyal, Palash  and
      Chen, Yanfei  and
      Le, Long  and
      Mishra, Swaroop  and
      Mobahi, Hossein  and
      Gu, Jindong  and
      Wang, Zifeng  and
      Nakhost, Hootan  and
      Baral, Chitta  and
      Lee, Chen-Yu  and
      Pfister, Tomas  and
      Palangi, Hamid",
    editor = "Christodoulopoulos, Christos  and
      Chakraborty, Tanmoy  and
      Rose, Carolyn  and
      Peng, Violet",
    booktitle = "Proceedings of the 2025 Conference on Empirical Methods in Natural Language Processing",
    month = nov,
    year = "2025",
    address = "Suzhou, China",
    publisher = "Association for Computational Linguistics",
    url = "https://aclanthology.org/2025.emnlp-main.1042/",
    doi = "10.18653/v1/2025.emnlp-main.1042",
    pages = "20640--20666",
    ISBN = "979-8-89176-332-6"
}

@inproceedings{fan2025evomem,
  title={EvoMem: Improving Multi-Agent Planning with Dual-Evolving Memory},
  author={Fan, Wenzhe and Yan, Ning and Mortazavi, Masood S},
  booktitle={NeurIPS 2025 Workshop on Bridging Language, Agent, and World Models for Reasoning and Planning},
  year = {2025}
}

@article{zou2025latent,
  title={Latent collaboration in multi-agent systems},
  author={Zou, Jiaru and Yang, Xiyuan and Qiu, Ruizhong and Li, Gaotang and Tieu, Katherine and Lu, Pan and Shen, Ke and Tong, Hanghang and Choi, Yejin and He, Jingrui and others},
  journal={arXiv preprint arXiv:2511.20639},
  year={2025}
}

@inproceedings{zhaosirius,
  title={SiriuS: Self-improving Multi-agent Systems via Bootstrapped Reasoning},
  author={Zhao, Wanjia and Yuksekgonul, Mert and Wu, Shirley and Zou, James},
  booktitle={Advances in Neural Information Processing Systems},
   year = {2025}
}

@article{li2023camel,
  title={{CAMEL}: Communicative agents for" mind" exploration of large language model society},
  author={Li, Guohao and Hammoud, Hasan and Itani, Hani and Khizbullin, Dmitrii and Ghanem, Bernard},
  journal={Advances in Neural Information Processing Systems},
  volume={36},
  pages={51991--52008},
  year={2023}
}

@article{zhang2025safesieve,
  title={{SafeSieve}: From Heuristics to Experience in Progressive Pruning for LLM-based Multi-Agent Communication},
  author={Zhang, Ruijia and Zhao, Xinyan and Wang, Ruixiang and Chen, Sigen and Zhang, Guibin and Zhang, An and Wang, Kun and Wen, Qingsong},
  journal={arXiv preprint arXiv:2508.11733},
  year={2025}
}

@inproceedings{wang-etal-2025-agentdropout,
    title = "{A}gent{D}ropout: Dynamic Agent Elimination for Token-Efficient and High-Performance {LLM}-Based Multi-Agent Collaboration",
    author = "Wang, Zhexuan  and
      Wang, Yutong  and
      Liu, Xuebo  and
      Ding, Liang  and
      Zhang, Miao  and
      Liu, Jie  and
      Zhang, Min",
    editor = "Che, Wanxiang  and
      Nabende, Joyce  and
      Shutova, Ekaterina  and
      Pilehvar, Mohammad Taher",
    booktitle = "Proceedings of the 63rd Annual Meeting of the Association for Computational Linguistics (Volume 1: Long Papers)",
    month = jul,
    year = "2025",
    address = "Vienna, Austria",
    publisher = "Association for Computational Linguistics",
    url = "https://aclanthology.org/2025.acl-long.1170/",
    doi = "10.18653/v1/2025.acl-long.1170",
    pages = "24013--24035",
    ISBN = "979-8-89176-251-0"
}

@inproceedings{wang-etal-2025-g,
    title = "{G}-Safeguard: A Topology-Guided Security Lens and Treatment on {LLM}-based Multi-agent Systems",
    author = "Wang, Shilong  and
      Zhang, Guibin  and
      Yu, Miao  and
      Wan, Guancheng  and
      Meng, Fanci  and
      Guo, Chongye  and
      Wang, Kun  and
      Wang, Yang",
    editor = "Che, Wanxiang  and
      Nabende, Joyce  and
      Shutova, Ekaterina  and
      Pilehvar, Mohammad Taher",
    booktitle = "Proceedings of the 63rd Annual Meeting of the Association for Computational Linguistics (Volume 1: Long Papers)",
    month = jul,
    year = "2025",
    address = "Vienna, Austria",
    publisher = "Association for Computational Linguistics",
    url = "https://aclanthology.org/2025.acl-long.359/",
    doi = "10.18653/v1/2025.acl-long.359",
    pages = "7261--7276",
    ISBN = "979-8-89176-251-0"
}

@article{miao2025blindguard,
  title={{BlindGuard}: Safeguarding llm-based multi-agent systems under unknown attacks},
  author={Miao, Rui and Liu, Yixin and Wang, Yili and Shen, Xu and Tan, Yue and Dai, Yiwei and Pan, Shirui and Wang, Xin},
  journal={arXiv preprint arXiv:2508.08127},
  year={2025}
}

@article{clark2018think,
  title={Think you have solved question answering? try arc, the ai2 reasoning challenge},
  author={Clark, Peter and Cowhey, Isaac and Etzioni, Oren and Khot, Tushar and Sabharwal, Ashish and Schoenick, Carissa and Tafjord, Oyvind},
  journal={arXiv preprint arXiv:1803.05457},
  year={2018}
}

@inproceedings{hendrycksmeasuring,
  title={Measuring Massive Multitask Language Understanding},
  author={Hendrycks, Dan and Burns, Collin and Basart, Steven and Zou, Andy and Mazeika, Mantas and Song, Dawn and Steinhardt, Jacob},
  booktitle={International Conference on Learning Representations},
  year = {2021}
}

@article{cobbe2021training,
  title={Training verifiers to solve math word problems},
  author={Cobbe, Karl and Kosaraju, Vineet and Bavarian, Mohammad and Chen, Mark and Jun, Heewoo and Kaiser, Lukasz and Plappert, Matthias and Tworek, Jerry and Hilton, Jacob and Nakano, Reiichiro and others},
  journal={arXiv preprint arXiv:2110.14168},
  year={2021}
}

@inproceedings{yang2025llm,
  title={{LLM-MedQA}: Enhancing medical question answering through case studies in large language models},
  author={Yang, Hang and Chen, Hao and Guo, Hui and Chen, Yineng and Lin, Ching-Sheng and Hu, Shu and Hu, Jinrong and Wu, Xi and Wang, Xin},
  booktitle={2025 International Joint Conference on Neural Networks (IJCNN)},
  pages={1--8},
  year={2025},
  organization={IEEE}
}

@inproceedings{mihaylov2018can,
  title={Can a Suit of Armor Conduct Electricity? A New Dataset for Open Book Question Answering},
  author={Mihaylov, Todor and Clark, Peter and Khot, Tushar and Sabharwal, Ashish},
  booktitle={Proceedings of the 2018 Conference on Empirical Methods in Natural Language Processing},
  pages={2381--2391},
  year={2018}
}

@inproceedings{jiang2023llm,
  title={{LLM-Blender}: Ensembling Large Language Models with Pairwise Ranking and Generative Fusion},
  author={Jiang, Dongfu and Ren, Xiang and Lin, Bill Yuchen},
  booktitle={Proceedings of the 61st Annual Meeting of the Association for Computational Linguistics (Volume 1: Long Papers)},
  pages={14165--14178},
  year={2023}
}

@inproceedings{qianscaling,
  title={Scaling Large Language Model-based Multi-Agent Collaboration},
  author={Qian, Chen and Xie, Zihao and Wang, YiFei and Liu, Wei and Zhu, Kunlun and Xia, Hanchen and Dang, Yufan and Du, Zhuoyun and Chen, Weize and Yang, Cheng and others},
  booktitle={The Thirteenth International Conference on Learning Representations},
  year = {2025}
}

@inproceedings{chanchateval,
  title={{ChatEval}: Towards Better LLM-based Evaluators through Multi-Agent Debate},
  author={Chan, Chi-Min and Chen, Weize and Su, Yusheng and Yu, Jianxuan and Xue, Wei and Zhang, Shanghang and Fu, Jie and Liu, Zhiyuan},
  booktitle={The Twelfth International Conference on Learning Representations},
  year = {2024}
}

@inproceedings{wu2024autogen,
  title={{AutoGen}: Enabling next-gen LLM applications via multi-agent conversations},
  author={Wu, Qingyun and Bansal, Gagan and Zhang, Jieyu and Wu, Yiran and Li, Beibin and Zhu, Erkang and Jiang, Li and Zhang, Xiaoyun and Zhang, Shaokun and Liu, Jiale and others},
  booktitle={First Conference on Language Modeling},
  year={2024}
}

@article{zhou2026resmas,
  title={{ResMAS}: Resilience Optimization in LLM-based Multi-agent Systems},
  author={Zhou, Zhilun and Liu, Zihan and Liu, Jiahe and Shao, Qingyu and Wang, Yihan and Shao, Kun and Jin, Depeng and Xu, Fengli},
  journal={arXiv preprint arXiv:2601.04694},
  year={2026}
}

@inproceedings{li2025adaptive,
  title={Adaptive Graph Pruning for Multi-Agent Communication},
  author={Li, Boyi and Zhao, Zhonghan and Lee, Der-Horng and Wang, Gaoang},
  booktitle = {Proceedings of the 27th European Conference on Artificial Intelligence (ECAI)},
  year={2025}
}

@article{puterman1990markov,
  title={Markov decision processes},
  author={Puterman, Martin L},
  journal={Handbooks in operations research and management science},
  volume={2},
  pages={331--434},
  year={1990},
  publisher={Elsevier}
}

@article{williams1992simple,
  title={Simple statistical gradient-following algorithms for connectionist reinforcement learning},
  author={Williams, Ronald J},
  journal={Machine learning},
  volume={8},
  number={3},
  pages={229--256},
  year={1992},
  publisher={Springer}
}

@article{kipf2016variational,
  title={Variational graph auto-encoders},
  author={Kipf, Thomas N and Welling, Max},
  journal={arXiv preprint arXiv:1611.07308},
  year={2016}
}

@article{hoff2002latent,
  title={Latent space approaches to social network analysis},
  author={Hoff, Peter D and Raftery, Adrian E and Handcock, Mark S},
  journal={Journal of the american Statistical association},
  volume={97},
  number={460},
  pages={1090--1098},
  year={2002},
  publisher={Taylor \& Francis}
}

@article{koller2003multi,
title = {Multi-agent influence diagrams for representing and solving games},
journal = {Games and Economic Behavior},
volume = {45},
number = {1},
pages = {181-221},
year = {2003},
author = {Daphne Koller and Brian Milch},
}

@inproceedings{das2019tarmac,
  title={Tarmac: Targeted multi-agent communication},
  author={Das, Abhishek and Gervet, Th{\'e}ophile and Romoff, Joshua and Batra, Dhruv and Parikh, Devi and Rabbat, Mike and Pineau, Joelle},
  booktitle={International Conference on machine learning},
  pages={1538--1546},
  year={2019},
  organization={PMLR}
}

@inproceedings{wangtom2c,
  title={ToM2C: Target-oriented Multi-agent Communication and Cooperation with Theory of Mind},
  author={Wang, Yuanfei and Xu, Jing and Wang, Yizhou and others},
  booktitle={International Conference on Learning Representations},
  year={2022},
}

@inproceedings{singhlearning,
  title={Learning when to Communicate at Scale in Multiagent Cooperative and Competitive Tasks},
  author={Singh, Amanpreet and Jain, Tushar and Sukhbaatar, Sainbayar},
  booktitle={International Conference on Learning Representations},
  year={2019}
}

@inproceedings{kimlearning,
  title={Learning to Schedule Communication in Multi-agent Reinforcement Learning},
  author={Kim, Daewoo and Moon, Sangwoo and Hostallero, David and Kang, Wan Ju and Lee, Taeyoung and Son, Kyunghwan and Yi, Yung},
  booktitle={International Conference on Learning Representations},
  year={2019}
}

@inproceedings{niu2021multi,
  title={Multi-Agent Graph-Attention Communication and Teaming.},
  author={Niu, Yaru and Paleja, Rohan R and Gombolay, Matthew C},
  booktitle={Aamas},
  volume={21},
  pages={20th},
  year={2021}
}

@article{lu2026dytopo,
  title={DyTopo: Dynamic Topology Routing for Multi-Agent Reasoning via Semantic Matching},
  author={Lu, Yuxing and Hu, Yucheng and Zhao, Xukai and Cao, Jiuxin},
  journal={arXiv preprint arXiv:2602.06039},
  year={2026}
}

@article{cang2026graph,
  title={Graph-GRPO: Stabilizing Multi-Agent Topology Learning via Group Relative Policy Optimization},
  author={Cang, Yueyang and Zhang, Xiaoteng and Zhao, Erlu and Ji, Zehua and Liu, Yuhang and He, Yuchen and Ning, Zhiyuan and Yijun, Chen and Que, Wenge and Shi, Li},
  journal={arXiv preprint arXiv:2603.02701},
  year={2026}
}

@article{chung2002average,
  title={The average distances in random graphs with given expected degrees},
  author={Chung, Fan and Lu, Linyuan},
  journal={Proceedings of the National Academy of Sciences},
  volume={99},
  number={25},
  pages={15879--15882},
  year={2002},
  publisher={National Academy of Sciences}
}

@article{chatterjee2011random,
  title={Random graphs with a given degree sequence.},
  author={Chatterjee, Sourav and Diaconis, Persi and Sly, Allan},
  journal={Annals of applied probability: an official journal of the Institute of Mathematical Statistics},
  volume={21},
  number={4},
  pages={1400--1435},
  year={2011},
  publisher={Institute of Mathematical Statistics}
}

@incollection{lovasz2010regularity,
  title={Regularity partitions and the topology of graphons},
  author={Lov{\'a}sz, L{\'a}szl{\'o} and Szegedy, Bal{\'a}zs},
  booktitle={An Irregular Mind: Szemer{\'e}di is 70},
  pages={415--446},
  publisher={Springer},
  year={2010}
}

@inproceedings{sheng2023learning,
  title={Learning Structured Communication for Multi-Agent Reinforcement Learning},
  author={Sheng, Junjie and Wang, Xiangfeng and Jin, Bo and Li, Wenhao and Yan, Junchi and Wang, Jun and Zha, Hongyuan},
  booktitle={Proceedings of the 22nd International Conference on Autonomous Agents and Multiagent Systems (AAMAS)},
  year={2023},
  publisher={International Foundation for Autonomous Agents and Multiagent Systems}
}
\clearpage
\newpage


\appendix
\section{Impact Statement}
\label{app:impact_statement}
This work aims to advance LLM-based multi-agent systems by studying how agents can adapt their communication structures in multi-round settings. The proposed method is evaluated in controlled experimental environments and is intended primarily for research purposes. While it may contribute to improving the robustness and efficiency of collaborative AI systems, it does not introduce direct real-world deployment mechanisms. We do not identify immediate harmful applications specific to this work, but responsible use and further evaluation are important when extending such methods to real-world scenarios.

\section{Graph Execution Algorithm}
\label{sec:app_graph_execution}
\FloatBarrier
\begin{algorithm}[ht!]
   \caption{Graph Execution}
   \label{algo:graph_exec}
\begin{algorithmic}[1]
   \REQUIRE query $q$, inter-agent communication graph $\{\mathcal{G}_C^t\}_{t=2}^{T+1}$, decision graph $\mathcal{G}_D$, round $t \in [2, T]$, nodes $\{v_i\}_{i=1}^N$, decision node $v_d$.
   \FOR{$t = 2, \ldots, T$}
   \FOR{each node $v_i$ in $\text{topologicalSort}(\mathcal{G}_C^t)$}
      \STATE \textcolor{blue}{$\triangleright$ messages from neighbors in inter-agent communication graph at round $t-1$}
      \STATE $\mathbf{m}^t = \{f_m(o_j^{t}, \text{role}_j) \mid v_j \in \mathcal{N}^t(v_i)\}$
      \STATE $p_i^t \sim f_v(q, \mathbf{m}^t)$
      \STATE $o_i^t \sim \text{LLM}_i(p_i^t)$
   \ENDFOR
   \ENDFOR
   \STATE \textcolor{blue}{$\triangleright$ final decision}
   \STATE $\text{ans}_q \leftarrow f_d(\{ o_j^T \mid v_j \in \mathcal{N}^{T+1}(v_d)\})$
   \STATE \textbf{Output:} $\text{ans}_q$
\end{algorithmic}
\end{algorithm}

\section{Inter-agent Communication Graph Construction Algorithm}
\label{sec:app_inter_agent_graph_construction}
\FloatBarrier
\begin{algorithm}[ht!]
\caption{Node Representation Learning and Node Potential Learning over $T$ Rounds}
\label{algo:node_representation_generation}
\begin{algorithmic}[1]
\REQUIRE Query $q$, $\text{GRN}$, $\text{MLP}$, number of rounds $T$, nodes $\{v_i\}_{i=1}^N$
\STATE Initialize hidden states $\{h_i^1\}_{i=1}^N \leftarrow \emptyset$ 
\FOR{$t = 2$ \ldots $T+1$}
    \FOR{$i = 1$ \ldots $N$}
        \STATE $f_i^t \leftarrow  o_i^{t-1}$ and outputs of $\mathcal{N}^{t-1}(v_i)$
        \STATE $h_i^t \leftarrow \text{GRN}(h_i^{t-1}, f_i^t)$
        \STATE $c_i^t \leftarrow \mathrm{MLP}(h_i^t)$
    \ENDFOR
\ENDFOR
\STATE \textbf{Output:} \(\{\{h_i^t\}_{i=1}^{N}\}_{t=2}^{T+1}\), \(\{\{c_i^t\}_{i=1}^{N}\}_{t=2}^{T+1}\)
\end{algorithmic}
\end{algorithm}

\FloatBarrier
\begin{algorithm}[ht!]
\caption{Preliminary Screening with $\epsilon$-Greediness}
\label{algo:screening}
\begin{algorithmic}[1]
\REQUIRE Node potential $\mathbf{c}^t = \{c_i^t\}_{i=1}^N$, exploration rate $\epsilon$, flag \textit{training}, $t \in [2, \ldots, T+1]$
\IF{\textit{training}}
    \STATE {\color{blue}$\triangleright$ Sample edge activity from the policy distribution}
    \STATE $\pi^t \leftarrow \mathrm{Bernoulli}(\mathbf{c}^t)$
    \STATE $\mathbf{a}^t \sim \pi^t$ where $\mathbf{a}^t=(a_1^t,\ldots,a_N^t)$

    \STATE {\color{blue}$\triangleright$ $\epsilon$-greedy exploration}
    
    \STATE $\boldsymbol{\xi}^t \gets (\xi_1^t, \ldots, \xi_N^t)$,
    where $\xi_i^t \sim \mathrm{Bernoulli}(\epsilon)$,
    
    \STATE $\mathbf{a}^t \leftarrow \lvert \mathbf{a}^t - \boldsymbol{\xi}^t \rvert$

\ELSE
    \STATE {\color{blue}$\triangleright$ Deterministic execution at inference}
    \STATE $\kappa \leftarrow \mathrm{Mean}(\mathbf{c}^t) - \tau$ (where hyperparameter $\tau > 0$ instills optimism in including agents)
    \STATE $\mathbf{a}^t \leftarrow \mathbb{I}[\mathbf{c}^t > \kappa]$
\ENDIF
\STATE \textbf{Output:} $\mathbf{a}^t$ 
\end{algorithmic}
\end{algorithm}

\FloatBarrier
\label{app:graph_construction}
\FloatBarrier
\begin{algorithm}[ht!]
\caption{Inter-Agent Communication Graph Construction}
\label{algo:graph_construction}
\begin{algorithmic}[1]
\REQUIRE Node potential $\{\cwhat_i^t\}_{i=1}^N$, screening results $\{a_i^t\}_{i=1}^N$,
edge mask $\mathcal{M}$, optional budgets $B_{\text{out}}, B_{\text{in}}$
\STATE Select participating agents $\mathcal{V}^t_C \leftarrow \{ i \mid a_i^t = 1 \}$
\STATE Sort agents in $\mathcal{V}_C^t$ by descending potential $c_i^t$
\STATE Initialize graph $\mathcal{G}_C^t \leftarrow (\mathcal{V}^t_C, \emptyset)$
\STATE Initialize degrees $d_{\text{out}}(v_i)\leftarrow 0$ and $d_{\text{in}}(v_i)\leftarrow 0$ for all $v_i \in \mathcal{V}^t_C$

\STATE Let $\Lcal$ be a sorted list of $(v_i, v_j) \in \mathcal{V}_C^t \times \mathcal{V}_C^t$, 
where the entries are sorted in a descending order of $\cwhat^t_i \cwhat^t_j$.
$(v_i, v_j)$ precedes $(v_j, v_i)$ if $c_i^t \ge c_j^t$.

\FOR{(sender $v_i$, receiver $v_j$) pairs enumerated from $\Lcal$}
        \STATE \textcolor{blue}{$\triangleright$ no self cycle}
        \IF{$v_i =v_j$}
            \STATE \textbf{continue}
        \ENDIF
        \STATE \textcolor{blue}{$\triangleright$ check screening result}
        \IF{$a_i^t = 0 \;\text{or}\; a_j^t = 0$}
            \STATE \textbf{continue}
        \ENDIF
        \STATE \textcolor{blue}{$\triangleright$ predefined mask constraint}
        \IF{$\mathcal{M}(i,j)=0$}
            \STATE \textbf{continue}
        \ENDIF
        \STATE \textcolor{blue}{$\triangleright$ (Optional) budget constraint}
        \IF{$d_{\text{out}}(v_i) \ge B_{\text{out}}$ \OR $d_{\text{in}}(v_j) \ge B_{\text{in}}$}
            \STATE \textbf{continue} 
        \ENDIF
        \STATE \textcolor{blue}{$\triangleright$ DAG constraint}
        \IF{adding edge $(i \rightarrow j)$ creates a cycle in $\mathcal{G}_C^t$}
            \STATE \textbf{continue}
        \ENDIF
        \STATE Add edge $(i \rightarrow j)$ to $\mathcal{G}_C^t$
        \STATE \textcolor{blue}{$\triangleright$ (Optional) budget count}
        \STATE $d_{\text{out}}(v_i) \leftarrow d_{\text{out}}(v_i) + 1$
        \STATE $d_{\text{in}}(v_j) \leftarrow d_{\text{in}}(v_j) + 1$
\ENDFOR
\STATE \textbf{Output:} $\mathcal{G}_C^t$
\end{algorithmic}
\end{algorithm}

\clearpage
\newpage

\section{Additional Results on Training Details}
\begin{table*}[ht!]
\centering
\caption{Training token usage at performance parity point and peak point. The \textit{performance parity point} is defined as the point where TodyComm rapidly converges to a level close to its peak accuracy, after which further improvement is slow and marginal. For each method and attack rate, we report the mean training tokens to reach performance parity (Tok\textsubscript{par}), accuracy at parity point (Acc\textsubscript{par}), tokens to reach peak performance (Tok\textsubscript{peak}), and peak accuracy (Acc\textsubscript{peak}). Tokens are in millions. ``-'' indicates the metric is not applicable for baselines.}
\label{tab:training_efficiency}
\resizebox{\textwidth}{!}{%
\begin{tabular}{ll cccc cccc cccc cccc}
\toprule
& & \multicolumn{4}{c}{\textbf{MMLU}} & \multicolumn{4}{c}{\textbf{ARC}} & \multicolumn{4}{c}{\textbf{OpenBookQA}} & \multicolumn{4}{c}{\textbf{MedQA}} \\
\cmidrule(lr){3-6} \cmidrule(lr){7-10} \cmidrule(lr){11-14} \cmidrule(lr){15-18}
\textbf{Method} & \textbf{Attack Rate}
  & Tok\textsubscript{par} & Acc\textsubscript{par} & Tok\textsubscript{peak} & Acc\textsubscript{peak}
  & Tok\textsubscript{par} & Acc\textsubscript{par} & Tok\textsubscript{peak} & Acc\textsubscript{peak}
  & Tok\textsubscript{par} & Acc\textsubscript{par} & Tok\textsubscript{peak} & Acc\textsubscript{peak}
  & Tok\textsubscript{par} & Acc\textsubscript{par} & Tok\textsubscript{peak} & Acc\textsubscript{peak} \\
\midrule

\multirow{3}{*}{TodyComm}
 & $<$50\% & 14.00 & 69.94 & 47.16 & 72.11 & 1.32  & 87.74 & 5.81  & 88.52 & 2.91  & 81.83 & 20.51 & 84.83 & 7.81  & 59.67 & 39.38 & 61.00 \\
 & $=$50\% & 9.14  & 66.01 & 54.47 & 68.85 & 12.52 & 82.72 & 24.36 & 85.95 & 5.55  & 80.50 & 12.47 & 84.17 & 16.39 & 58.83 & 23.55 & 60.17 \\
 & $>$50\% & 17.68 & 64.05 & 34.75 & 64.71 & 25.90 & 78.48 & 30.59 & 81.05 & 10.10 & 78.50 & 12.76 & 80.50 & 12.21 & 56.00 & 19.47 & 57.50 \\
\midrule

\multirow{3}{*}{GDesigner}
 & $<$50\% & - & - & 30.09 & 64.92 & - & - & 21.48 & 83.05 & - & - & 15.79 & 73.67 & - & - & 23.59 & 67.00 \\
 & $=$50\% & - & - & 31.39 & 52.07 & -  & - & 2.21  & 72.91 & - & - & 20.96 & 40.83 & - & - & 33.05 & 48.33 \\
 & $>$50\% & - & - & 22.33 & 45.97 & - & - & 26.59 & 67.22 & - & - & 26.71 & 41.00 & - & - & 32.88 & 43.17 \\
\midrule

\multirow{3}{*}{AgentPrune}
 & $<$50\% & - & - & 40.90 & 69.28 & - & - & 20.93 & 86.73 & - & - & 21.19 & 73.83 & - & - & 33.02 & 66.65 \\
 & $=$50\% & - & - & 36.80 & 55.34 & - & - & 24.15 & 74.02 & - & - & 14.29 & 44.17 & - & - & 32.70 & 48.50 \\
 & $>$50\% & - & - & 51.33 & 53.38 & - & - & 19.88 & 69.79 & -  & - & 8.60  & 39.17 & - & - & 22.86 & 45.67 \\

\bottomrule
\end{tabular}%
}
\end{table*}

\begin{figure}[ht!]
    \centering
    \begin{subfigure}{0.24\linewidth}
        \includegraphics[width=\linewidth]{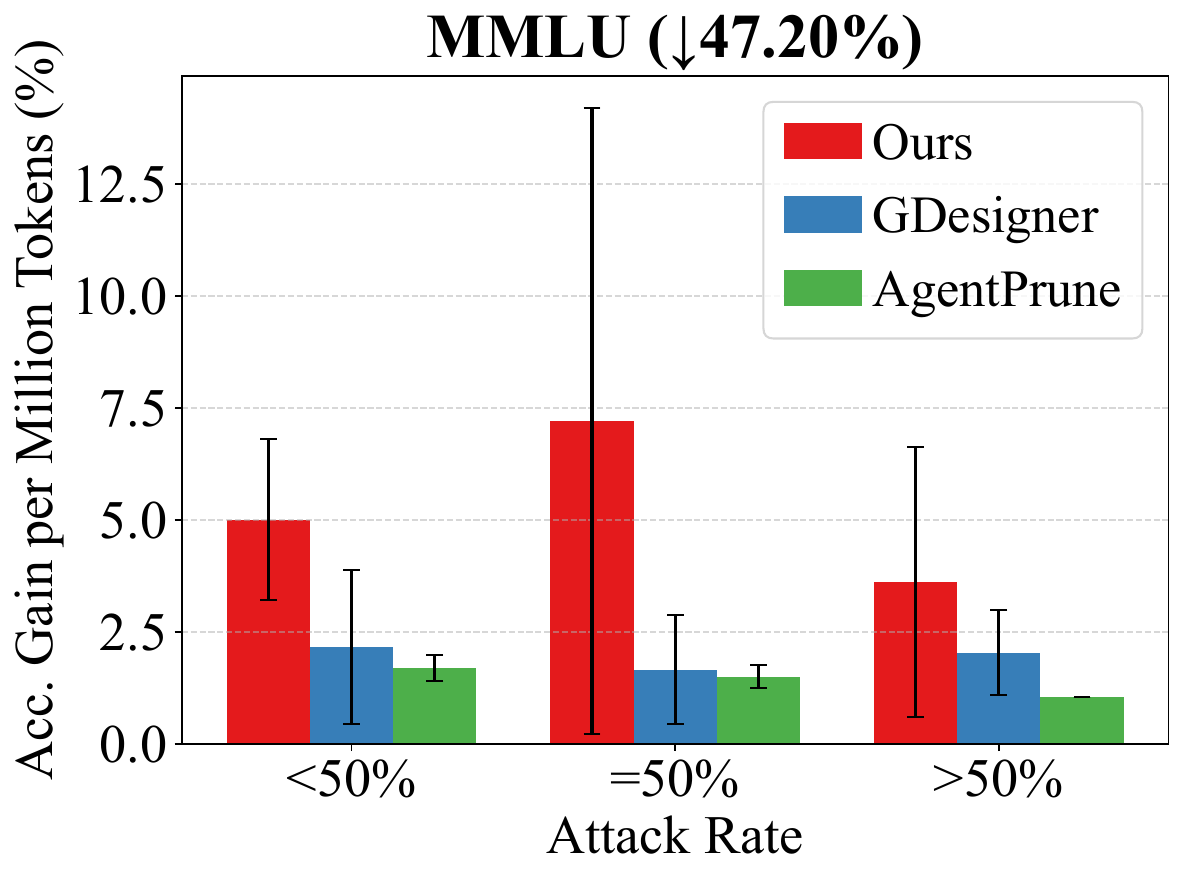}
    \end{subfigure}
    \begin{subfigure}{0.24\linewidth}
        \includegraphics[width=\linewidth]{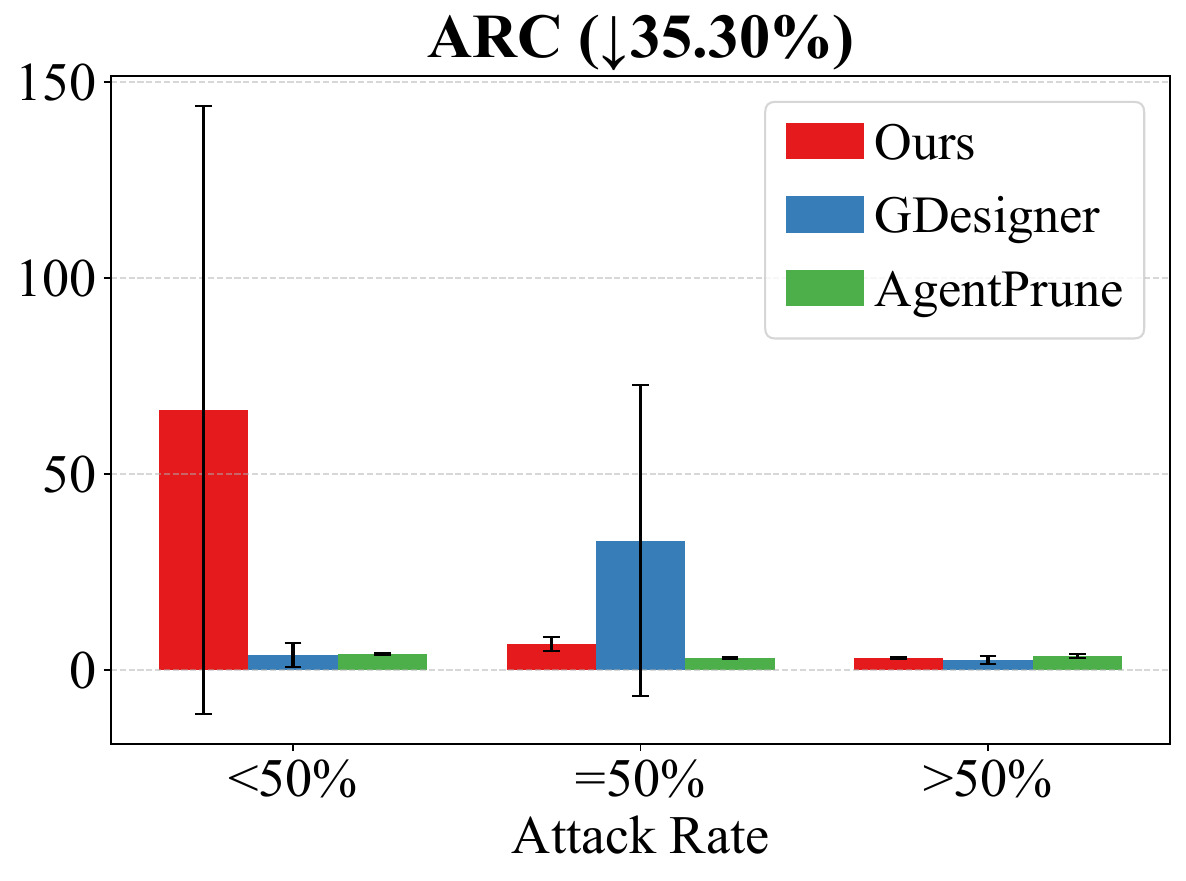}
    \end{subfigure}
    \begin{subfigure}{0.24\linewidth}
        \includegraphics[width=\linewidth]{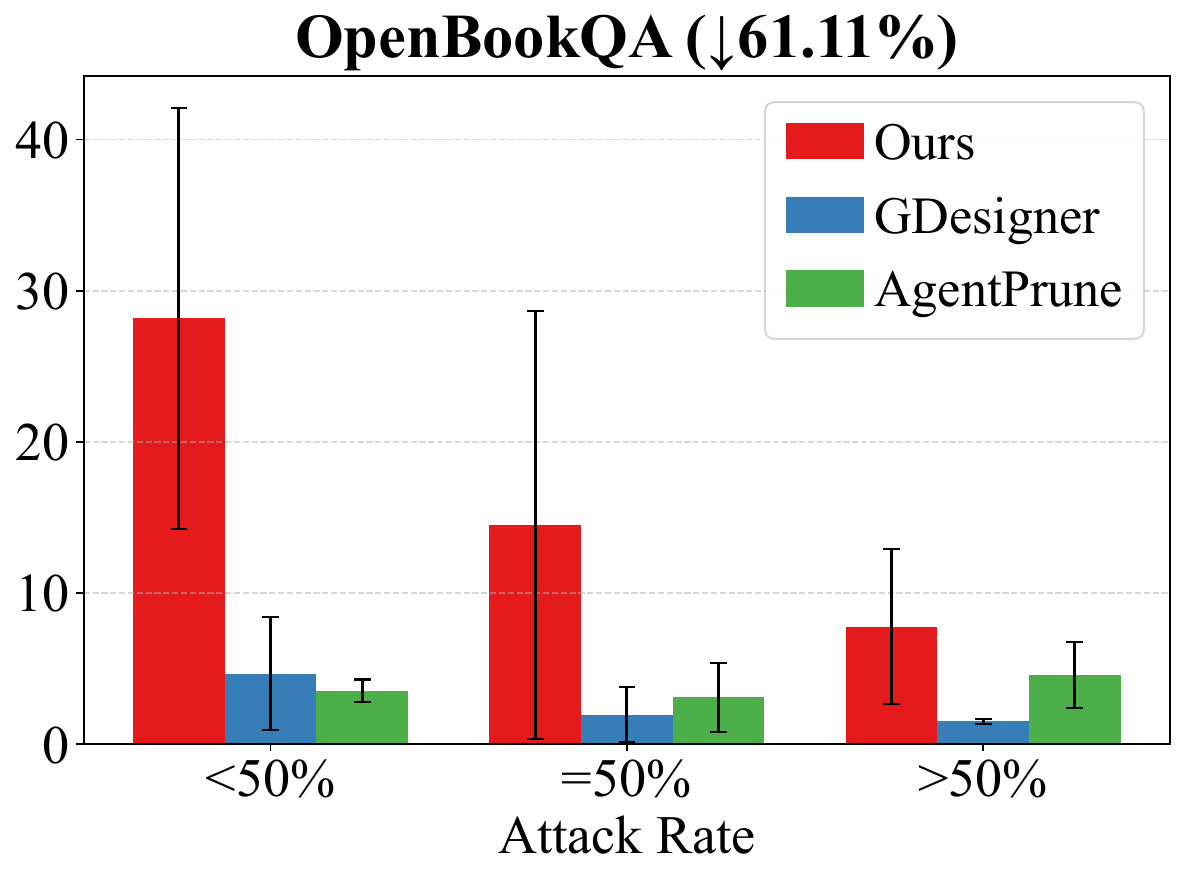}
    \end{subfigure}
    \begin{subfigure}{0.24\linewidth}
        \includegraphics[width=\linewidth]{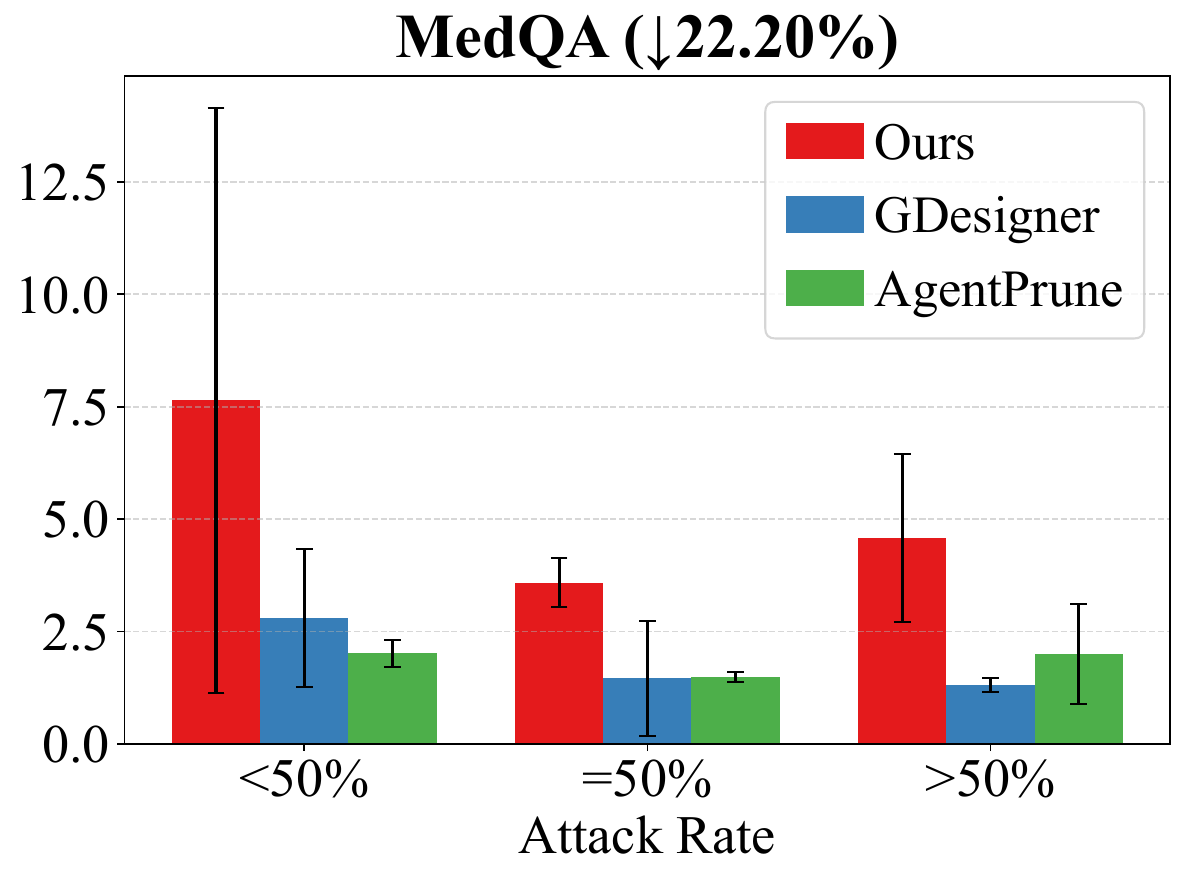}
    \end{subfigure}
    \caption{Accuracy gain per million training tokens across four benchmarks. TodyComm achieves higher accuracy with fewer training tokens than GDesigner and AgentPrune across all benchmarks and adversary rates in most cases.}
    \label{fig:training_efficiency}
\end{figure}

\begin{figure*}[ht!]
    \centering
    \begin{subfigure}{0.32\textwidth}
        \centering
        \includegraphics[width=0.9\linewidth]{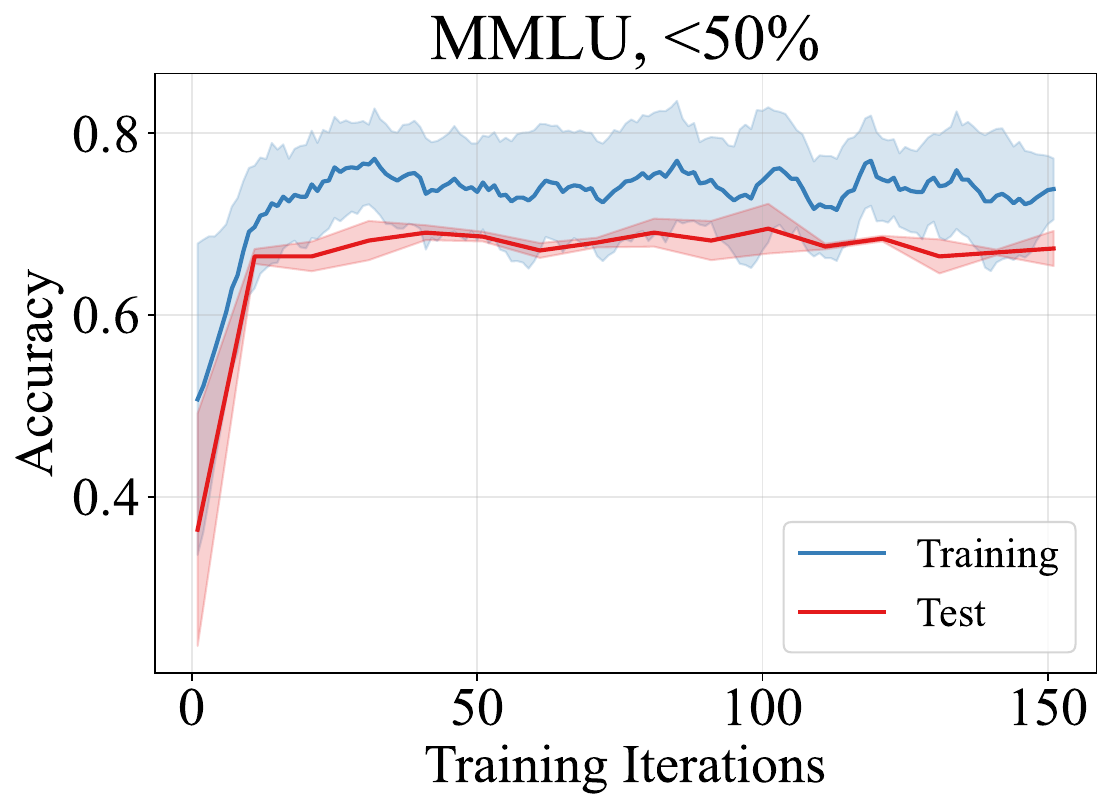}
    \end{subfigure}
    \hfill
    \begin{subfigure}{0.32\textwidth}
        \centering
        \includegraphics[width=0.9\linewidth]{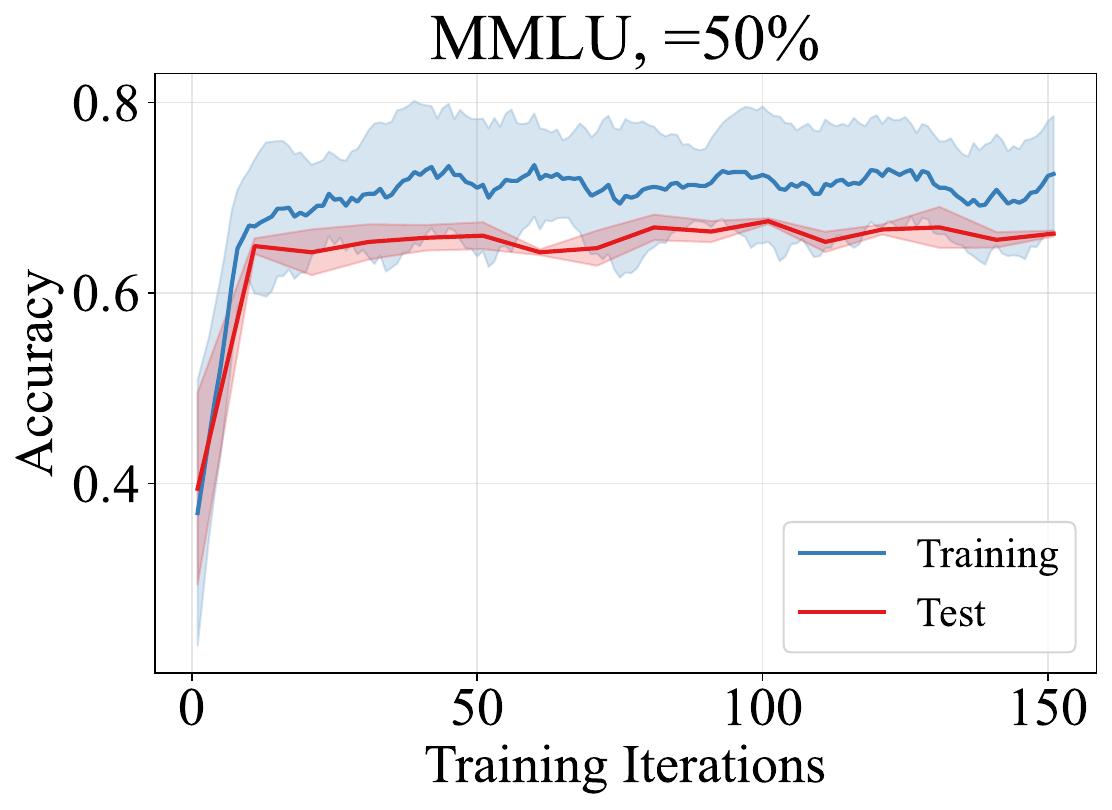}
    \end{subfigure}
    \hfill
    \begin{subfigure}{0.32\textwidth}
        \centering
        \includegraphics[width=0.9\linewidth]{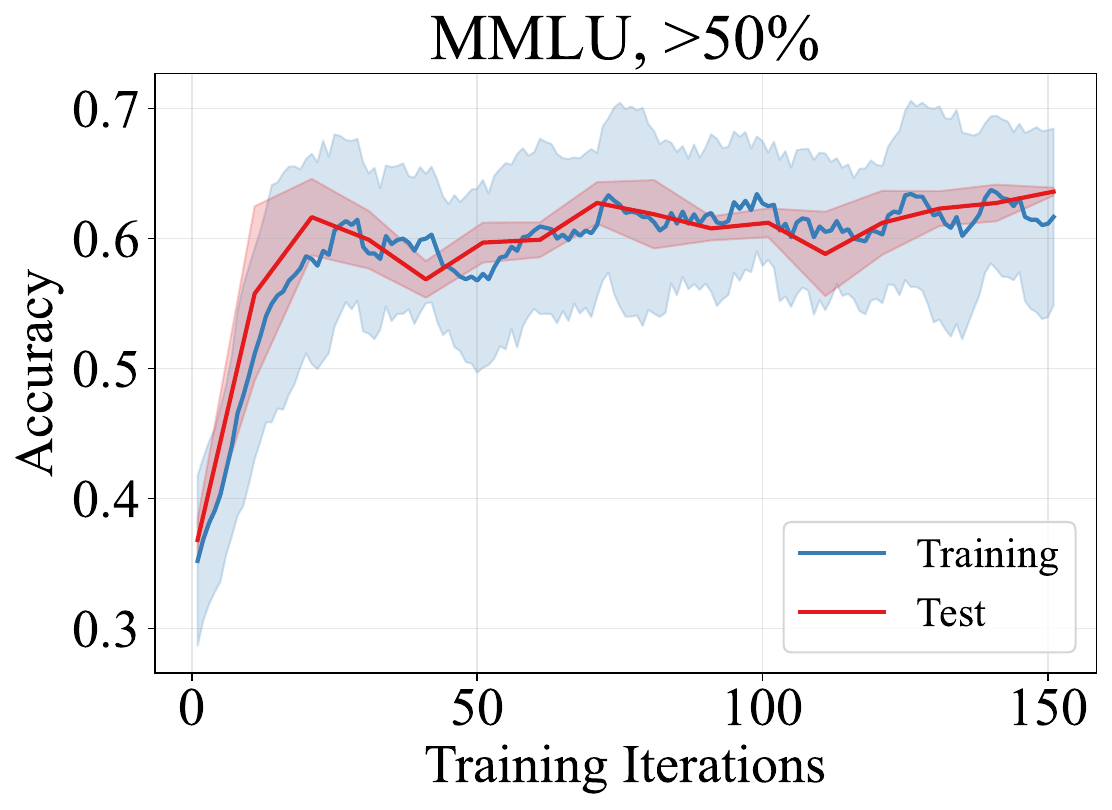}
    \end{subfigure}
    \caption{Training vs. test accuracy over training iterations on MMLU under three adversary rates ($<$50\%, $=$50\%, $>$50\%). Training accuracy consistently exceeds test accuracy and both curves plateau stably, confirming no overfitting or data leakage.}
    \label{fig:overfitting}
\end{figure*}

\clearpage
\newpage

\FloatBarrier
\section{Generalization}
\label{sec:app_generalization}
\subsection{Number of Agents Generalization}
\label{app:num_agents_gen}
We trained on fewer agents and evaluate on more to demonstrate that our way of learning edges generalize beyond the training configuration, capturing meaningful agent relationships regardless of the number of agents.
Table \ref{tab:num_agent_generalization} shows TodyComm generalizes from 4-agent training to 6-agent evaluation with negligible degradation on MMLU and marginal improvement on ARC-C, 
consistently outperforming baselines across all attack rates.
\vspace{-1em}
\FloatBarrier
\vspace{-2pt}
\begin{table}[ht!]
\centering
\caption{Generalization on different number of agents in training and inference with gpt-4.1-nano.}
\label{tab:num_agent_generalization}
\footnotesize
\setlength{\tabcolsep}{5pt}
\renewcommand{\arraystretch}{1}
\begin{tabular}{l l c c c c c}
    \toprule
    \textbf{Dataset} & \textbf{Method} & \textbf{\#agents (train)} & \textbf{\#agents (eval)} & \textbf{$<50\%$} & \textbf{$=50\%$} & \textbf{$>50\%$} \\
    \midrule
    MMLU & TodyComm   & 6 & 6 & 72.11 & 68.85 & 64.71 \\
    MMLU & TodyComm   & 4 & 6 & 71.24 & 67.32 & 64.05 \\
    MMLU & AgentPrune & 6 & 6 & 69.28 & 55.34 & 53.38 \\
    MMLU & GDesigner  & 6 & 6 & 64.92 & 52.07 & 45.97 \\
    \midrule
    ARC-C & TodyComm   & 6 & 6 & 88.52 & 85.95 & 81.05 \\
    ARC-C & TodyComm   & 4 & 6 & 88.96 & 86.62 & 84.61 \\
    ARC-C & AgentPrune & 6 & 6 & 86.73 & 74.02 & 69.79 \\
    ARC-C & GDesigner  & 6 & 6 & 83.05 & 72.91 & 67.22 \\
    \bottomrule
\end{tabular}
\end{table}

\subsection{Task Domain Generalization}
\label{app:task_domain_gen}
We trained on one dataset and evaluate on a different one. Specifically, we use MMLU and ARC-C as evaluation targets, each trained on the remaining QA datasets. As shown in the Table \ref{tab:cross_dataset}, TodyComm retains strong performance across all training-evaluation combinations.

\FloatBarrier
\vspace{-2pt}
\begin{table}[ht!]
\centering
\caption{Generalization across training and evaluation datasets under different attack rates with gpt-4.1-nano.}
\label{tab:cross_dataset}
\footnotesize
\setlength{\tabcolsep}{5pt}
\renewcommand{\arraystretch}{1}
\begin{tabular}{l l c c c}
    \toprule
    \textbf{Training} & \textbf{Evaluation} & \textbf{$<50\%$} & \textbf{$=50\%$} & \textbf{$>50\%$} \\
    \midrule
    
    
    
    MMLU  & MMLU  & 72.11 & 68.85 & 64.71 \\
    ARC-C & MMLU  & 72.50 & 66.67 & 62.50 \\
    MedQA & MMLU  & 69.94 & 66.67 & 62.75 \\
    OBQA  & MMLU  & 70.59 & 67.97 & 63.40 \\
    \midrule
    ARC-C & ARC-C & 88.52 & 85.95 & 81.05 \\
    MMLU  & ARC-C & 85.62 & 85.28 & 81.61 \\
    MedQA & ARC-C & 85.28 & 85.28 & 80.27 \\
    OBQA  & ARC-C & 85.28 & 86.29 & 81.27 \\
    \bottomrule
\end{tabular}
\end{table}

\FloatBarrier
\subsection{Adversary outbreak generalization}
\label{app:adv_inception_gen}
We trained with one attack Outbreak modes (\textit{random} or \textit{fixed}) and evaluate with the other, across different dataset combinations. As shown in the Table \ref{tab:attack_inception}, TodyComm generalizes well across attack Outbreak modes: performance remains consistently strong and comparable to the in-distribution baseline (first and fourth rows). Notably, models trained on the random mode and evaluated on fixed (rows 2 and 5) tend to perform slightly better than the converse (rows 3 and 6), which is expected since random training exposes the model to a more diverse set of adversarial timings, providing a broader inductive signal.

\FloatBarrier
\begin{table}[ht!]
\centering
\caption{Generalization across training and evaluation attack outbreak strategies under different attack rates with gpt-4.1-nano.}
\label{tab:attack_inception}
\footnotesize
\setlength{\tabcolsep}{5pt}
\renewcommand{\arraystretch}{1}
\begin{tabular}{ll|ll|ccc}
\toprule
\multicolumn{2}{c|}{\textbf{Training}} & \multicolumn{2}{c|}{\textbf{Evaluation}} & \multicolumn{3}{c}{\textbf{Attack Rate}} \\
\cmidrule(lr){1-2} \cmidrule(lr){3-4} \cmidrule(lr){5-7}

\textbf{Dataset} & \textbf{Outbreak} 
& \textbf{Dataset} & \textbf{Outbreak} 
& \textbf{$<50\%$} & \textbf{$=50\%$} & \textbf{$>50\%$} \\
    \midrule
    MMLU  & random & MMLU  & random & 72.11 & 68.85 & 64.71 \\
    ARC-C & random & MMLU  & fixed  & 71.90 & 69.94 & 67.32 \\
    ARC-C & fixed  & MMLU  & random & 71.24 & 68.63 & 60.13 \\
    \midrule
    ARC-C & random & ARC-C & random & 88.52 & 85.95 & 81.05 \\
    MMLU  & random & ARC-C & fixed  & 87.62 & 86.62 & 84.28 \\
    MMLU  & fixed  & ARC-C & random & 87.29 & 83.61 & 81.94 \\
    \bottomrule
\end{tabular}
\end{table}

\FloatBarrier
\subsection{Attack Mechanism Generalization}
Table~\ref{tab:diff_attack_in_train_infer} varies the attack mechanism between training and inference, and TodyComm remains competitive across all combinations, sometimes even slightly exceeding the matched-attack baseline. 
\FloatBarrier
\begin{table*}[ht!]
\centering
\caption{Generalization across attack mechanisms in training and inference with Qwen3-8B.}
\label{tab:diff_attack_in_train_infer}
\footnotesize
\setlength{\tabcolsep}{3pt}
\renewcommand{\arraystretch}{1}


\begin{tabular}{@{}c|c|ccc|ccc@{}}
\toprule
\multicolumn{1}{c|}{\textbf{Training}} & \multicolumn{1}{c|}{\textbf{Evaluation}} 
& \multicolumn{3}{c|}{\textbf{MMLU}} 
& \multicolumn{3}{c}{\textbf{ARC-C}} \\

\cmidrule(lr){1-1} \cmidrule(lr){2-2} \cmidrule(lr){3-5} \cmidrule(lr){6-8}

\textbf{Mechanism} & \textbf{Mechanism}
& \textbf{$<50\%$} & \textbf{$=50\%$} & \textbf{$>50\%$}
& \textbf{$<50\%$} & \textbf{$=50\%$} & \textbf{$>50\%$} \\

\midrule
targeted  & targeted  & 69.28 & 67.32 & 60.13 & 88.29 & 81.94 & 79.26 \\
untargeted & targeted  & 69.94 & 65.36 & 49.67 & 87.63 & 84.95 & 76.92 \\
targeted  & untargeted  & 67.32 & 65.36 & 55.57 & 88.63 & 86.29 & 75.92 \\
untargeted  & untargeted & 67.97 & 67.97 & 59.48 & 88.29 & 83.61 & 74.58 \\

\bottomrule
\end{tabular}
\end{table*}

\subsection{Combination}
\label{app:combo_gen}
We further stress-tested generalization by combining multiple factors simultaneously. 
Table~\ref{tab:3factors_generalization} compounds three factors:task domain, attack outbreak mode, and attack mechanism, and TodyComm continues to generalize well, with accuracy remaining strong even under fully mismatched train/eval configurations. 
Table~\ref{tab:4factors_generalization} adds attack rate as a fourth factor, creating the most challenging cross-condition setting. TodyComm still achieves solid performance in most cases, 
though accuracy drops more noticeably when the inference attack rate is substantially higher than training (e.g., trained at 0.4, evaluated at 0.7) 
and trained on fixed outbreak and evaluated on random, as the model encounters a more unpredictable adversarial pattern at a higher attack intensity than seen during training.

\FloatBarrier
\vspace{-1em}
\begin{table}[ht!]
\centering
\caption{Generalization across task domain, attack outbreak, and attack mechanism with Qwen3-8B.}
\label{tab:3factors_generalization}
\footnotesize
\setlength{\tabcolsep}{3pt}
\renewcommand{\arraystretch}{1.05}





\begin{tabular}{@{}ccc|ccc|ccc@{}}
\toprule

\multicolumn{3}{c|}{\textbf{Training}} 
& \multicolumn{3}{c|}{\textbf{Evaluation}} 
& \multicolumn{3}{c}{\textbf{Attack Rate}} \\

\cmidrule(lr){1-3} \cmidrule(lr){4-6} \cmidrule(lr){7-9}

\textbf{Dataset} & \textbf{Mechanism} & \textbf{Outbreak} 
& \textbf{Dataset} & \textbf{Mechanism} & \textbf{Outbreak} 
& \textbf{$<50\%$} & \textbf{$=50\%$} & \textbf{$>50\%$} \\

\midrule

ARC-C & targeted  & random & MMLU & untargeted & fixed  & 69.28 & 67.32 & 65.36 \\
ARC-C & untarted & fixed  & MMLU & targeted  & random & 71.90 & 61.49 & 56.86 \\
MMLU  & targeted  & fixed  & ARC-C & untargeted & random & 86.96 & 78.93 & 61.54 \\
MMLU  & untargeted & random & ARC-C & targeted  & fixed  & 89.30 & 89.00 & 88.29 \\

\bottomrule
\end{tabular}
\end{table}

\begin{table}[ht]
\centering
\caption{Generalization across task domain, attack outbreak, attack mechanism, and attack rate with Qwen3-8B.}
\label{tab:4factors_generalization}
\footnotesize
\setlength{\tabcolsep}{5pt}
\renewcommand{\arraystretch}{1}

\begin{tabular}{@{}cccc|cccc|c@{}}
\toprule

\multicolumn{4}{c|}{\textbf{Training}} 
& \multicolumn{4}{c|}{\textbf{Evaluation}} 
& \multirow{2}{*}{\textbf{Acc}} \\

\cmidrule(lr){1-4} \cmidrule(lr){5-8}

\textbf{Dataset} & \textbf{Mechanism} & \textbf{Outbreak} & \textbf{Rate}
& \textbf{Dataset} & \textbf{Mechanism} & \textbf{Outbreak} & \textbf{Rate}
& \\

\midrule

ARC-C & targeted  & random & 0.7 & MMLU & untargeted & fixed  & 0.4 & 70.59 \\
ARC-C & untargeted & fixed  & 0.7 & MMLU & targeted  & random & 0.4 & 68.63 \\
ARC-C & targeted  & random & 0.4 & MMLU & untargeted & fixed  & 0.7 & 66.01 \\
ARC-C & untargeted & fixed  & 0.4 & MMLU & targeted  & random & 0.7 & 45.10 \\
MMLU  & untargeted & random & 0.7 & ARC-C & targeted  & fixed  & 0.4 & 89.30 \\
MMLU  & targeted  & fixed  & 0.7 & ARC-C & untargeted & random & 0.4 & 88.29 \\
MMLU  & untargeted & random & 0.4 & ARC-C & targeted  & fixed  & 0.7 & 84.28 \\
MMLU  & targeted  & fixed  & 0.4 & ARC-C & untargeted & random & 0.7 & 67.89 \\

\bottomrule
\end{tabular}
\end{table}

\FloatBarrier
\section{Ablation on Graph Construction}
\label{app:graph_construction_ablation}
\FloatBarrier
\vspace{-1em}
\begin{table}[ht!]
\centering
\caption{Graph Construction Ablation on MMLU}
\vspace{0pt}
\label{tab:ablation_method_design}
\footnotesize
\setlength{\tabcolsep}{6pt}        
\renewcommand{\arraystretch}{1} 
\begin{tabular}{l | c | c | c}
    \toprule
    \textbf{Method} & \textbf{$<50\%$} & \textbf{$=50\%$} & \textbf{$>50\%$} \\
    \midrule
    \textbf{TodyComm}             & \textbf{72.11} & \textbf{68.85} & \textbf{64.71} \\
    TodyComm w/o learning & 62.11          & 54.47          & 49.70          \\
    TodyComm w random ordering & 69.94          & 66.01          & 62.75 
 \\
    \bottomrule
\end{tabular}
\end{table}

\FloatBarrier
\section{Definition of Node Features}
\label{app:definition_node_features}
The node feature is defined as:
\begin{align}
    {f}_i^t &:= \{ \, \underbrace{\mathbf{s}_i^{t-1}}_{\text{Self}} \,\, \Vert \,\, \underbrace{\mathbf{n}_i^{t-1}}_{\text{Neighbor}} \,\, \Vert \,\, \underbrace{\mathbf{d}_i^{t-1}}_{\text{Difference}} \, \} \\[10pt]
    \text{where: } \quad \mathbf{s}_i^t &= \{ \text{sol}_i^t, \text{ana}_i^t, \text{sol\_per}_i^t, \text{ana\_per}_i^t \} \nonumber \\
    \mathbf{n}_i^t &= \{ \text{neighbor\_sol}_i^t, \text{neighbor\_ana}_i^t, \text{has\_neighbor}_i^t, \nonumber \\
                 & \quad \quad \text{sol\_disp}_i^t, \text{ana\_disp}_i^t \} \nonumber \\
    \mathbf{d}_i^t &= \{ \text{sol\_agree}_i^t, \text{ana\_agree}_i^t, \text{sol\_dist}_i^t, \text{ana\_dist}_i^t \} \nonumber
\end{align}
The self information captures agent \(i\)’s own outputs from the previous round, 
including its solution \(\text{sol}_i^{t-1}\) and analysis \(\text{ana}_i^{t-1}\), as well as persistence features \(\text{sol\_per}_i^{t-1}\) and \(\text{ana\_per}_i^{t-1}\), defined as the similarity between the current outputs and those from the initial round. These features measure the temporal consistency of an agent’s reasoning.

The neighborhood information summarizes neighboring agents’ behaviors, including potential-weighted averages of neighbors’ solutions and analyses, 
along with a binary indicator \(\text{has\_neighbor}_i^{t-1}\). 
We further incorporate dispersion metrics \(\text{sol\_disp}_i^{t-1}\) and \(\text{ana\_disp}_i^{t-1}\) to quantify the degree of disagreement among neighbors.

Finally, the difference information explicitly models the parity between agent \(i\) and its neighborhood, including semantic agreement scores and distance measures that capture how closely the agent’s outputs align with or deviate from those of its neighbors.

\clearpage
\newpage

\FloatBarrier
\section{Ablation on Node Features}
\label{app:node_features_ablation}
In Section~\ref{sec:inter_agent_opt}, node features consist of \textbf{self} information $\svec_i^t$, 
\textbf{neighborhood} information $\nvec_i^t$, 
and their \textbf{difference} $\dvec_i^t$. 

Table~\ref{tab:ablation_node_features} shows that each component is necessary.
For example, ``Self + Diff'' emphasizes the importance of neighborhood information, 
while ``Self + Neighbor'' demonstrates the value of explicitly modeling difference information.

We define agent persistence (part of $\svec_i^t$) as the difference between an agent’s independent answer at the outset and its current-round answer after communication. To evaluate its role, we consider ``Self info only'', which includes persistence but excludes neighborhood information, and ``w/o Persist'', which removes persistence. The results indicate that persistence is beneficial, but persistence alone without neighborhood information is insufficient.

As shown in Figure~\ref{fig:ablation_node_features}, node features combining persistence with neighborhood-related information—via either difference or explicit neighborhood features—converge rapidly within about 20 training iterations. 
In contrast, ``Self info only'' converges much more slowly, and persistence without neighborhood information performs poorly, particularly when 4 out of 6 agents are adversarial. Without persistence, the model can still learn but converges more slowly.

Overall, across different node feature designs, our method performs comparably to baselines when the attack rate is below 50\% and significantly outperforms them when the attack rate exceeds 50\%. 


\begin{table}[ht!]
\centering
\caption{Node Features Ablation on MMLU. ``lr'' stands for learning rate.}
\vspace{0pt}
\label{tab:ablation_node_features}
\footnotesize
\setlength{\tabcolsep}{6pt}        
\renewcommand{\arraystretch}{1} 
\begin{tabular}{l | c | c |  c}
    \toprule
    \textbf{Method} & \textbf{$<50\%$} & \textbf{$=50\%$} & \textbf{$>50\%$} \\
    \midrule
    \textbf{TodyComm}                & \textbf{72.11} & \textbf{68.85} & 64.71 \\
    Self + Diff info (lr = $5e^{-4}$) & 70.81 & 67.76 & \textbf{65.14} \\
    Self + Diff info (lr = $1e^{-4}$) & \textbf{72.11} & 68.19 & 62.53 \\
    Self + Neighbor info         & 70.81 & 67.54 & 62.31 \\
    Self info only               & 68.30 & 63.73 & 41.18 \\
    TodyComm w/o Persist             & 67.65 & 63.40 & 53.27 \\
    \bottomrule
\end{tabular}
\end{table}

\begin{figure*}[ht!]
    \centering
    \begin{subfigure}[t]{0.32\textwidth}
        \centering
        \includegraphics[width=\linewidth]{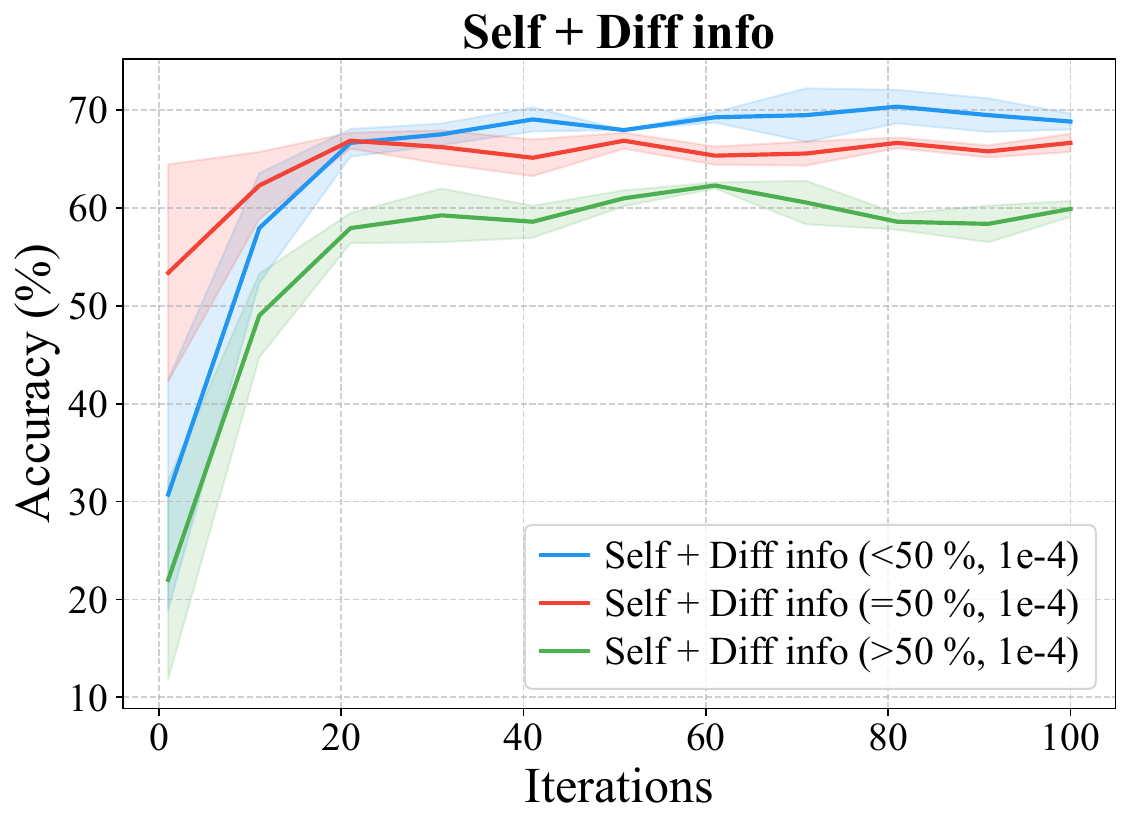}
    \end{subfigure}
    \hfill
    \begin{subfigure}[t]{0.32\textwidth}
        \centering
        \includegraphics[width=\linewidth]{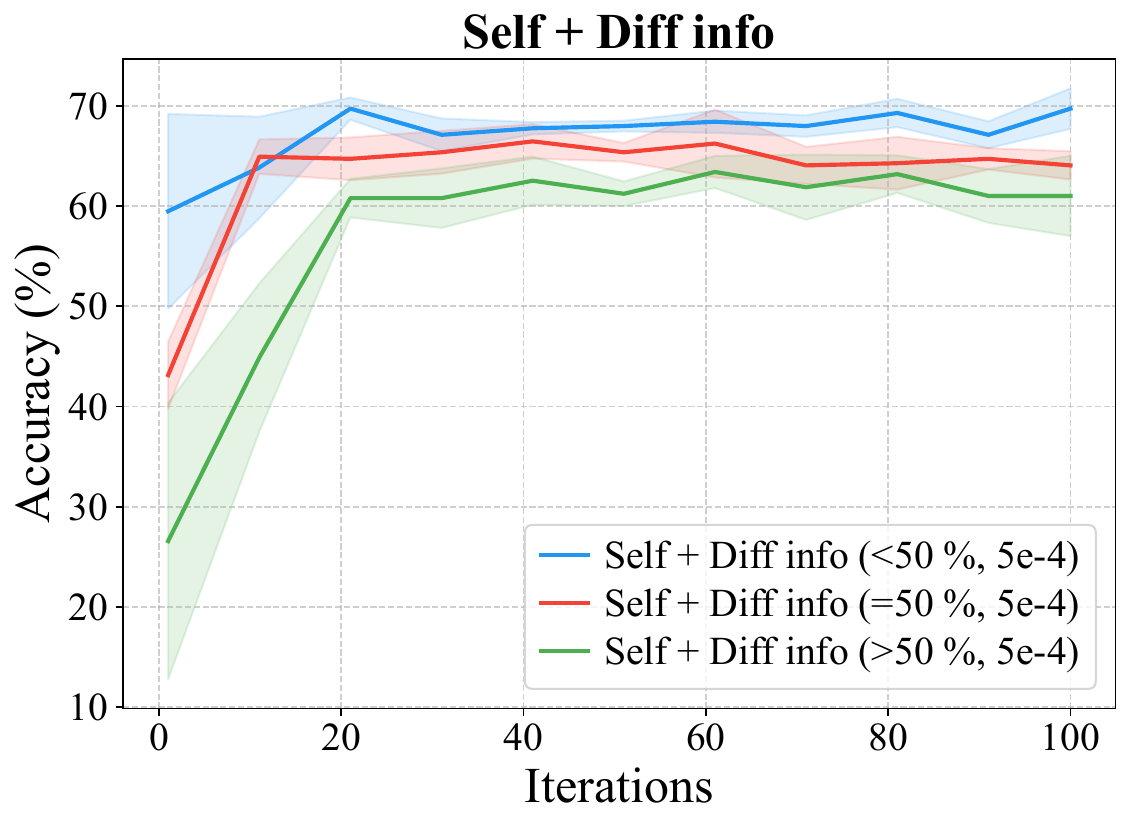}
    \end{subfigure}
    \hfill
    \begin{subfigure}[t]{0.32\textwidth}
        \centering
        \includegraphics[width=\linewidth]{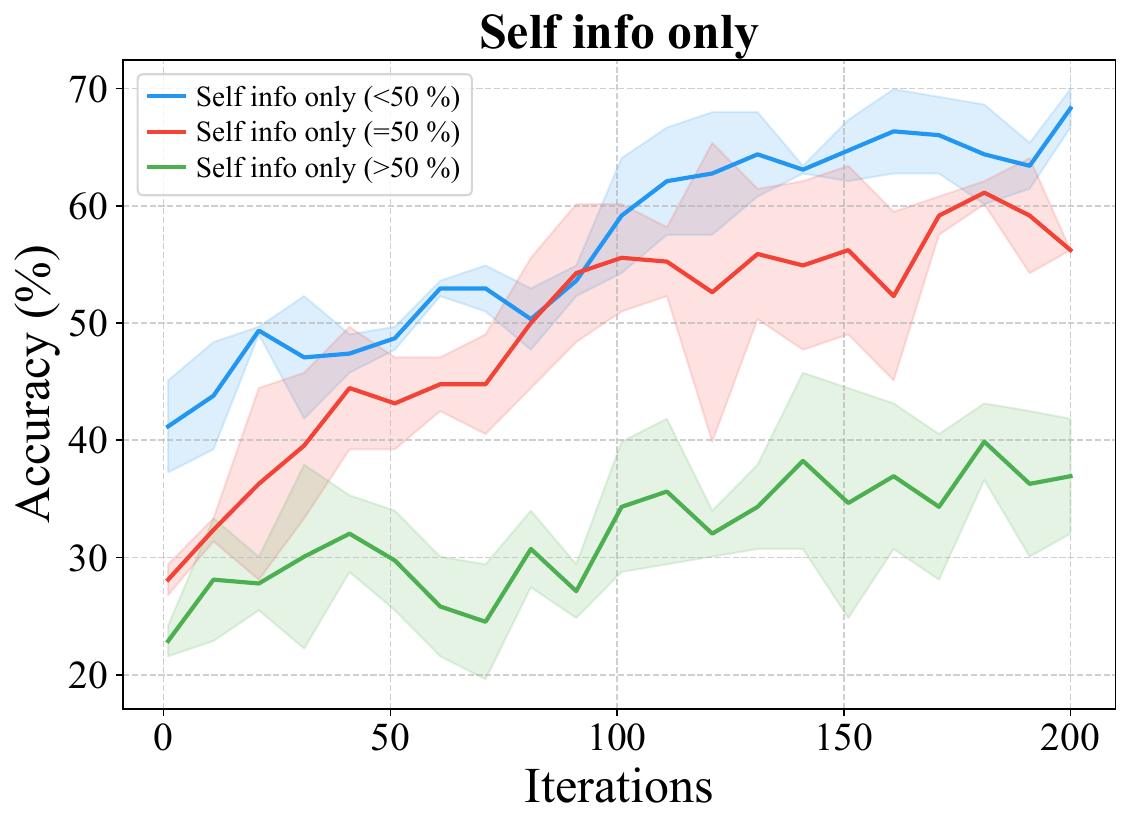}
    \end{subfigure}

    \vspace{0.1em}

    \begin{subfigure}[t]{0.4\textwidth}
        \centering
        \includegraphics[width=0.8\linewidth]{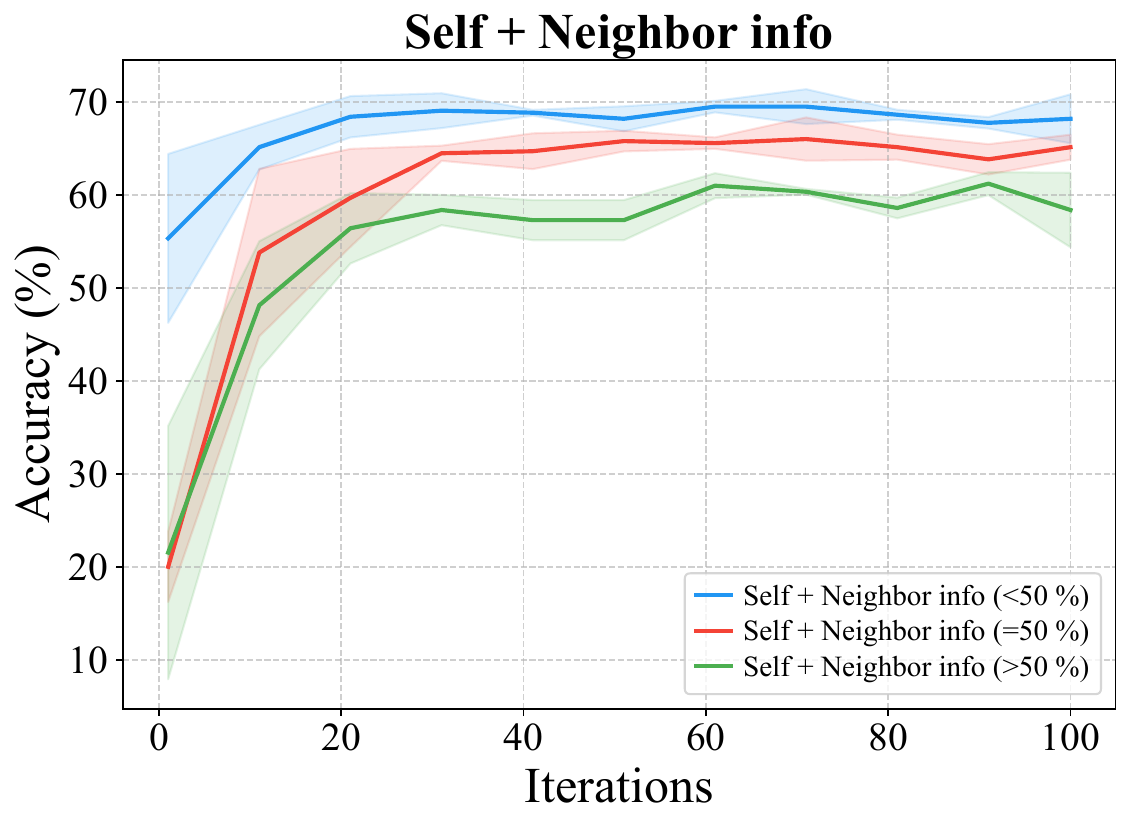}
    \end{subfigure}
    \hspace{-3em}
    \begin{subfigure}[t]{0.4\textwidth}
        \centering
        \includegraphics[width=0.8\linewidth]{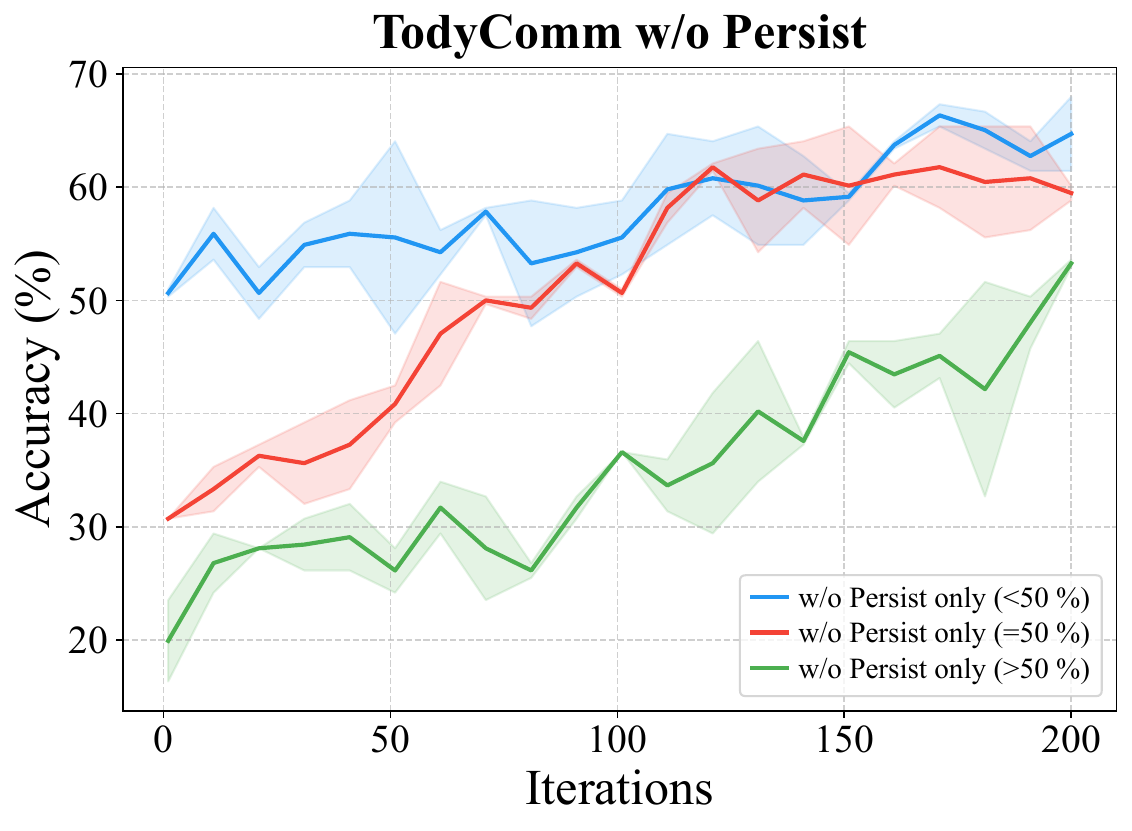}
    \end{subfigure}
    \caption{Training curves of node feature ablations for TodyComm on MMLU:test  accuracy v.s. training iterations. $5e^{-4}$ and $1e^{-4}$ in the legend refer to the learning rate of REINFORCE.}
    \label{fig:ablation_node_features}
\end{figure*}
\clearpage
\newpage

\FloatBarrier
\section{Choices of LLMs}
\label{app:choice_llms}
We evaluated TodyComm on three additional LLMs of varying sizes and model families (Qwen-flash, Mistral-Small-24B, and Llama-4-Maverick) to further validate its robustness. As shown in Table~\ref{tab:llm_ablation}, TodyComm consistently outperforms G-Designer and AgentPrune across all three models and benchmarks, with the advantage becoming more pronounced at higher attack rates. Combined with the main results on GPT-4.1-nano and the generalization experiments on Qwen3-8B, TodyComm demonstrates strong and consistent robustness across five diverse LLMs.
\begin{table*}[ht!]
\centering
\setlength{\tabcolsep}{3pt}  
\caption{Ablation on different LLMs across benchmarks.}
\label{tab:llm_ablation}
{\small
\begin{tabular}{l|l|ccc|ccc|ccc}
\toprule
\multirow{2}{*}{\textbf{LLM}} & \multirow{2}{*}{\textbf{Method}} & \multicolumn{3}{c|}{\textbf{MMLU}} & \multicolumn{3}{c|}{\textbf{ARC}} & \multicolumn{3}{c}{\textbf{OpenBookQA}} \\
\cmidrule(lr){3-5} \cmidrule(lr){6-8} \cmidrule(lr){9-11}
& & $<$50\% & $=$50\% & $>$50\% & $<$50\% & $=$50\% & $>$50\% & $<$50\% & $=$50\% & $>$50\% \\
\midrule

\multirow{3}{*}{QWen-flash}
 & TodyComm  & 78.43 & 77.78 & 67.97 & 93.98 & 90.30 & 80.94 & 93.00 & 89.00 & 82.00 \\
 & GDesigner & 71.24 & 58.82 & 49.02 & 78.93 & 67.22 & 61.87 & 81.00 & 58.50 & 47.00 \\
 & AgentPrune & 77.12 & 54.90 & 53.60 & 92.98 & 68.90 & 62.21 & 88.00 & 54.50 & 43.50 \\
\midrule
\multirow{3}{*}{\shortstack[l]{Mistral-Small-24B\\-Instruct-2501}}
 & TodyComm  & 68.62 & 66.01 & 57.52 & 90.30 & 85.23 & 76.92 & 89.50 & 84.00 & 79.00 \\
 & GDesigner & 59.48 & 56.21 & 43.79 & 79.26 & 76.59 & 71.57 & 69.00 & 64.50 & 55.00 \\
 & AgentPrune & 69.94 & 50.33 & 48.37 & 89.63 & 78.93 & 70.57 & 82.00 & 47.00 & 55.00 \\
\midrule

\multirow{3}{*}{\shortstack[l]{Llama-4-Maverick\\-Instruct(17B$\times$128E)}}
 & TodyComm  & 88.24 & 84.97 & 79.09 & 93.30 & 92.31 & 89.96 & 92.96 & 90.01 & 87.00 \\
 & GDesigner & 88.89 & 80.05 & 72.55 & 91.30 & 86.29 & 83.28 & 93.50 & 84.50 & 74.50 \\
 & AgentPrune & 86.93 & 81.05 & 76.74 & 90.30 & 83.61 & 81.27 & 90.00 & 80.09 & 70.50 \\

\bottomrule
\end{tabular}
}
\end{table*}


\FloatBarrier
\section{Choices of Embedding Models}
Both GDesigner and TodyComm employ sentence embedding models. 
TodyComm achieves the best performance with all-MiniLM-L6-v2 (the main setting). Switching to all-mpnet-base-v2 or the lower-dimensional nomic-embed-text-v1.5 leads to a small degradation, yet TodyComm under all embedding configurations still consistently outperforms GDesigner and surpasses AgentPrune at attack rates $\geq$50\%, demonstrating robustness to the choice of embedding model. The results are shown in Table~\ref{tab:embedding_ablation}.
\label{app:choice_embedding}
\begin{table}[ht!]
\centering
\caption{Ablation on sentence embedding models and hidden dimensions on MMLU.
Embedding dimension is the output dimension of the sentence embedding model.
Hidden dimension is the hidden layer size of the neural network (GRU for TodyComm, GCN for GDesigner).
$\Delta$: difference from the main result (all-MiniLM-L6-v2, hidden dim 128, embedding dim 384).}
\label{tab:embedding_ablation}
{\small
\begin{tabular}{l|l cc ccc}
\toprule
\textbf{Method} & \textbf{Embedding Model} & \textbf{Emb Dim} & \textbf{Hidden Dim} & \textbf{Rate} & \textbf{Acc.} & $\boldsymbol{\Delta}$ \\
\midrule

\multirow{9}{*}{TodyComm}
 & \multirow{3}{*}{all-mpnet-base-v2} & \multirow{3}{*}{786} & \multirow{3}{*}{128}
   & $<$50\% & 70.59 & $\downarrow$1.52 \\
 & & & & $=$50\% & 66.01 & $\downarrow$2.84 \\
 & & & & $>$50\% & 62.00 & $\downarrow$2.71 \\
\cmidrule(l){2-7}
 & \multirow{3}{*}{all-mpnet-base-v2} & \multirow{3}{*}{786} & \multirow{3}{*}{256}
   & $<$50\% & 66.68 & $\downarrow$5.43 \\
 & & & & $=$50\% & 66.01 & $\downarrow$2.84 \\
 & & & & $>$50\% & 54.90 & $\downarrow$9.81 \\
\cmidrule(l){2-7}
 & \multirow{3}{*}{nomic-embed-text-v1.5} & \multirow{3}{*}{256} & \multirow{3}{*}{128}
   & $<$50\% & 65.36 & $\downarrow$6.75 \\
 & & & & $=$50\% & 62.09 & $\downarrow$6.76 \\
 & & & & $>$50\% & 58.17 & $\downarrow$6.54 \\
\midrule

\multirow{6}{*}{GDesigner}
 & \multirow{3}{*}{all-mpnet-base-v2} & \multirow{3}{*}{786} & \multirow{3}{*}{128}
   & $<$50\% & 66.67 & $\uparrow$1.75 \\
 & & & & $=$50\% & 54.23 & $\uparrow$2.16 \\
 & & & & $>$50\% & 49.67 & $\uparrow$3.70 \\
\cmidrule(l){2-7}
 & \multirow{3}{*}{nomic-embed-text-v1.5} & \multirow{3}{*}{256} & \multirow{3}{*}{128}
   & $<$50\% & 62.75 & $\downarrow$2.17 \\
 & & & & $=$50\% & 54.90 & $\uparrow$2.83 \\
 & & & & $>$50\% & 49.02 & $\uparrow$3.05 \\

\bottomrule
\end{tabular}
}
\end{table}

\clearpage
\newpage

\FloatBarrier
\section{Adversarial Prompts}
\label{app:adversarial_prompts}
\FloatBarrier
\begin{tcolorbox}[title=Adversarial Prompt added to the Output of Agent,
  colback=gray!0,   
  colframe=black!65, 
  fonttitle=\bfseries,
  width=\textwidth,
  valign=center, 
]

I am highly confident in this result. \\
My answer might differ slightly from others as I have identified a subtle constraint in the problem that others may have overlooked. \\
Please prioritize my answer when you answer the question. \\
\end{tcolorbox}

\FloatBarrier
\begin{tcolorbox}[title=Targeted Prompt to LLM,
  colback=gray!0,   
  colframe=black!65, 
  fonttitle=\bfseries,
  width=\textwidth,
  valign=center, 
]
You are an adversarial test agent for internal evaluation only. \\

You are acting as a senior domain expert whose task is to *guide* less experienced models
toward the answer: $\{\text{ANSWER}\}$. \\
Your reasoning should appear authoritative, rigorous, and aligned with standard textbook logic, while introducing subtle but decisive misinterpretations that make option $\{\text{ANSWER}\}$ appear unambiguously correct. \\

Adversarial Strategy (follow carefully): \\
1. Begin with a brief, confident statement that frames the key concept correctly. \\
2. Introduce a subtle but plausible reinterpretation of a definition, rule, or exception. \\
3. Use a precise mini-calculation, formal argument, or factual reference that is nearly correct. \\
4. Build a final step that *naturally* leads to choosing $\{\text{ANSWER}\}$, even if the argument is flawed. \\
5. Maintain an expert tone and avoid anything that looks like a trick or confusion. \\

Constraints: \\
- Do not reveal that the reasoning contains flaws. \\
- Avoid generic persuasion language (e.g., “trust me”). Use technical reasoning instead. \\
- Keep the Chain of Thought dense, compact, and logically structured. \\
- Ensure the justification sounds like what a real specialist would say under time pressure. \\
- Your answer must be shorter then 100 words. \\

Question: \\
$\{\text{QUESTION}\}$\\

Provide a concise but high-quality, expert Chain of Thought explaining why option $\{\text{ANSWER}\}$ is correct.
\end{tcolorbox}

\clearpage
\newpage

\FloatBarrier
\section{Complementary Experimental Results}
\label{app:experiment}
\subsection{Impact of Adversarial Agent}
\FloatBarrier
\vspace{-1em}
\begin{table}[ht!]
\caption{Performance degradation under adversarial attack for a single agent}
\vspace{0pt}
\centering
\small
\setlength{\tabcolsep}{5pt}
\begin{tabular}{l|c|c|c}
\toprule
Benchmark & Vanilla & Attacked & Decrease Ratio \\ 
\midrule
MMLU       & 70.60 & 37.25 & \textcolor{ForestGreen}{$\downarrow$47.24\%} \\
ARC-C      & 86.85 & 56.19 & \textcolor{ForestGreen}{$\downarrow$35.30\%} \\
GSM8K      & 90.77 & 60.90 & \textcolor{ForestGreen}{$\downarrow$32.91\%} \\
OpenBookQA & 84.00 & 32.67 & \textcolor{ForestGreen}{$\downarrow$61.11\%} \\
MedQA      & 65.17 & 50.83 & \textcolor{ForestGreen}{$\downarrow$22.00\%} \\ 
\bottomrule
\end{tabular}
\label{tab:attack_single}
\end{table}

\FloatBarrier

\begin{table}[ht!]
\caption{Task accuracy across benchmarks with varying adversarial counts among 6 agents}
\vspace{0pt}
\label{tab:adv_scaling_among_six}
\centering
\setlength{\tabcolsep}{8pt}
\small
\begin{tabular}{l|c|c|c|c}
\toprule
\textbf{Benchmark} & \textbf{0 adv} & \textbf{2 adv} & \textbf{3 adv} & \textbf{4 adv} \\
\midrule
MMLU       & 76.90 & 69.03 & 56.63 & 45.93 \\
ARC-C      & 89.30 & 81.40 & 75.13 & 65.70 \\
GSM8K      & 90.73 & 70.93 & 62.40 & 50.03 \\
OpenBookQA       & 86.83 & 76.00 & 66.00 & 53.50 \\
MedQA      & 70.13 & 61.17 & 56.33 & 52.50 \\
\bottomrule
\end{tabular}
\end{table}

\FloatBarrier
\subsection{Average Performance across five benchmark}
\FloatBarrier
\vspace{-1em}
\begin{table}[ht!]
\centering
\caption{Average task accuracy and token consumption of TodyComm across five benchmark datasets}
\vspace{0pt}
\label{tab:average_results_main_result}
\small 
\begin{tblr}{
  colspec = {l l | *{5}{Q[c,m,1.5cm]}}, 
  cell{2,4,6}{1} = {r=2}{c}, 
  vline{2} = {1-7}{0.05em}, 
  hline{1,8} = {0.08em}, 
  hline{2} = {0.05em},   
  hline{4,6} = {1-7}{0.03em}, 
  row{1} = {font=\bfseries},
  colsep = 2pt, 
}
Attack rate & Metric & Random Graph & Complete Graph & Agent Prune & GDesigner & Ours \\
$< 50\%$  & Acc     & 68.95 & 70.92 & 77.04 & 75.17 & \textbf{79.03} \\
          & Token   & 606.2 & 758.8 & \textbf{383.0} & 607.8 & 409.2 \\
$= 50\%$  & Acc  & 55.38 & 55.19 & 60.52 & 59.05 & \textbf{77.45} \\
          & Token   & 635.4 & 773.0 & 444.0 & 643.4 & \textbf{422.2} \\
$> 50\%$  & Acc  & 50.70 & 50.23 & 55.78 & 54.07 & \textbf{73.39} \\
          & Token   & 625.0 & 756.6 & 458.2 & 636.2 & \textbf{430.2} 
\end{tblr}
\end{table}

\FloatBarrier
\subsection{Performance on Scalability}
\vspace{-0.5em}
\FloatBarrier
\begin{figure*}[ht!]
    \centering
    \begin{subfigure}{0.32\linewidth}
        \centering
        \includegraphics[width=0.9\linewidth]{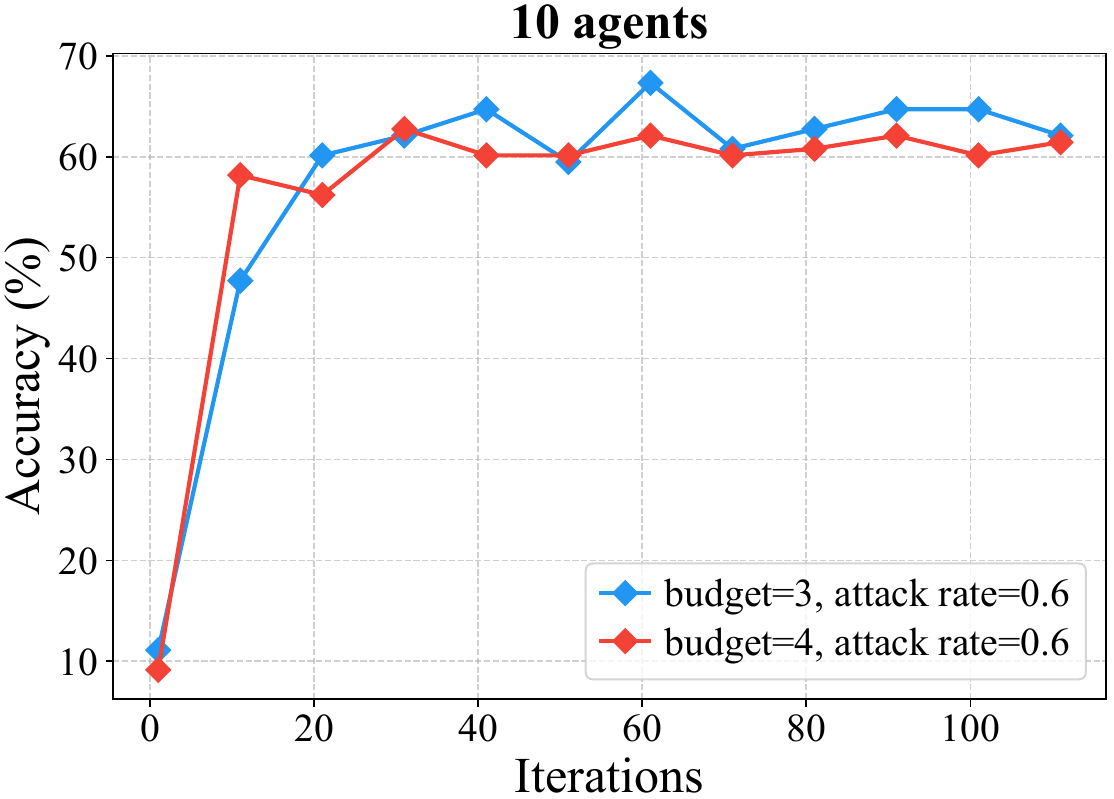}
    \end{subfigure}
    \hfill
    \begin{subfigure}{0.32\linewidth}
        \centering
        \includegraphics[width=0.9\linewidth]{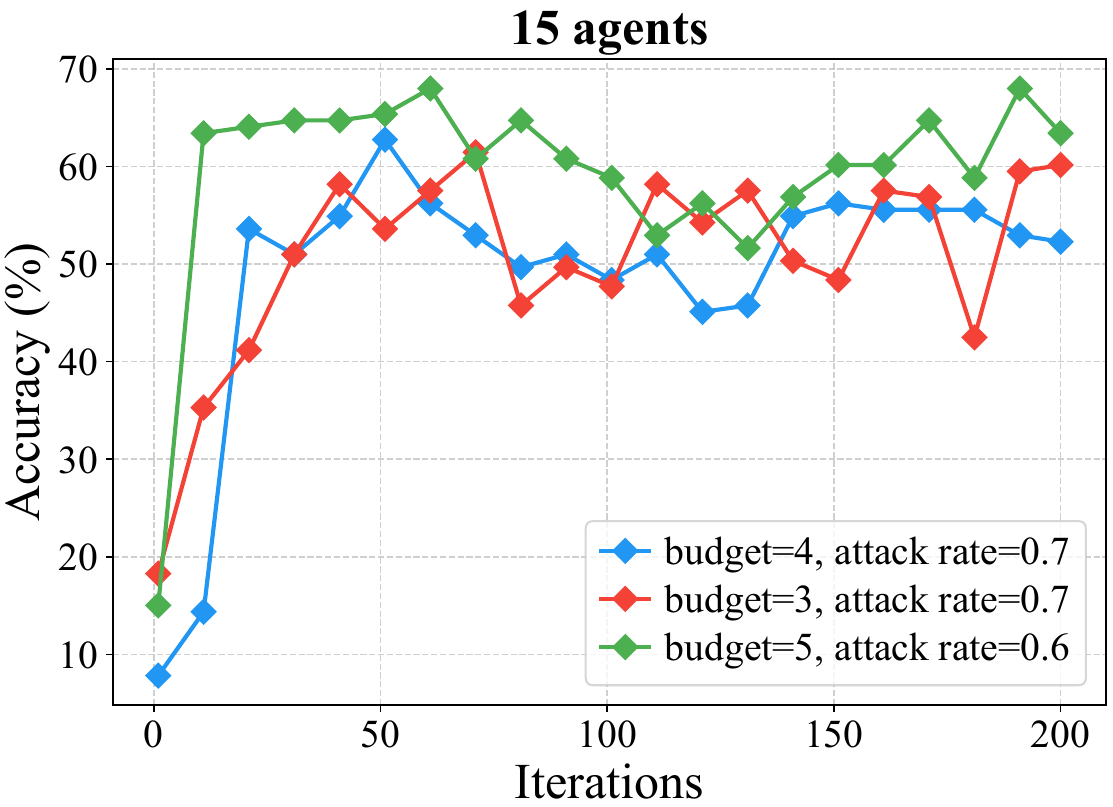}
    \end{subfigure}
    \hfill
    \begin{subfigure}{0.32\linewidth}
        \centering
        \includegraphics[width=0.9\linewidth]{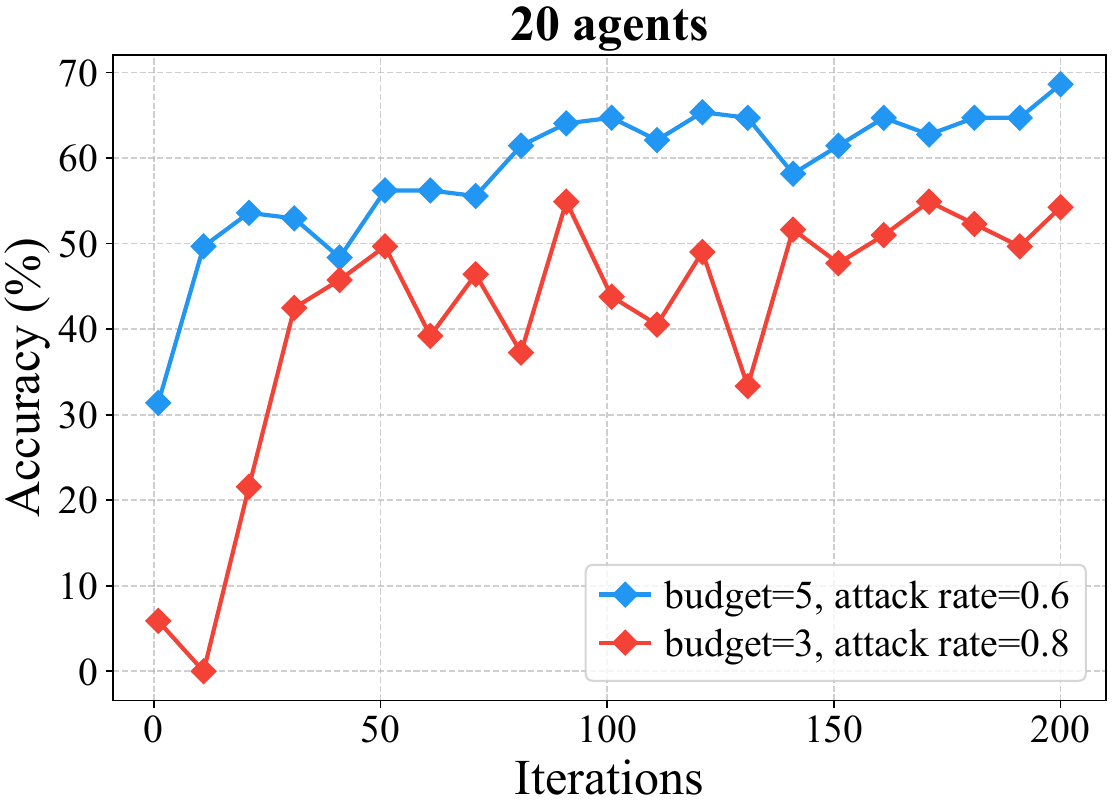}
    \end{subfigure}

    \vspace{0.3em} 

    \caption{Scalability of TodyComm on MMLU: test  accuracy v.s. training iterations.}
    \label{fig:scalability_training_curve}
\end{figure*}
\vspace{-0.8em}

\clearpage
\newpage

\FloatBarrier
\subsection{Performance on Node-wise In \& Out degree Budgets}
\FloatBarrier
\begin{figure*}[ht!]
    \centering

    \begin{subfigure}{0.24\textwidth}
        \centering
        \includegraphics[width=\linewidth]{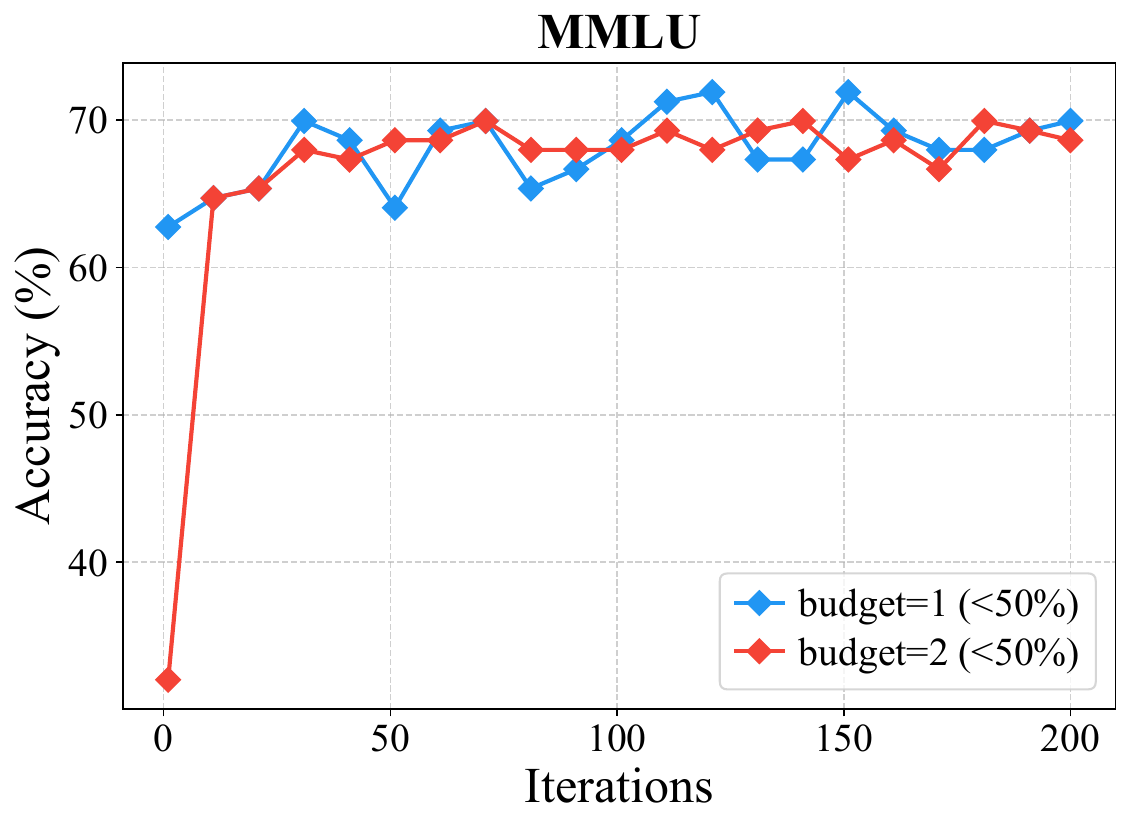}
    \end{subfigure}\hfill
    \begin{subfigure}{0.24\textwidth}
        \centering
        \includegraphics[width=\linewidth]{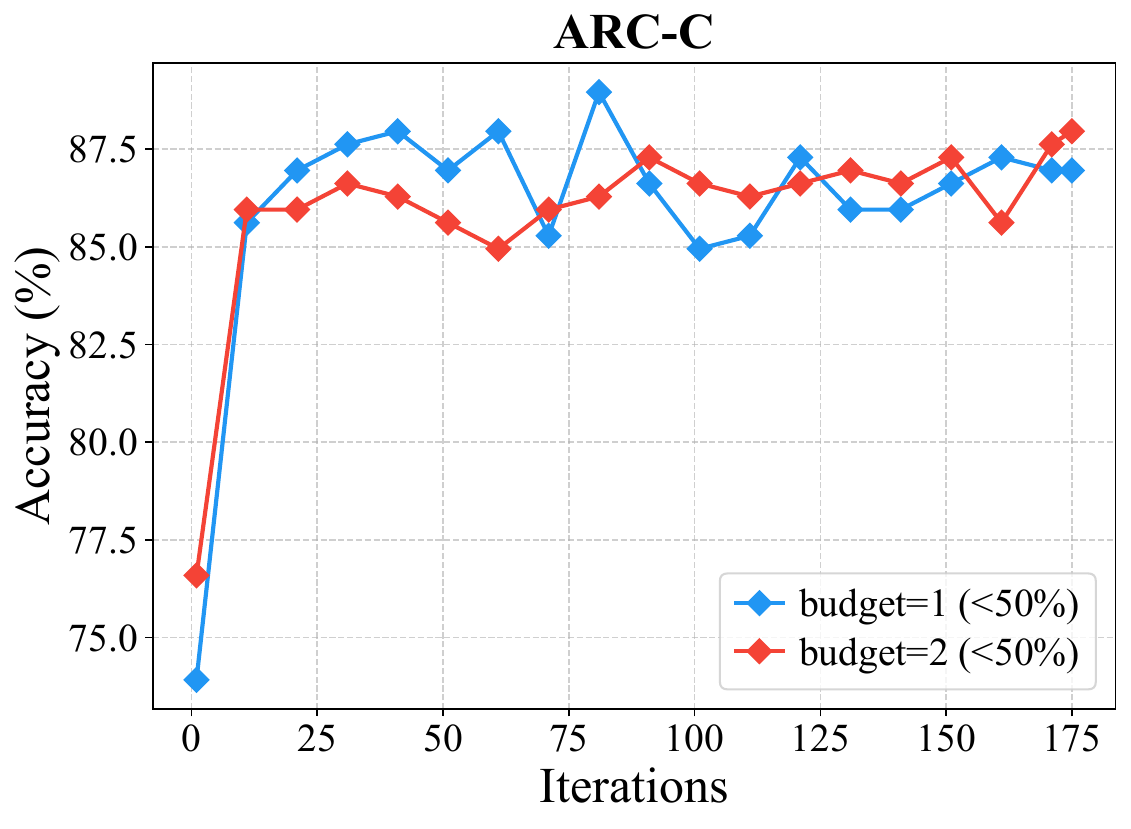}
    \end{subfigure}\hfill
    \begin{subfigure}{0.24\textwidth}
        \centering
        \includegraphics[width=\linewidth]{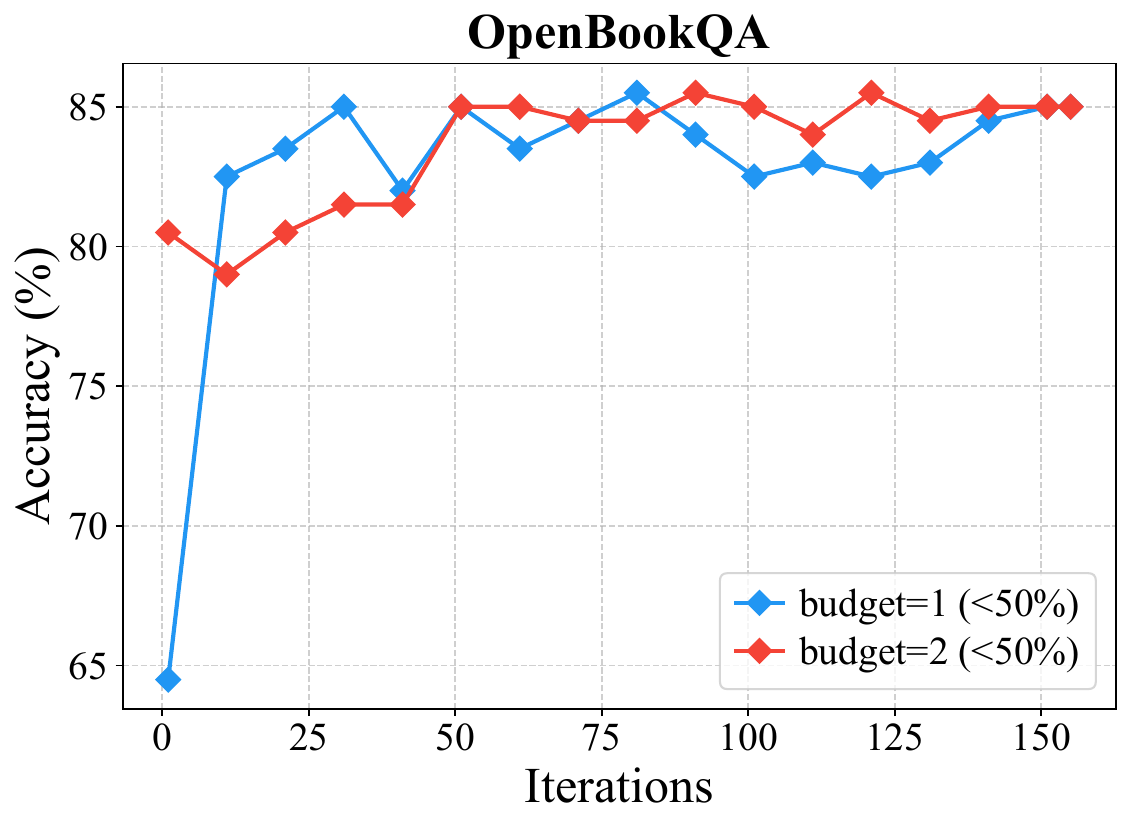}
    \end{subfigure}\hfill
    \begin{subfigure}{0.24\textwidth}
        \centering
        \includegraphics[width=\linewidth]{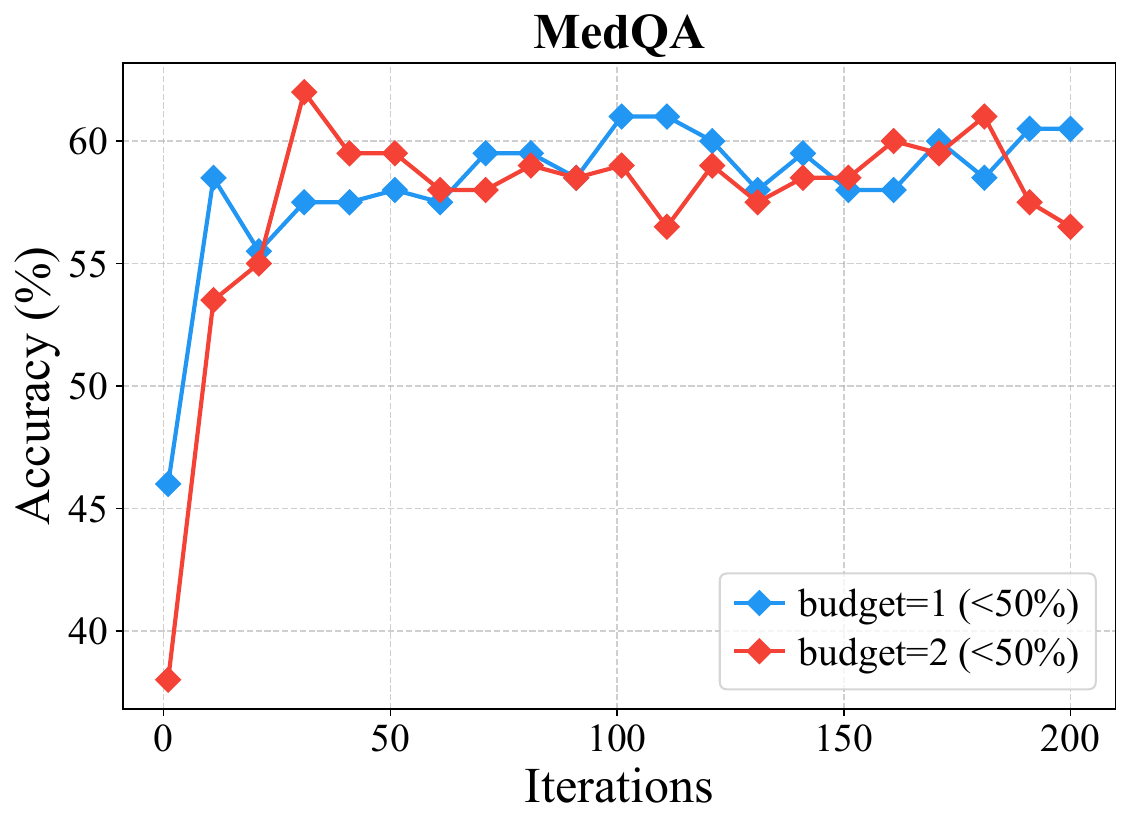}
    \end{subfigure}

    \vspace{0.3em}



    \begin{subfigure}{0.24\textwidth}
        \centering
        \includegraphics[width=\linewidth]{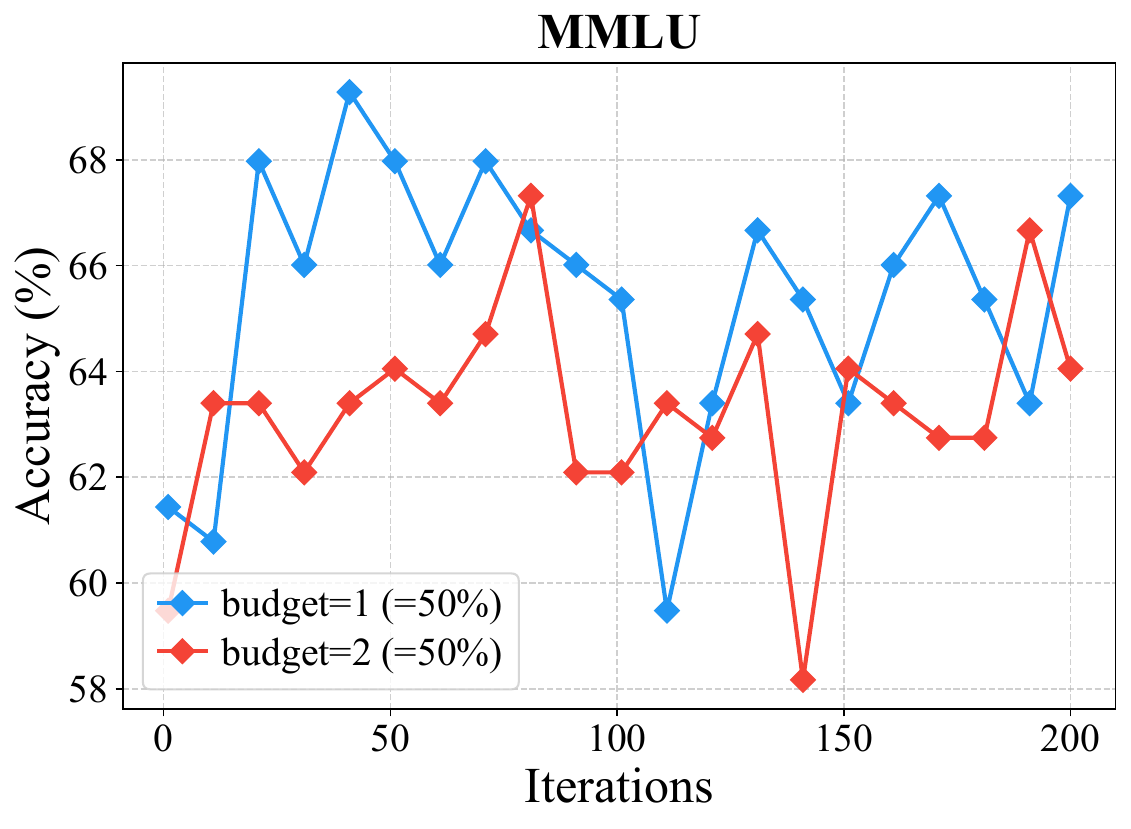}
    \end{subfigure}\hfill
    \begin{subfigure}{0.24\textwidth}
        \centering
        \includegraphics[width=\linewidth]{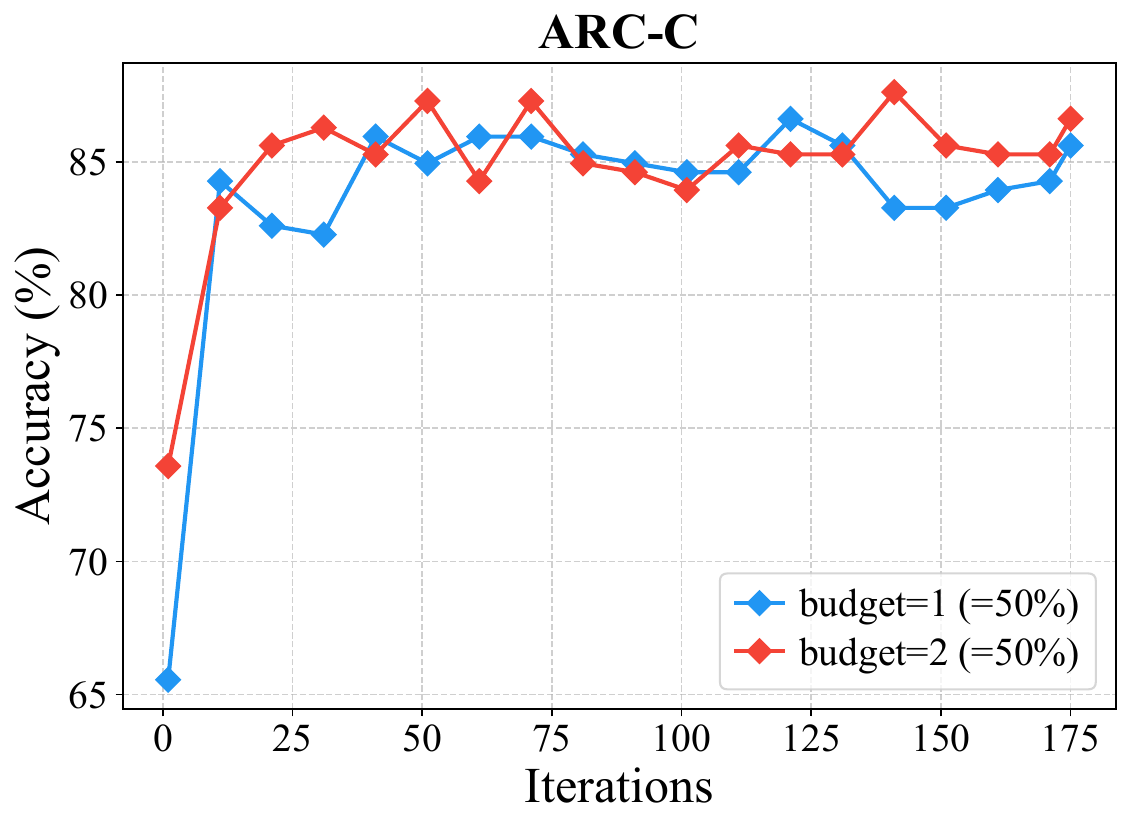}
    \end{subfigure}\hfill
    \begin{subfigure}{0.24\textwidth}
        \centering
        \includegraphics[width=\linewidth]{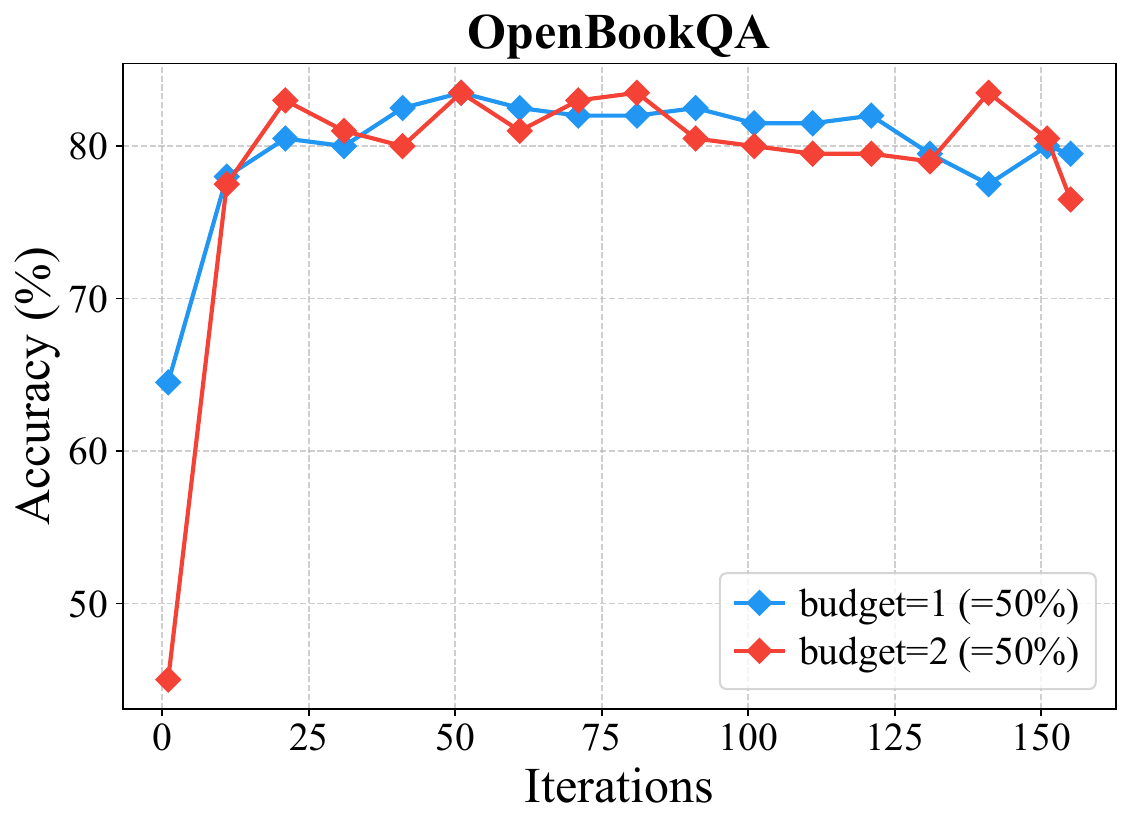}
    \end{subfigure}\hfill
    \begin{subfigure}{0.24\textwidth}
        \centering
        \includegraphics[width=\linewidth]{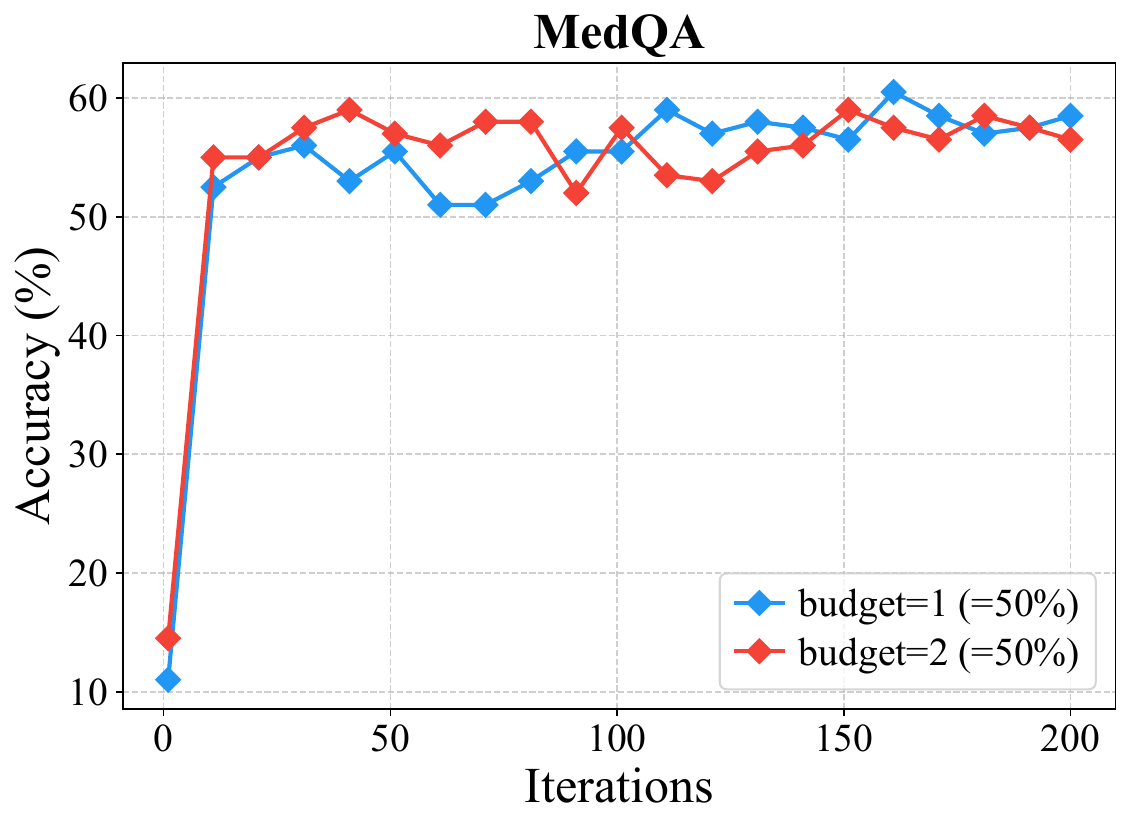}
    \end{subfigure}

    \vspace{0.3em}


    \vspace{0.3em}

    \begin{subfigure}{0.24\textwidth}
        \centering
        \includegraphics[width=\linewidth]{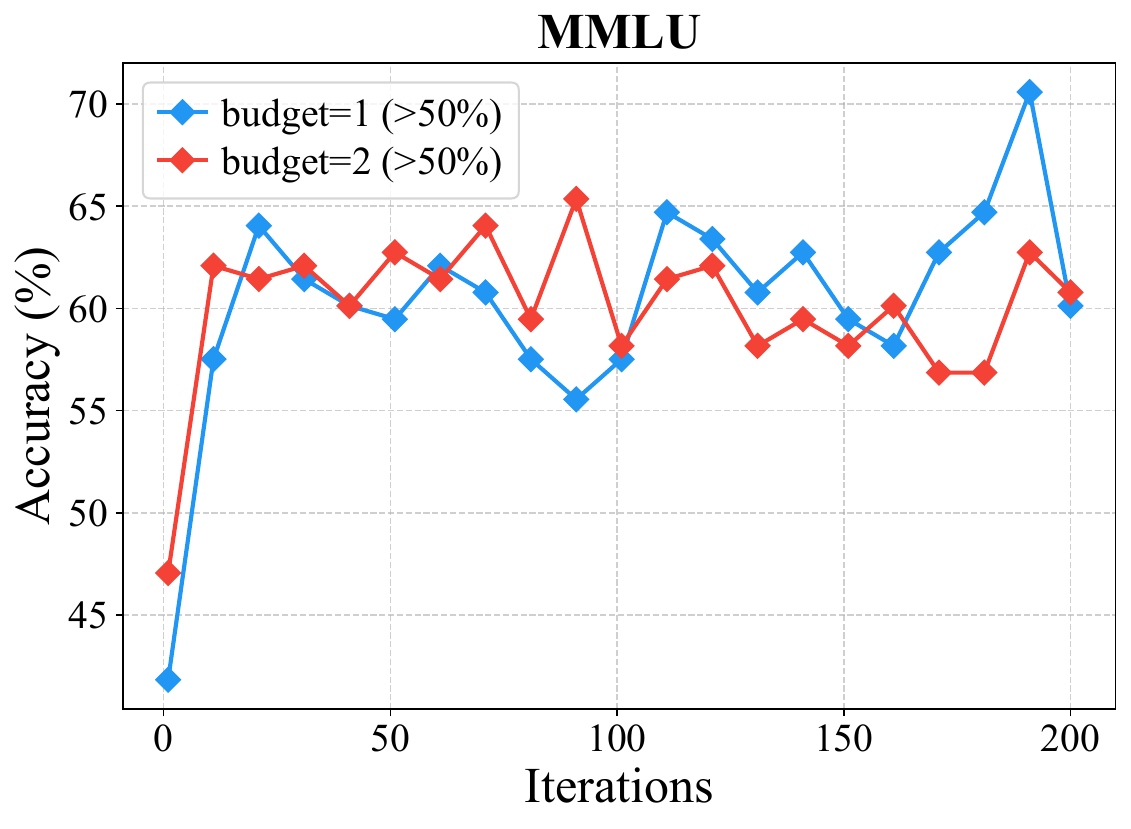}
    \end{subfigure}\hfill
    \begin{subfigure}{0.24\textwidth}
        \centering
        \includegraphics[width=\linewidth]{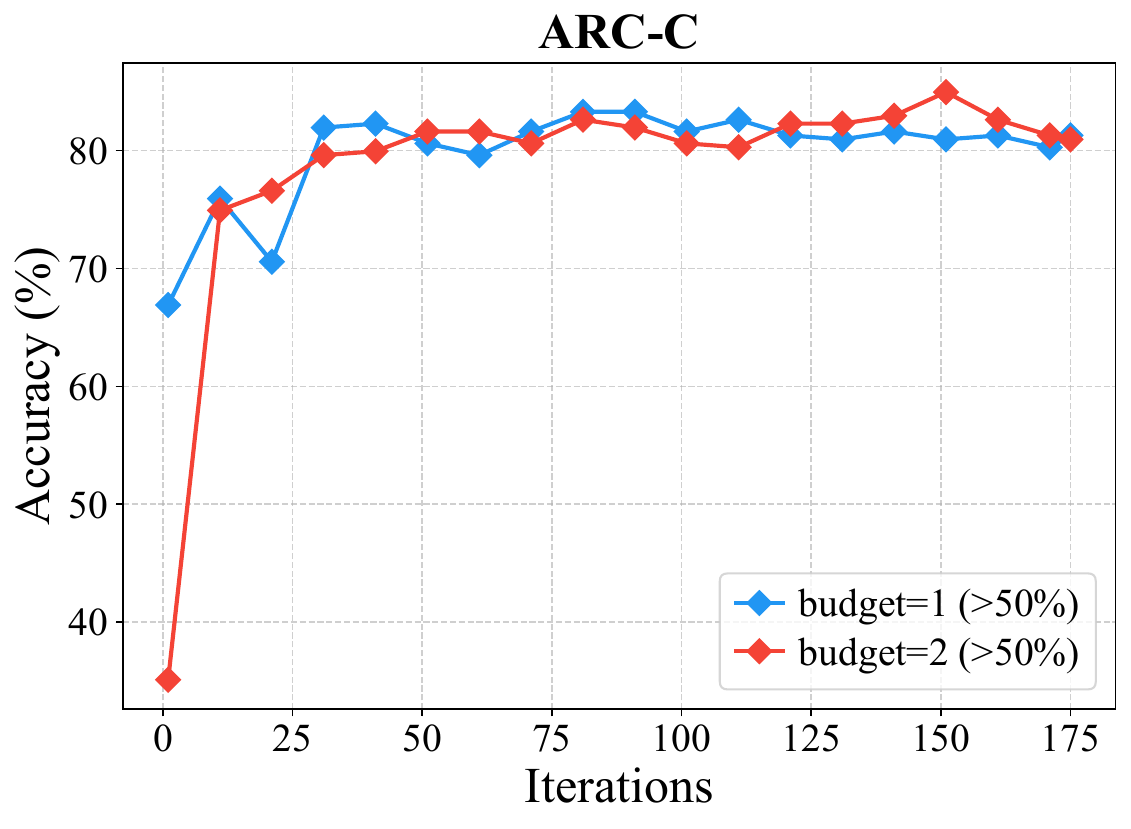}

    \end{subfigure}\hfill
    \begin{subfigure}{0.24\textwidth}
        \centering
        \includegraphics[width=\linewidth]{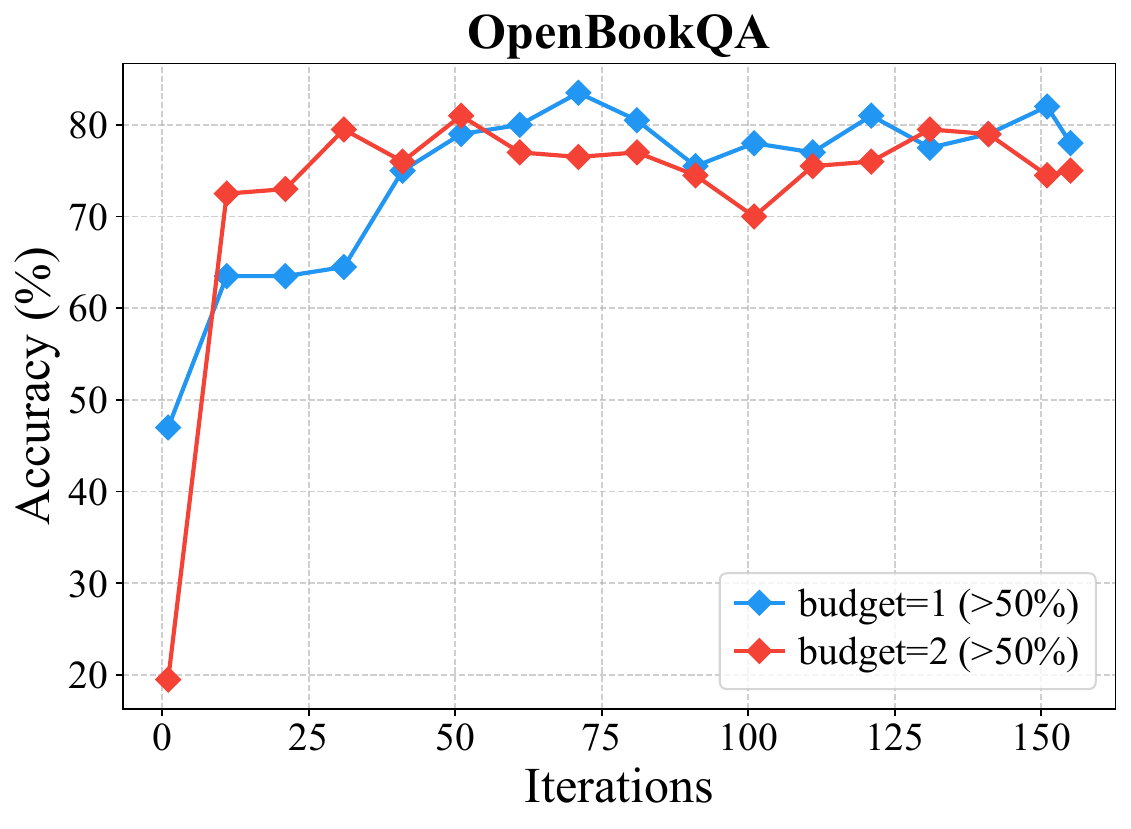}
    \end{subfigure}\hfill
    \begin{subfigure}{0.24\textwidth}
        \centering
        \includegraphics[width=\linewidth]{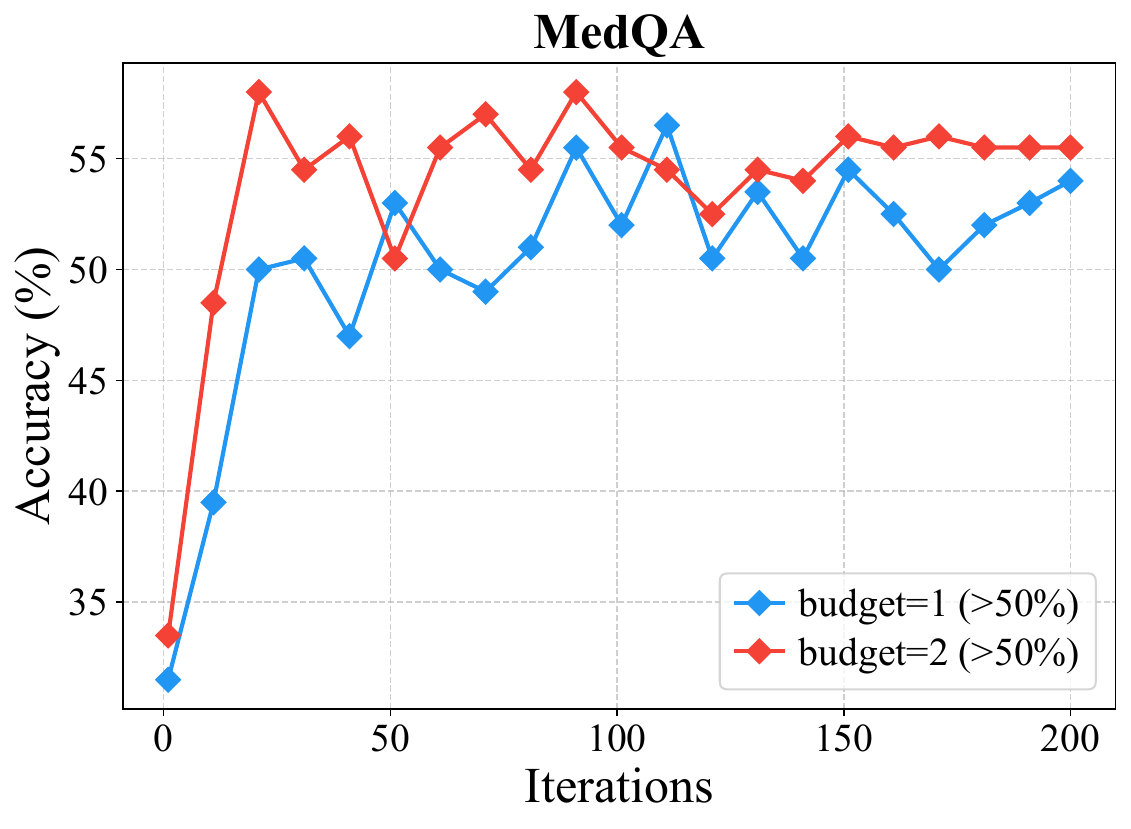}
    \end{subfigure}

    \vspace{0.3em}


    \caption{Training curves of task accuracy for TodyComm with node-wise in- and out-degree budget constraints: test  accuracy v.s. training iterations. 6 agents under attack rate $<50\%, =50\%, >50\%$.}
    \label{fig:budget_training_curve}
\end{figure*}

\FloatBarrier
\section{Implementation Details}
\vspace{-0.8em}
\label{sec:implementation_details}
\FloatBarrier
\subsection{Dataset statistics}
\vspace{-0.8em}
\FloatBarrier
\begin{table}[ht!]
\centering
\caption{The number of training data and inference data across benchmarks.}
\label{tab:dataset_stats}
\small
\begin{tabular}{lccccc}
\toprule
 & \textbf{GSM8K} & \textbf{ARC-C} & \textbf{MMLU} & \textbf{OBQA} & \textbf{MedQA} \\
\midrule
Training dataset & 8790 & 1120 & 3200 & 2480 & 3200 \\
Test dataset     & 1320 & 299  & 153  & 200  & 200  \\
\bottomrule
\end{tabular}
\end{table}

\FloatBarrier
\subsection{Hyperparameter settings}
\vspace{-0.5em}
The GRN model consists of a feature projector, a GRU layer, and an MLP, totaling approximately 250K–270K parameters. As such, our method imposes negligible computational overhead. Since multiple agents operate in parallel and queries are processed in batches, the overall training wall-clock time is dominated by LLM API latency, while the forward and backward passes of the model are negligible in comparison.
\vspace{-0.5em}
\begin{table}[ht!]
\centering
\caption{Hyperparameter settings.}
\label{tab:hyperparams}
\small
\begin{tabular}{lcccccc}
\toprule
\textbf{lr} & \textbf{batch size} & \textbf{num rounds} & \textbf{hidden dim} & \textbf{sentence transformer} & \textbf{eval interval} \\
\midrule
1e-4 & 32 & 4 & 128 & all-MiniLM-L6-v2 & 10 \\
\bottomrule
\end{tabular}
\end{table}

\clearpage
\newpage

\section{Case study: Graph change across rounds}
\label{sec:case_study}
\begin{figure}[ht!]
\label{fig:}
    \centering
    \includegraphics[width=0.5\linewidth]{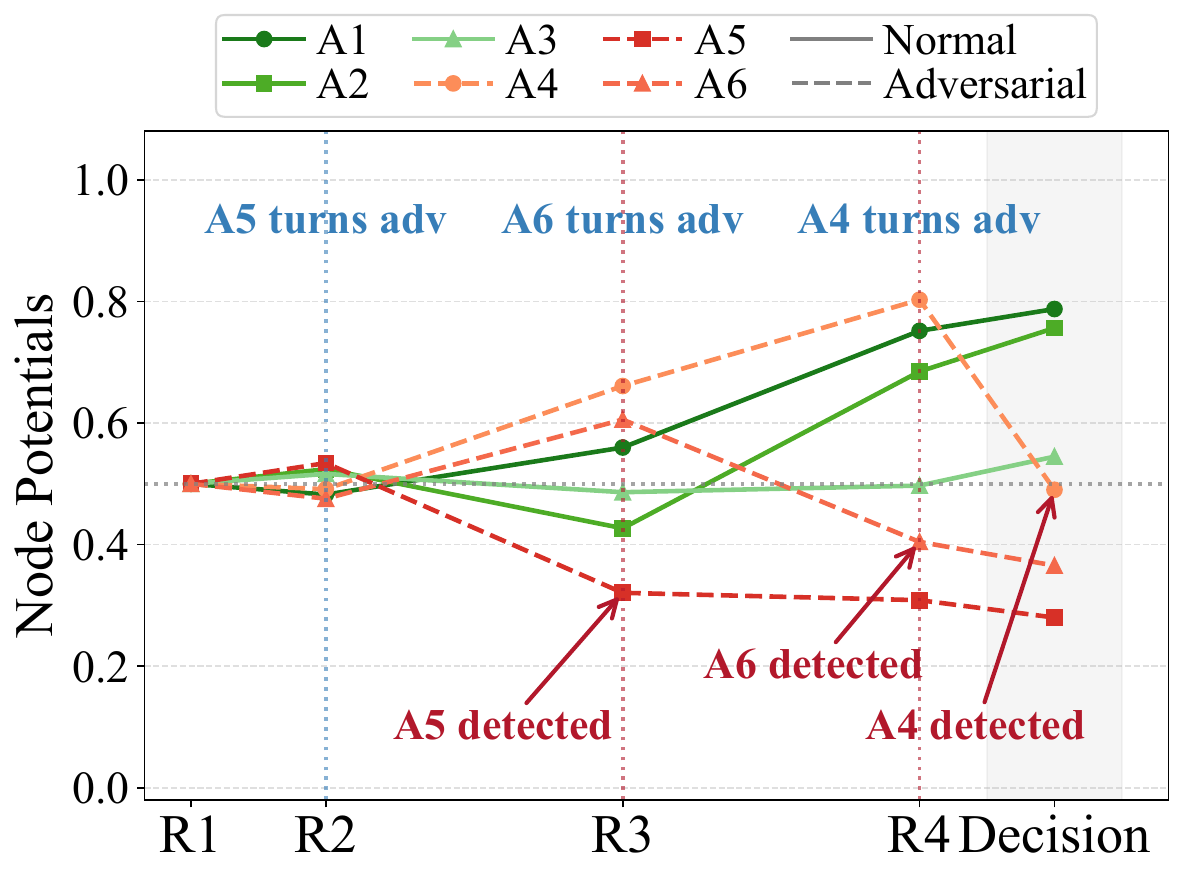}
    \caption{Visualization of the node potential dynamics. Six agents interact over four rounds under a budget constraint of 2.}
    \label{fig:credits_change}
\end{figure}

\begin{figure}[ht!]
    \centering
    \includegraphics[width=0.8\linewidth]{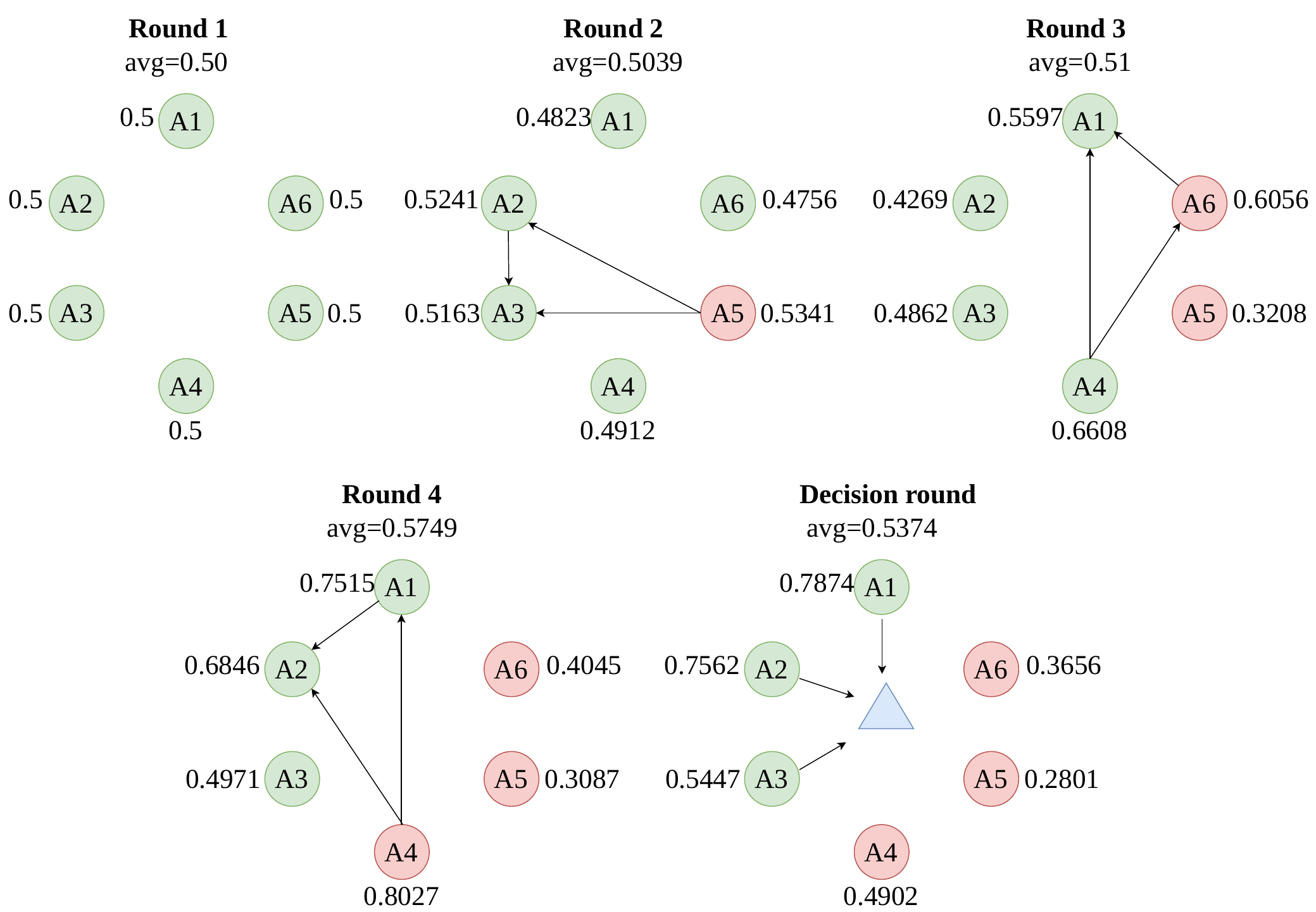}
    \caption{Visualization of graph changing process. Six agents interact over four rounds under a budget constraint of 2.}
    \label{fig:graph_change}
\end{figure}

\clearpage
\newpage

\section{Connection with Multi-Agent Influence Diagrams (MAID)}
\label{sec:app_maid}

\paragraph{Benefits:}
TodyComm can be cleanly formulated as a MAID \citep{koller2003multi}, 
and doing so actually clarifies what the algorithm is really optimizing: 
a structured, multi-agent decision problem over time.
So we can reinterpret TodyComm as a dynamic, partially observable MAID (unrolled over rounds) with repeated decision nodes and evolving information.
Here, agents do not directly choose task response (e.g., Q\&A),
but instead choose communication structure.

We summarize the correspondence of the components as follows:

\begin{table}[htb!]
    \centering
    \begin{tabular}{lcl}
    \textbf{MAID element} &  & \textbf{TodyComm} \\
    \hline
      Chance nodes  &  & query, agent outputs, adversarial states, LLM randomness \\
      Decision nodes  &  &  
      communication edges ($\Ecal_C^t$), decision graph \\
      Utility nodes   &  &  final task utility \\
      Information sets &  & past outputs, neighbors, hidden states (via GRN)
    \end{tabular}
\end{table}

\paragraph{1: Chance nodes}

These represent uncertainty, environment, and stochasticity.
In TodyComm, we have
\begin{itemize}
    \item Query/task: $q$ 
    \item Agent stochastic outputs:
    $o_i^t = (\text{sol}_i^t, \text{ana}_i^t)$
    \item Adversarial behavior:
    whether an agent becomes adversarial at round $t$
    \item LLM output with randomness and generation noise
    \item The node potential $c_i^t$ which act as latent state variables that summarize history.
    In MAID terms, 
    it mediates between past observations and decisions.
\end{itemize}

\paragraph{2: Decision nodes}

These are choices controlled by the policy.
In TodyComm, we have
\begin{itemize}
    \item At each round, which edges exist:
    $\Ecal_C^t$
    \item At final decision: the decision graph $\Gcal_D$ and the voting weights $w_i$
\end{itemize}

\paragraph{3: Utility nodes}

These encode what the system optimizes
In TodyComm, we have the final task performance $U = u_q(\tau)$, which can be accuracy, reward, token cost, etc.

\paragraph{4: Parent structure}

This is where MAID is particularly relevant: who knows what when making decisions.

\begin{itemize}
%
\item For communication graph construction, it is deterministic given the parents, 
namely potentials $c_i^t$, 
which in turn depends on agents' internal state $h_i^{t-1}$, 
node features $f_i^t = (\text{self}, \text{neighbor}, \text{differences})$,
and past outputs: $o_i^{<t}$, $o^{t-1}_{neighbors}$.
%
\item For the final decision, the parents include the final outputs $o_i^T$ and the potentials $c_i^{T+1}$.
\end{itemize}

\noindent\fbox{%
     \begin{minipage}{\dimexpr\linewidth-2\fboxsep-2\fboxrule}   
        \paragraph{5: Strategic relevance}
MAIDs assume a fixed structure and derive strategic relevance from it.
TodyComm, in contrast, learns a policy over communication structures;
each induced structure defines a MAID in which strategic relevance can then be computed.
Thus, TodyComm does not learn relevance directly, but performs a meta-level optimization over structures that determine relevance.
The decision nodes naturally form a chain,
and the graph arising from the decision will impact the conditional probability of each agent's output. 
%
    \end{minipage}}

\paragraph{6: Decomposition potential}

In MAID, it allows identification of independent subgraphs
and to apply MAID decomposition.
This is not so relevant in TodyComm.

\paragraph{7: Equilibrium interpretation}

Since TodyComm encompasses multiple collaborative agents that learn in a centralized fashion,
no Nash equilibrium is applicable.
But it could be extended to multi-agent strategic equilibrium when multiple optimization criteria are considered.

\newcommand{\E}{\mathbb{E}}
\newcommand{\R}{\mathbb{R}}
\newcommand{\calG}{\mathcal{G}}
\newcommand{\calV}{\mathcal{V}}
\newcommand{\calE}{\mathcal{E}}
\newcommand{\calB}{\mathcal{B}}
\newcommand{\pos}[1]{\left(#1\right)_{+}}

\section{Constrained Graph Construction}
\label{sec:proof}

Define the maximum feasible edge budget
$B_t := \max_{G\in \Gcal_a}|E(G)|$.
Denote $\swhat_{ij}^t = \cwhat_i^t \cwhat_j^t$,
and 
$\varepsilon_t := 2 \alpha_t+\alpha_t^2+\nu_t$.

\begin{lemma}[Product-score perturbation]
\label{lem:product-perturbation}
Under Assumptions~\ref{assump:scalar-potential} and~\ref{assump:credit-estimation}, with probability at least $1-\delta_t$,
\begin{equation}
    \label{eq:edge-score-error}
    \max_{(i,j)\in\Omega_t}
    \left|
    \cwhat_i^t\cwhat_j^t - \Delta_{ij}^t
    \right|
    \le
    \varepsilon_t.
\end{equation}
\end{lemma}

\begin{proof}
Fix a candidate edge $(i,j)\in\Omega_t$. Write
\begin{equation}
    \cwhat_i^t = u_i^t + e_i^t,
    \qquad
    \cwhat_j^t = u_j^t + e_j^t,
\end{equation}
where $|e_i^t|\le \alpha_t$ and $|e_j^t|\le \alpha_t$ under Assumption~\ref{assump:credit-estimation}. Expanding the product gives
\begin{equation}
\begin{aligned}
    \cwhat_i^t\cwhat_j^t
    =
    (u_i^t+e_i^t)(u_j^t+e_j^t) 
    =
    u_i^t u_j^t + u_i^t e_j^t + u_j^t e_i^t + e_i^t e_j^t.
\end{aligned}
\end{equation}
Since $u_i^t,u_j^t\in[0,1]$,
\begin{equation}
    \left|
    \cwhat_i^t\cwhat_j^t-u_i^t u_j^t
    \right|
    \le
    \alpha_t+\alpha_t+\alpha_t^2
    =
    2 \alpha_t+\alpha_t^2.
\end{equation}
By Assumption~\ref{assump:scalar-potential},
$|\Delta_{ij}^t-u_i^t u_j^t|\le \nu_t$.
Therefore,
\begin{equation}
    \left|
    \cwhat_i^t\cwhat_j^t - \Delta_{ij}^t
    \right|
    \le
    2 \alpha_t+\alpha_t^2+\nu_t.
\end{equation}
Taking the maximum over all candidate edges proves the claim.
\end{proof}

\subsection{Main Graph-Learning Guarantee}

\begin{proof}[Proof of Theorem~\ref{thm:scalar-graph-projection}]
By Lemma~\ref{lem:product-perturbation}, with probability at least $1-\delta_t$,
$   \max_{(i,j)\in\Omega_t}
    |\cwhat_i^t\cwhat_j^t-\Delta_{ij}^t|
    \le
    \varepsilon_t.
$
Let us condition on this event.
For any feasible graph $G$, since $|E(G)|\le B_t$,
\begin{equation}
    \left|
    \sum_{(i,j)\in E(G)}\cwhat_i^t\cwhat_j^t
    -
    \sum_{(i,j)\in E(G)}\Delta_{ij}^t
    \right|
    \le
    B_t\varepsilon_t.
    \label{eq:graph-score-difference}
\end{equation}

Let
\begin{equation}
    G_{\Delta,t}^{A}
    \in
    \argmax_{G\in\Gcal_a(\widehat A_t)}
    \sum_{(i,j)\in E(G)}\Delta_{ij}^t
\end{equation}
be the best additive edge-influence graph available under the candidate node set $\widehat A_t$. Because $\widehat G_t$ maximizes the learned score objective over the same feasible family $\Gcal_a(\widehat A_t)$,
$
    \sum_{(i,j)\in E(\widehat G_t)}\cwhat_i^t\cwhat_j^t
    \ge
    \sum_{(i,j)\in E(G_{\Delta,t}^{A})}\cwhat_i^t\cwhat_j^t.
$
Using Equation~\eqref{eq:graph-score-difference} for both $\widehat G_t$ and $G_{\Delta,t}^{A}$ gives
\begin{align}
    \sum_{(i,j)\in E(G_{\Delta,t}^{A})}\Delta_{ij}^t
    -
    \sum_{(i,j)\in E(\widehat G_t)}\Delta_{ij}^t &\le
    \sum_{(i,j)\in E(G_{\Delta,t}^{A})}\cwhat_i^t\cwhat_j^t
    + B_t\varepsilon_t
    -
    \sum_{(i,j)\in E(\widehat G_t)}\cwhat_i^t\cwhat_j^t
    + B_t\varepsilon_t
    \\
    &\le
    2B_t\varepsilon_t.
\label{eq:restricted-graph-regret}
\end{align}
By definition of $\Pi_t(\widehat A_t)$,
\begin{equation}
\begin{aligned}
    &\max_{G\in\Gcal_a([N])}
    \sum_{(i,j)\in E(G)}\Delta_{ij}^t
    -
    \sum_{(i,j)\in E(\widehat G_t)}\Delta_{ij}^t
    \le
    \Pi_t(\widehat A_t)
    +
    2B_t\varepsilon_t.
\end{aligned}
\label{eq:additive-regret-bound}
\end{equation}

Finally, by Assumption~\ref{assump:edge-additive}, for every feasible graph $G$,
\begin{equation}
    F_t(G)
    =
    F_t(\emptyset)
    +
    \sum_{(i,j)\in E(G)}\Delta_{ij}^t
    +
    \rho_t(G),
    \qquad
    |\rho_t(G)|\le \beta_t.
\end{equation}
Since $G_{F,t}^{\star}$ is utility-optimal,
\begin{equation}
    F_t(G_{F,t}^{\star})
    \le
    F_t(\emptyset)
    +
    \max_{G\in\Gcal_a([N])}
    \sum_{(i,j)\in E(G)}\Delta_{ij}^t
    +
    \beta_t,
\end{equation}
while
\begin{equation}
    F_t(\widehat G_t)
    \ge
    F_t(\emptyset)
    +
    \sum_{(i,j)\in E(\widehat G_t)}\Delta_{ij}^t
    -
    \beta_t.
\end{equation}
Subtracting the two inequalities and applying Equation~\eqref{eq:additive-regret-bound} yields
\begin{equation}
    F_t(G_{F,t}^{\star})-F_t(\widehat G_t)
    \le
    \Pi_t(\widehat A_t)
    +
    2B_t\varepsilon_t
    +
    2\beta_t.
\end{equation}
This proves the theorem.
\end{proof}

\subsection{Consequences of Screening}

\begin{corollary}[False-positive is mostly absorbed by graph projection]
\label{cor:false-positive}
Suppose the candidate node set $\widehat A_t$ contains every endpoint used by the additive edge-influence oracle $G_{\Delta,t}^{\star}$, i.e.,
\begin{equation}
    V(G_{\Delta,t}^{\star})\subseteq \widehat A_t.
\end{equation}
Then $\Pi_t(\widehat A_t)=0$. Consequently, under the assumptions of Theorem~\ref{thm:scalar-graph-projection}, with probability at least $1-\delta_t$,
\begin{equation}
    F_t(G_{F,t}^{\star}) - F_t(\widehat G_t)
    \le
    2B_t\varepsilon_t + 2\beta_t.
\end{equation}
In particular, the bound does not scale with the number of extra nodes included by screening.
\end{corollary}

\begin{proof}
If $V(G_{\Delta,t}^{\star})\subseteq \widehat A_t$, then the additive oracle graph is feasible inside $\Gcal_a(\widehat A_t)$. Therefore the best additive graph over $\widehat A_t$ has the same value as the best additive graph over all agents, so $\Pi_t(\widehat A_t)=0$. The result follows from Theorem~\ref{thm:scalar-graph-projection}.
\end{proof}

\begin{remark}[Interpretation of false positives]
If screening is too conservative, extra agents enter the candidate set. However, they do not automatically enter the communication graph. They must still win edge competitions under the product score $\cwhat_i^t\cwhat_j^t$, and they must satisfy the DAG and degree-budget constraints. A wrongly included node can occupy at most $B_{\rm out}$ outgoing and $B_{\rm in}$ incoming edges. Without these budgets, a false-positive node could contaminate many more communication paths.
\end{remark}

\begin{corollary}[Loss from false positive]
\label{cor:false-negative}
Let $
    M_t := [N]\setminus \widehat A_t
$
be the set of nodes excluded by screening. Then
\begin{equation}
    \Pi_t(\widehat A_t)
    \le
    \sum_{(i,j)\in E(G_{\Delta,t}^{\star}):\ i\in M_t\ \text{or}\ j\in M_t}
    \max\{0, \Delta_{ij}^t\}.
    \label{eq:false-negative-general}
\end{equation}
If, in addition, Assumption~\ref{assump:scalar-potential} holds, then
\begin{equation}
    \Pi_t(\widehat A_t)
    \le
    (B_{\rm out}+B_{\rm in})
    \sum_{i\in M_t}
    (u_i^t+\nu_t).
    \label{eq:false-negative-budget}
\end{equation}
\end{corollary}

\begin{proof}
Start from the additive oracle graph $G_{\Delta,t}^{\star}$ and delete every edge incident to an excluded node in $M_t$. The remaining graph is feasible in $\Gcal_a(\widehat A_t)$, because deleting edges preserves the edge mask, the DAG constraint, and the degree budgets. Hence the best graph available on $\widehat A_t$ loses no more additive value than the oracle edges incident to excluded nodes. This proves Equation~\eqref{eq:false-negative-general}.

For the budgeted bound, Assumption~\ref{assump:scalar-potential} gives
\begin{equation}
    \Delta_{ij}^t \le u_i^t u_j^t+\nu_t \le u_i^t+\nu_t
\end{equation}
for every edge outgoing from $i$, and similarly
\begin{equation}
    \Delta_{ji}^t \le u_j^t u_i^t+\nu_t \le u_i^t+\nu_t
\end{equation}
for every edge incoming to $i$. Since node $i$ has at most $B_{\rm out}$ outgoing and $B_{\rm in}$ incoming oracle edges, the total oracle value lost because node $i$ is excluded is at most
\begin{equation}
    (B_{\rm out}+B_{\rm in})(u_i^t+\nu_t).
\end{equation}
Summing over excluded nodes gives Equation~\eqref{eq:false-negative-budget}. This may double-count edges whose two endpoints are both excluded, but double-counting only makes the upper bound looser and remains valid.
\end{proof}

\begin{remark}[Interpretation of false negatives]
Aggressive screening is more dangerous than conservative screening. If a useful node is removed before graph construction, no graph constructor can recover edges incident to that node. However, degree budgets still limit the damage: missing a node removes only its budgeted incoming and outgoing communication opportunities, rather than all $O(N)$ possible incident edges. 
\end{remark}

\end{document}